\documentclass{article}

\PassOptionsToPackage{round}{natbib}

\usepackage[final]{neurips_2021}

\usepackage[utf8]{inputenc} 
\usepackage[T1]{fontenc}    
\usepackage[colorlinks = true,
            linkcolor = blue,
            urlcolor  = blue,
            citecolor = blue]{hyperref}       
\usepackage{url}            
\usepackage{booktabs}       
\usepackage{amsfonts}       
\usepackage{nicefrac}       
\usepackage{microtype}      
\usepackage{xcolor}
\usepackage{cancel,soul}
\usepackage{amsmath}
\usepackage{amssymb}
\usepackage{amsthm}
\usepackage{graphicx}
\usepackage{caption}
\usepackage{subcaption}
\usepackage{relsize}
\usepackage{mathrsfs}
\usepackage{multicol}
\usepackage{wrapfig}
\usepackage{selectp}

\newcommand{\R}{{\mathbb{R}}}
\newcommand{\N}{{\mathbb{N}}}

\newcommand{\esp}{{\mathbb{E}}}

\newcommand{\prob}{{\mathbb{P}}}

\newcommand{\stirling}{\genfrac\{\}{0pt}{}}

\definecolor{ForestGreen}{cmyk}{0.91,0,0.88,0.12}

\newtheorem{theorem}{Theorem}
\newtheorem{corollary}{Corollary}
\newtheorem{lemma}{Lemma}
\newtheorem{proposition}{Proposition}
\newtheorem{definition}{Definition}
\newtheorem{example}{Example}

\title{Framing RNN as a kernel method: \\ A neural ODE approach}

\author{%
   Adeline Fermanian$^{1} \thanks{Equal contribution}$
   \And
  Pierre Marion$^1\footnotemark[1]$
   \And
   Jean-Philippe Vert$^2$ 
   \And
     Gérard Biau$^1$  \\
  \And \\[-24pt] 
  \null $^1$  Sorbonne Université, CNRS, \\ Laboratoire de Probabilités, Statistique et Modélisation, LPSM, \\
  F-75005 Paris, France \\
  \texttt{\{adeline.fermanian, pierre.marion, gerard.biau\}@sorbonne-universite.fr} \\
  \null $^2$  Google Research, Brain team, \\
  Paris, France \\
  \texttt{jpvert@google.com}
}

\begin{document}
\setstcolor{ForestGreen}

\maketitle

\begin{abstract}
Building on the interpretation of a recurrent neural network (RNN) as a continuous-time neural differential equation, we show, under appropriate conditions, that the solution of a RNN can be viewed as a linear function of a specific feature set of the input sequence, known as the signature. This connection allows us to frame a RNN as a kernel method in a suitable reproducing kernel Hilbert space. As a consequence, we obtain theoretical guarantees on generalization and stability for a large class of recurrent networks. Our results are illustrated on simulated datasets.
\end{abstract}

\section{Introduction}

Recurrent neural networks (RNN) are among the most successful methods for modeling sequential data. 
They have achieved state-of-the-art results in difficult problems such as natural language processing \citep[e.g.,][]{mikolov2010recurrent,collobert2011natural} or speech recognition \citep[e.g.,][]{hinton2012deep,graves2013speech}. This class of neural networks has a natural interpretation in terms of (discretization of) ordinary differential equations (ODE), which casts them in the field of neural ODE \citep{chen2018neural}. This observation has led to the development of continuous-depth models for handling irregularly-sampled time-series data, including the ODE-RNN model \citep{rubanova2019latent}, GRU-ODE-Bayes \citep{de2019gru}, or neural CDE models \citep{kidger2020neural,morrill2020neural}. In addition, the time-continuous interpretation of RNN allows to leverage the rich theory of differential equations to develop new recurrent architectures \citep{chang2018antisymmetricrnn,herrera2020theoretical,erichson2021lipschitz}, which are better at learning long-term dependencies.

On the other hand, the development of kernel methods for deep learning offers theoretical insights on the functions learned by the networks \citep{cho_saul,belkin2018understand,neural_tangent_kernel}. Here, the general principle consists in defining a reproducing kernel Hilbert space (RKHS)---that is, a function class $\mathscr{H}$---, which is rich enough to describe the architectures of networks. A good example is the construction of \citet{bietti2017invariance,bietti2019group}, who exhibit an RKHS for convolutional neural networks. This kernel perspective has several advantages. First, by separating the representation of the data from the learning process, it allows to study invariances of the representations learned by the network. Next, by reducing the learning problem to a linear one in $\mathscr{H}$, generalization bounds can be more easily obtained. Finally, the Hilbert structure of $\mathscr{H}$ provides a natural metric on neural networks, which can be used for example for regularization \citep{bietti2019kernel}. 

\paragraph{Contributions.} By taking advantage of the neural ODE paradigm for RNN, we show that RNN are, in the continuous-time limit, linear predictors over a specific space associated with the signature of the input sequence \citep{levin2013learning}. The signature transform, first defined by \citet{chen1958integration} and central in rough path theory \citep{lyons2007differential,friz2010multidimensional}, summarizes sequential inputs by a graded feature set of their iterated integrals. Its natural environment is a tensor space that can be endowed with an RKHS structure \citep{Kiraly2016}. We exhibit general conditions under which classical recurrent architectures such as feedforward RNN, Gated Recurrent Units \citep[GRU,][]{Cho2014LearningPR}, or Long Short-Term Memory networks \citep[LSTM,][]{hochreiter1997long}, can be framed as a kernel method in this RKHS. This enables us to provide generalization bounds for RNN as well as stability guarantees via regularization. The theory is illustrated with some experimental results.

\paragraph{Related works.} The neural ODE paradigm was first formulated by \citet{chen2018neural} for residual neural networks. It was then extended to RNN in several articles, with a focus on handling irregularly sampled data \citep{rubanova2019latent,kidger2020neural} and learning long-term dependencies \citep{chang2018antisymmetricrnn}. The signature transform has recently received the attention of the machine learning community \citep{levin2013learning,kidger2019deep,liao2019learning,toth2020bayesian,fermanian2021embedding} and, combined with deep neural networks, has achieved state-of-the-art performance for several applications \citep{yang2016deepwriterid,yang2017leveraging,perez2018derivatives,wang2019path,morrill2020utilization}. \citet{Kiraly2016} use the signature transform to define kernels for sequential data and develop fast computational methods. The connection between continuous-time RNN and signatures has been pointed out by \citet{lim2020understanding} for a specific model of stochastic RNN. Deriving generalization bounds for RNN is an active research area \citep{zhang2018stabilizing,akpinar2019sample,tu2019understanding}. By leveraging the theory of differential equations, our approach encompasses a large class of RNN models, ranging from feedforward RNN to LSTM. This is in contrast with most existing generalization bounds, which are architecture-dependent. Close to our point of view is the work of \citet{bietti2017invariance} for convolutional neural networks.

\paragraph{Mathematical context.}
We place ourselves in a supervised learning setting. The input data is a sample of $n$ i.i.d.~vector-valued sequences $\{\mathbf{x}^{(1)}, \dots,\mathbf{x}^{(n)}\}$, where $\mathbf{x}^{(i)} = (x^{(i)}_{1}, \dots, x^{(i)}_{T}) \in (\R^d)^{T}$, $T\geq 1$. The outputs of the learning problem can be either labels (classification setting) or sequences (sequence-to-sequence setting). Even if we only observe discrete sequences, each $\mathbf{x}^{(i)}$ is mathematically considered as a regular discretization of a continuous-time process $X^{(i)}\in BV^{c}([0, 1], \R^d)$, where $BV^{c}([0,1], \R^d)$ is the space of continuous functions from $[0,1]$ to $\R^d$ of finite total variation. Informally, the total variation of a process corresponds to its length. Formally, for any $[s,t] \subset [0,1]$, the total variation of a process $X \in BV^{c}([0,1], \R^d)$ on $[s,t]$ is defined by
\begin{equation*}
    \|X\|_{TV; [s,t]} = \underset{(t_0, \dots, t_k) \in D_{s,t}}{\textnormal{sup}} \  \sum_{j=1}^k \|X_{t_j} - X_{t_{j-1}} \|,
\end{equation*}
where $D_{s,t}$ denotes the set of all finite partitions of $[s,t]$ and $\| \cdot\|$ the Euclidean norm. We therefore have that $x^{(i)}_j = X^{(i)}_{\nicefrac{j}{T}}$, $1 \leq j \leq T$, where $X^{(i)}_{t}:=X^{(i)}(t)$. We make two assumptions on the processes $X^{(i)}$. First, they all begin at zero, and second, their lengths are bounded by $L \in (0, 1)$. These assumptions are not too restrictive, since they amount to data translation and normalization, common in practice. Accordingly, we denote by $\mathscr{X}$ the subset of $BV^{c}([0,1],\R^d)$ defined by
\begin{equation*}
    \mathscr{X} = \big\{X \in BV^{c}([0,1], \R^d) \, | \,  X_0=0 \quad \textnormal{and} \quad  \|X\|_{TV; [0,1]} \leq L \big\}
\end{equation*}
and assume therefore that $X^{(1)}, \dots, X^{(n)}$ are i.i.d.~according to some $X \in \mathscr{X}$. The norm on all spaces $\R^m$, $m \geq 1$, is always the Euclidean one. Observe that assuming that $X \in \mathscr{X}$ implies that, for any $t \in [0,1]$, $\|X_t\|=\|X_t-X_0\| \leq \|X\|_{TV;[0,1]} \leq L$.

\paragraph{Recurrent neural networks.}
Classical RNN are defined by a sequence of hidden states $h_{1}, \dots, h_{T} \in \R^e$, where, for $\mathbf{x}= (x_1, \dots, x_T)$ a generic data sample,
\begin{equation*}
    h_0 = 0 \quad  \text{and} \quad h_{j+1} = f(h_j, x_{j+1}) \quad \textnormal{for } 0 \leq j \leq T-1.
\end{equation*}
At each time step $1 \leq j \leq T$, the output of the network is $z_j = \psi(h_j)$, where $\psi$ is a linear function. In the present article, we rather consider the following residual version, which is a natural adaptation of classical RNN in the neural ODE framework \citep[see, e.g.,][]{yue2018residual}: 
\begin{equation} \label{eq:residual-rnn}
    h_0 = 0 \quad \text{and} \quad h_{j+1} = h_{j} + \frac{1}{T}  f(h_{j}, x_{j+1}) \quad \textnormal{for } 0 \leq j \leq T-1.
\end{equation}
The simplest choice for the function $f$ is the feedforward model, say $f_{\textnormal{RNN}}$, defined by 
\begin{equation}    \label{eq:def_classical-rnn}
    f_{\textnormal{RNN}}(h, x) = \sigma(U h + V x + b),
\end{equation}
where $\sigma$ is an activation function, $U \in \R^{e \times e}$ and $V \in \R^{e \times d}$ are weight matrices, and $b \in \R^{e}$ is the bias. The function $f_{\textnormal{RNN}}$, equipped with a smooth activation $\sigma$ (such as the logistic or hyperbolic tangent functions), will be our leading example throughout the paper. However, the GRU and LSTM models can also be rewritten under the form \eqref{eq:residual-rnn}, as shown in Appendix \ref{apx:extension_gru_lstm}. Thus, model \eqref{eq:residual-rnn} is flexible enough to encompass most recurrent networks used in practice.

\paragraph{Overview.} Section \ref{sec:kernel} is devoted to framing RNN as linear functions in a suitable RKHS. We start by embedding iteration \eqref{eq:residual-rnn} into a continuous-time model, which takes the form of a controlled differential equation (CDE). This allows, after introducing the signature transform, to define the appropriate RKHS, and, in turn, to show that model \eqref{eq:residual-rnn} boils down, in the continuous-time limit, to a linear problem on the signature. This framework is used in Section \ref{sec:generalization_bound} to derive generalization bounds and stability guarantees. We provide some experiments in Section \ref{sec:experiments} before discussing our results in Section \ref{sec:discussion}. All proofs are postponed to the supplementary material.

\section{Framing RNN as a kernel method} \label{sec:kernel}
\paragraph{Roadmap.} First, we quantify the difference between the discrete recurrent network \eqref{eq:residual-rnn} and its continuous-time counterpart (Proposition \ref{prop:forward_euler}). Then, we rewrite the corresponding ODE as a CDE (Proposition \ref{prop:ode_to_cde}). Under appropriate conditions, Proposition \ref{prop:euler_convergence} shows that the solution of this equation is a linear function of the signature of the driving process. Importantly, these assumptions are valid for a feedforward RNN, as stated by Proposition \ref{prop:bounding_dn_norm_for_rnns}. We conclude in Theorem \ref{thm:rnn_in_H_binary}. 

\subsection{From discrete to continuous time}

Recall that $h_0,\dots, h_T$ denote the hidden states of the RNN \eqref{eq:residual-rnn}, and let $H:[0,1] \to \R^e$ be the solution of the ODE
\begin{equation} \label{eq:residual-rnn-ode}
    dH_t = f(H_t, X_t)dt, \quad H_0 =h_0.
\end{equation}
By bounding the difference between  $H_{\nicefrac{j}{T}}$ and $h_j$, the following proposition shows how to pass from discrete to continuous time, provided $f$ satisfies the following assumption:
\begin{align*}
    (A_1) \quad &\textnormal{The function $f$ is Lipschitz continuous in $h$ and $x$, with Lipschitz constants $K_h$ and $K_x$.} \\
    &\textnormal{We let $K_f = \max(K_h, K_x)$.}
\end{align*}
\begin{proposition}
\label{prop:forward_euler}
Assume that $(A_1)$ is verified. Then there exists a unique solution $H$ to~\eqref{eq:residual-rnn-ode} and, for any $0\leq j\leq T$,
\begin{equation*}
    \|H_{\nicefrac{j}{T}} - h_j \| \leq \frac{c_1}{T},
\end{equation*}
where $c_1 = K_f e^{K_f} \big( L +\underset{\|h \| \leq M, \|x\| \leq L }{\sup}\|f(h,x) \| e^{K_f} \big)$ and $M=\underset{\|x\| \leq L}{\sup} \|f(h_0,x) \| e^{K_f}$. Moreover, for any $t \in [0,1]$, $\|H_t \| \leq M$.
\end{proposition}
Then, following \citet{kidger2020neural}, we show that the ODE \eqref{eq:residual-rnn-ode} can be rewritten under the form of a CDE. At the cost of increasing the dimension of the hidden state from $e$ to $e+d$, this allows us to reframe model \eqref{eq:residual-rnn-ode} as a linear model in $dX$, in the sense that $X$ has been moved `outside' of $f$.
\begin{proposition} \label{prop:ode_to_cde}
    Assume that $(A_1)$ is verified. Let $H:[0,1] \to \R^e$ be the solution of \eqref{eq:residual-rnn-ode}, and let $\bar{X}: [0,1] \to \R^{d+1}$ be the time-augmented process $\bar{X}_t = (X_t^\top, \frac{1-L}{2}t)^\top$. Then there exists a tensor field $\mathbf{F}: \R^{\bar{e}} \to \R^{\bar{e} \times \bar{d}}$, $\bar{e} = e +d$, $\bar{d}=d+1$, such that if $\bar{H}:[0,1] \to \R^{\bar{e}}$ is the solution of the CDE
    \begin{equation}\label{eq:residual-rnn-cde}
          d\bar{H}_t = \mathbf{F}(\bar{H}_t)d\bar{X}_t, \quad \bar{H}_0 = (H_0^\top, X_0^\top)^\top,
    \end{equation}
    then its first $e$ coordinates are equal to $H$.
\end{proposition}
Equation \eqref{eq:residual-rnn-cde} can be better understood by the following equivalent integral equation: 
\[\bar{H}_t = \bar{H}_0 + \int_0^t \mathbf{F}(\bar{H}_u)d\bar{X}_u, \]
where the integral should be understood as Riemann-Stieljes integral \citep[][Section~I.2]{friz2010multidimensional}.  Thus, the output of the RNN can be approximated by the solution of the CDE \eqref{eq:residual-rnn-cde}, and, according to Proposition~\ref{prop:forward_euler}, the approximation error is $\mathscr{O}(\nicefrac{1}{T})$.

\begin{example} \label{exampe:vector-field-rnn}
Consider $f_{\textnormal{RNN}}$ as in \eqref{eq:def_classical-rnn}. If $\sigma$ is Lipschitz continuous with constant $K_\sigma$, then, for any $h_1, h_2 \in \R^e$, $x_1, x_2 \in \R^d$,
\begin{align*}
    \|f_{\textnormal{RNN}}(h_1, x_1) - f_{\textnormal{RNN}}(h_2, x_1)\| &= \|\sigma(U h_1 + V x_1 + b) -\sigma(U h_2 + V x_1 + b) \| \\
    &\leq K_\sigma \|U\|_{\textnormal{op}} \|h_1 - h_2\|,
\end{align*}
where $\|\cdot\|_{\textnormal{op}}$ denotes the operator norm---see Appendix \ref{apx:operator-norm}. Similarly, $ \|f(h_1, x_1) - f(h_1, x_2)\| \leq K_\sigma \|V\|_{\textnormal{op}} \|x_1 - x_2\|$.  Thus, assumption $(A_1)$ is satisfied. The tensor field $\mathbf{F}_{\textnormal{RNN}}$ of Proposition \ref{prop:ode_to_cde} corresponding to this network is defined for any $\bar{h} \in \R^{\Bar{e}}$ by
\begin{equation} \label{eq:rnn-vector_field}
    \mathbf{F}_{\textnormal{RNN}}(\bar{h}) = \begin{pmatrix}0_{e \times d} & \frac{2}{1-L}\sigma (W \bar{h} + b) \\ I_{d\times d} & 0_{d \times 1} \end{pmatrix}, \quad \text{where} \quad W = \begin{pmatrix} U & V \end{pmatrix} \in \R^{e \times \Bar{e}}.
\end{equation}
\end{example}

\subsection{The signature}

An essential ingredient towards our construction is the signature of a continuous-time process, which we briefly present here. We refer to \citet{primer2016} for a gentle introduction and to \citet{lyons2007differential,levin2013learning} for details. 

\paragraph{Tensor Hilbert spaces.} We denote by $(\R^d)^{\otimes k}$ the $k$th tensor power of $\R^d$ with itself, 
which is a Hilbert space of dimension $d^k$. The key space to define the signature and, in turn, our RKHS, consists in infinite square-summable sequences of tensors of increasing order:
\begin{equation}\label{eq:def_T}
    \mathscr{T} = \Big\{ a = (a_0, \dots, a_k, \dots ) \, \Big| \, a_k \in (\R^d)^{\otimes k}, \, \sum_{k=0}^{ \infty} \| a_k \|_{(\R^d)^{\otimes k}}^2 <  \infty \Big\}.
\end{equation}
Endowed with the scalar product $\langle a,b\rangle_{\mathscr{T}} := \sum_{k=0}^{ \infty} \langle a_k, b_k \rangle_{(\R^d)^{\otimes k}}$, $\mathscr{T}$ is a Hilbert space, as shown in Appendix \ref{apx:tensor_hilbert_space}.
\begin{definition}
Let $X \in BV^{c}([0,1], \R^d)$. For any $t \in [0,1]$, the signature of $X$ on $[0,t]$ is defined by
$S_{[0,t]}(X) = (1,\mathbb{X}^1_{[0,t]},\dots,\mathbb{X}^k_{[0,t]},\dots )$,
where, for each $k \geq 1$,	
\[ \mathbb{X}^k_{[0,t]} = k! \idotsint\limits_{ 0\leq u_1 <  \cdots <u_k \leq t } dX_{u_1}\otimes \dots \otimes dX_{u_k} \in (\R^d)^{\otimes k}.\]
\end{definition}
Although this definition is technical, the signature should simply be thought of as a feature map that embeds a bounded variation process into an infinite-dimensional tensor space.  The signature has several good properties that make it a relevant tool for machine learning \citep[e.g.,][]{levin2013learning,primer2016,fermanian2021embedding}. In particular, under certain assumptions, $S(X)$ characterizes $X$ up to translations and reparameterizations, and has good approximation properties. We also highlight that fast libraries exist for computing the signature \citep{reizenstein2018iisignature,signatory}.

The expert reader is warned that this definition differs from the usual one by the normalization of $\mathbb{X}^k_{[0,t]}$ by $k!$, which is more adapted to our context. In the sequel, for any index $(i_1,\dots,i_k) \subset \{1,\dots, d\}^k$, $S_{[0,t]}^{(i_1, \dots, i_k)}(X)$ denotes the term associated with the coordinates $(i_1, \dots, i_k)$ of $\mathbb{X}^k_{[0,t]}$. When the signature is taken on the whole interval $[0,1]$, we simply write $S(X)$, $S^{(i_1, \dots, i_k)}(X)$, and $\mathbb{X}^k$.

\begin{example}
	Let $X$ be the $d$-dimensional linear path defined by $X_t= ( a_1 + b_1t, \dots, a_d + b_dt)^\top$, $a_i,b_i \in \R$. Then $S^{(i_1,\dots ,i_k)}(X)=b_{i_1} \cdots b_{i_k} $ and $\mathbb{X}^k = b^{\otimes k}$.
\end{example}
The next proposition, which ensures that $S_{[0,t]}(\bar{X}) \in \mathscr{T}$, is an important step.
\begin{proposition} \label{prop:up_bound_norm_sig}
	Let $X\in \mathscr{X}$ and $\Bar{X}_t=(X_t^\top, \frac{1-L}{2}t)^\top$ as in Proposition \ref{prop:ode_to_cde}. Then, for any $t \in [0,1]$, $\|S_{[0,t]}(\Bar{X})\|_{ \mathscr{T}} \leq 2(1 - L)^{-1}$.
\end{proposition}

\paragraph{The signature kernel.}
By taking advantage of the structure of Hilbert space of $\mathscr{T}$, it is natural to introduce the following kernel:
\begin{align*}\label{eq:def_truncated_signature_kernel}
   K : \mathscr{X} \times \mathscr{X} &\to \R \\
   (X, Y) &\mapsto \langle S(\bar{X}), S(\bar{Y}) \rangle_{\mathscr{T}},
\end{align*}
which is well defined according to Proposition \ref{prop:up_bound_norm_sig}. We refer to \citet{Kiraly2016} for a general presentation of kernel methods with signatures and to \citet{cass2020computing} for a kernel trick. The RKHS associated with $K$ is the space of functions
\begin{equation}\label{eq:def_signature_kernel}
\mathscr{H}= \big\{\xi_{\alpha}: \mathscr{X} \to \R \, | \, \xi_{\alpha}(X)= \langle \alpha, S(\bar{X}) \rangle_{\mathscr{T}}, \alpha \in \mathscr{T} \big\},
\end{equation}
with scalar product $\langle \xi_{\alpha}, \xi_{{\beta}} \rangle_{\mathscr{H}} = \langle \alpha, {\beta} \rangle_{\mathscr{T}}$ \citep[see, e.g.,][]{scholkopf2002learning}.


\subsection{From the CDE to the signature kernel} 

An important property of signatures is that the solution of the CDE \eqref{eq:residual-rnn-cde} can be written, under certain assumptions, as a linear function of the signature of the driving process $X$. This operation can be thought of as a Taylor expansion for CDE. More precisely, let us rewrite \eqref{eq:residual-rnn-cde} as
\begin{equation} \label{eq:cde-vector-field}
dH_t = \mathbf{F}(H_t)dX_t = \sum_{i=1}^d F^{i}(H_t) dX^i_{t},
\end{equation}
where $X_t = (X^1_t, \dots, X^d_t)^\top$, $\mathbf{F}:\R^e \to \R^{e \times d}$, and $F^{i}: \R^e \to \R^e$ are the columns of $\mathbf{F}$---to avoid heavy notation, we momentarily write $e$, $d$, $H$, and $X$ instead of $\bar{e}$, $\bar{d}$, $\bar{H}$, and $\bar{X}$. 
Throughout, the bold notation is used to distinguish tensor fields and vector fields. We recall that a vector field $F:\R^e \to \R^e$ or a tensor field $\mathbf{F}:\R^e \to \R^{e \times d}$ are said to be smooth if each of their coordinates is $\mathscr{C}^\infty$.
\begin{definition}
Let $F, G: \R^e \to \R^e$ be smooth vector fields and denote by $J(\cdot)$ the Jacobian matrix. Their differential product is the smooth vector field $F \star G: \R^e \to \R^e$  defined, for any $h \in \R^e$, by
\begin{equation*}
(F \star G)(h) = \sum_{j=1}^e \frac{\partial G}{\partial h_j}(h)  F_j(h) = J(G)(h) F(h).
\end{equation*}
\end{definition}
In differential geometry, $F \star G$ is simply denoted by $FG$. Since the~$\star$ operation is not associative, we take the convention that it is evaluated from right to left, i.e., $F^1 \star F^2 \star F^3 := F^1 \star (F^2 \star F^3)$.

\paragraph{Taylor expansion.}Let $H$ be the solution of \eqref{eq:cde-vector-field}, where $\mathbf{F}$ is assumed to be smooth. We now show that $H$ can be written as a linear function of the signature of $X$, which is the crucial step to embed the RNN in the RKHS $\mathscr{H}$.
The step-$N$ Taylor expansion of $H$ \citep{friz2008euler} is defined by 
\begin{align*} \label{eq:taylor-expansion}
H^N_t &= H_0 + \sum_{k=1}^N \frac{1}{k!} \sum_{1 \leq i_1, \dots, i_k \leq d} S^{(i_1, \dots, i_k)}_{[0,t]}(X) F^{i_1}  \star \cdots \star F^{i_k} (H_0).
\end{align*}
Throughout, we let 
\begin{equation*}
    \Lambda_k(\mathbf{F})
    =  \sup_{\|h\| \leq M, 1 \leq i_1, \dots, i_k \leq d} \| F^{i_1} \star \dots \star F^{i_k}(h)\|.
\end{equation*}
\begin{example}\label{ex:star_product_linear_rnn}
 Let $\mathbf{F} = \mathbf{F}_{\textnormal{RNN}}$ defined by \eqref{eq:rnn-vector_field} with an identity activation. Then, for any $\Bar{h} \in \R^{\Bar{e}}$, $1 \leq i \leq d +1$, $F_{\textnormal{RNN}}^i(\Bar{h}) =  W_i \Bar{h} + b_i$, where 
$b_i$ is the $(i+d)$th vector of the canonical basis of $\R^{\Bar{e}}$, and
\[    
   W_i = 0_{\Bar{e} \times \Bar{e}}, \quad W_{d+1} = \begin{pmatrix} \frac{2}{1-L} W \\ 0_{d\times \Bar{e}}\end{pmatrix}, \quad \text{and} \quad b_{d+1} =  \begin{pmatrix} \frac{2}{1-L} b \\ 0_{d}\end{pmatrix}.
\]
The vector fields $F^{i}_{\textnormal{RNN}}$ are then affine, $J(F_{\textnormal{RNN}}^i) = W_i$, and the iterated star products have a simple expression: for any $1 \leq i_1, \dots, i_k \leq d$, $F_{\textnormal{RNN}}^{i_1} \star \dots \star F_{\textnormal{RNN}}^{i_k}(\Bar{h}) = W_{i_k} \cdots W_{i_2} (W_{i_1}\Bar{h} + b_{i_1}).$
\end{example}
The next proposition shows that the step-$N$ Taylor expansion $H^N$ is a good approximation of $H$.
\begin{proposition} \label{prop:euler_convergence}
Assume that the tensor field $\mathbf{F}$ is smooth. Then, for any $t \in [0,1]$,
\begin{equation} \label{eq:bound_euler_cvg}
\|H_t - H^N_t \| \leq  \frac{d^{N+1}}{(N+1)!}  \Lambda_{N+1}(\mathbf{F}).
\end{equation}
\end{proposition}
Thus, provided that $ \Lambda_{N}(\mathbf{F})$ is not too large, the right-hand side of \eqref{eq:bound_euler_cvg} converges to zero, hence
\begin{equation} \label{eq:H_infinite_sum_signature}
    H_t = H_0 + \sum_{k=1}^\infty \frac{1}{k!} \sum_{1 \leq i_1, \dots, i_k \leq d} S^{(i_1, \dots, i_k)}_{[0,t]}(X) F^{i_1}  \star \cdots \star F^{i_k} (H_0).
\end{equation}
We conclude from the above representation that the solution $H$ of \eqref{eq:cde-vector-field} is in fact a linear function of the signature of $X$. A natural concern is to know whether the upper bound of Proposition \ref{prop:euler_convergence} vanishes with $N$ for standard architectures. This property is encapsulated in the following more general assumption:
\begin{equation*}
    (A_2) \quad \textnormal{The tensor field $\mathbf{F}$ is smooth and $\sum_{k=0}^\infty \Big(\frac{d^k}{k!} \Lambda_{k}(\mathbf{F})  \Big)^2 < \infty$.}
\end{equation*}
Clearly, if $(A_2)$ is verified, then the right-hand side of \eqref{eq:bound_euler_cvg} converges to 0.  The next proposition states formally the conditions under which $(A_2)$ is verified for $\mathbf{F}_{\textnormal{RNN}}$. It is further illustrated in Figure \ref{fig:euler_convergence}, which shows that the convergence is fast with two common activation functions. We let $\| \sigma\|_\infty = \sup_{\|h\| \leq M, \|x\| \leq L} \|\sigma(U h + V x + b)\|$ and $\| \sigma^{(k)}\|_\infty = \sup_{\|h\| \leq M, \|x\| \leq L} \|\sigma^{(k)}(U h + V x + b)\|$, where $\sigma^{(k)}$ is the derivative of order $k$ of $\sigma$.
\begin{proposition} \label{prop:bounding_dn_norm_for_rnns}
   Let $\mathbf{F}_{\text{RNN}}$ be defined by \eqref{eq:rnn-vector_field}. If $\sigma$ is the identity function, then $(A_2)$ is satisfied. In the general case,  $(A_2)$ holds if $\sigma$ is smooth and there exists $a > 0$ such that, for any $k \geq 0$,
   \begin{equation} \label{eq:condition_activation_function}
        \| \sigma^{(k)}\|_{\infty} \leq a^{k+1} k!
    \quad
    \textnormal{and}
    \quad 
       \|W\|_F < \frac{1-L}{8a^2d},
   \end{equation}
   where $\|\cdot\|_F$ is the Frobenius norm. Moreover, 
    $\Lambda_{N}(\mathbf{F}_{\text{RNN}}) \leq \sqrt{2} a \Big(\frac{8a^2\|W\|_{F}}{1-L} \Big)^{N-1} N! \,.$
\end{proposition}
The proof of Proposition \ref{prop:bounding_dn_norm_for_rnns}, based on the manipulation of higher-order derivatives of tensor fields, is highly non-trivial. We highlight that the conditions on $\sigma$ are mild and verified for common smooth activations. For example, they are verified for the logistic function (with $a=2$) and for the hyperbolic tangent function (with $a=4$)---see Appendix \ref{apx:bounding-derivatives-activations}.
The second inequality of \eqref{eq:condition_activation_function} puts a constraint on the norm of the weights, and can be regarded as a radius of convergence for the Taylor expansion. 

\paragraph{Putting everything together.} We now have all the elements at hand to embed the RNN into the RKHS $\mathscr{H}$. To fix the idea, we assume in this paragraph that we are in a $\pm 1$ classification setting. In other words, given an input sequence $\mathbf{x}$, we are interested in the final output $z_T = \psi(h_T) \in \R$, where $h_T$ is the solution of \eqref{eq:residual-rnn}. The predicted class is $2 \cdot \mathbf{1}(z_T > 0) - 1$. 

By Propositions \ref{prop:forward_euler} and \ref{prop:ode_to_cde}, $z_T$ is approximated by the first $e$ coordinates of the solution of the CDE \eqref{eq:residual-rnn-cde}, which outputs a $\R^{e+d}$-valued process $\Bar{H}$. According to Proposition \ref{prop:euler_convergence}, $\Bar{H}$ is a linear function of the signature of the time-augmented process $\Bar{X}$. Thus, on top of $\Bar{H}$, it remains to successively apply the projection $\textnormal{Proj}$ on the $e$ first coordinates followed by the linear function $\psi$ to obtain an element of the RKHS $\mathscr{H}$. This mechanism is summarized in the following theorem.
\begin{theorem}\label{thm:rnn_in_H_binary}
    Assume that $(A_1)$ and $(A_2)$ are verified. Then there exists a function $\xi_\alpha \in \mathscr{H}$ such that
        \begin{equation} \label{eq:rkhs-embedding-binary-rnn}
       | z_T-\xi_\alpha(X)|  \leq \|\psi\|_{\textnormal{op}} \frac{c_1}{T},
    \end{equation}
    where $\xi_\alpha(X) = \langle \alpha, S(\bar{X}) \rangle_{\mathscr{T}}$ and $\bar{X}_t = (X_t^\top, \frac{1-L}{2} t)^\top$. We have $\alpha = (\alpha_k)_{k=0}^{\infty}$, where each $\alpha_k \in (\R^d)^{\otimes k}$ is defined by 
    \begin{equation*}
        \alpha_k^{(i_1, \dots, i_k)} = \frac{1}{k!} \psi \circ \textnormal{Proj} \big( F^{i_1} \star \cdots \star F^{i_k}(\bar{H_0}) \big).
    \end{equation*}
    Moreover, $\|\alpha\|^2_{\mathscr{T}} \leq \|\psi\|_{\textnormal{op}}^2 \sum_{k=0}^\infty \Big(\frac{d^k}{k!} \Lambda_k(\mathbf{F}) \Big)^2.$
\end{theorem}
We conclude that in the continuous-time limit, the output of the network can be interpreted as a scalar product between the signature of the (time-augmented) process $\bar{X}$ and an element of $\mathscr{T}$. This interpretation is important for at least two reasons: $(i)$ it facilitates the analysis of generalization of RNN by leveraging the theory of kernel methods, and $(ii)$ it provides new insights on regularization strategies to make RNN more robust. These points will be explored in the next section. Finally, we stress that the approach works for a large class of RNN, such as GRU and LSTM. The derivation of conditions $(A_1)$ and $(A_2)$ beyond the feedforward RNN is left for future work. 

\section{Generalization and regularization}\label{sec:generalization_bound}

\subsection{Generalization bounds}

\paragraph{Learning procedure.} A first consequence of framing a RNN as a kernel method is that it gives natural generalization bounds under mild assumptions. In the learning setup, we are given an i.i.d.~sample $\mathscr{D}_n$ of $n$ random pairs of observations $(\mathbf{x}^{(i)}, \mathbf{y}^{(i)}) \in (\R^d)^{T} \times \mathscr{Y}$, where $\mathbf{x}^{(i)} = (x^{(i)}_{1}, \dots, x^{(i)}_{T})$. We distinguish the binary classification problem, where $\mathscr{Y} = \{-1, 1\}$, from the sequential prediction problem, where $\mathscr{Y} = (\R^p)^T$ and $\mathbf{y}^{(i)} = (y^{(i)}_{1}, \dots, y^{(i)}_{T})$. The RNN is assumed to be parameterized by $\theta \in \Theta \subset \R^q$, where $\Theta$ is a compact set. To clarify the notation, we use a $\theta$ subscript whenever a quantity depends on $\theta$ (e.g., $f_\theta$ for $f$, etc.). In line with Section \ref{sec:kernel}, it is assumed that the tensor field $\mathbf{F}_\theta$ associated with $f_\theta$ satisfies $(A_1)$ and $(A_2)$, keeping in mind that Proposition \ref{prop:bounding_dn_norm_for_rnns} guarantees that these requirements are fulfilled by a feedforward recurrent network with a smooth activation function. 

Let $g_{\theta} : (\R^d)^{T} \to \mathscr{Y}$ denote the output of the recurrent network. The parameter $\theta$ is fitted by empirical risk minimization using a loss function $\ell: \mathscr{Y} \times \mathscr{Y} \to \R^{+}$. The theoretical and empirical risks are respectively defined, for any $\theta \in \Theta$, by
\begin{equation*}
   \mathscr{R}(\theta) = \esp[\ell(\mathbf{y}, g_{\theta}(\mathbf{x}))] \quad \text{and} \quad \widehat{\mathscr{R}}_{n}(\theta) = \frac{1}{n} \sum_{i=1}^n \ell \big(\mathbf{y}^{(i)}, g_{\theta}(\mathbf{x}^{(i)}) \big),
\end{equation*}
where the expectation $\esp$ is evaluated with respect to the distribution of the generic random pair $(\mathbf{x}, \mathbf{y})$. We let $\widehat{\theta}_n \in \textnormal{argmin }_{\theta \in \Theta } \widehat{\mathscr{R}}_n (\theta)$ and aim at upper bounding $\prob(\mathbf{y} \neq g_{\widehat{\theta}_n}(\mathbf{x}))$ in the classification regime (Theorem \ref{thm:generalization_bound_binary}) and $\mathscr{R}(\widehat{\theta}_n)$ in the sequential regime (Theorem \ref{thm:generalization_bound_sequential}). To reach this goal, our strategy is to approximate the RNN by its continuous version and then use the RKHS machinery of Section \ref{sec:kernel}.

\paragraph{Binary classification.}
In this context, the network outputs a real number $g_{\theta}(\mathbf{x})=\psi(h_T) \in \R$ and the predicted class is $2 \cdot \mathbf{1}(g_{\theta}(\mathbf{x}) > 0) - 1$. The loss $\ell:\R \times \R \to \R^{+}$ is assumed to satisfy the assumptions of \citet[][Theorem 7]{bartlett2002rademacher}, that is, for any $y \in \{-1,1\}$, $\ell(\mathbf{y}, g_{\theta}(\mathbf{x})) = \phi(\mathbf{y}g_{\theta}(\mathbf{x}))$, where $\phi(u) \geq \mathbf{1}(u\leq 0)$, and $\phi$ is Lipschitz-continuous with constant $K_{\ell}$. For example, the logistic loss satisfies such assumptions. We let $\xi_{\alpha_\theta} \in \mathscr{H}$ be the function of Theorem \ref{thm:rnn_in_H_binary} that approximates the RNN with parameter $\theta$. Thus, $z_T \approx \xi_{\alpha_\theta}(\Bar{X}) = \langle \alpha_\theta, S(\Bar{X}) \rangle_{\mathscr{T}}$, up to a $\mathscr{O}(\nicefrac{1}{T})$ term.

\begin{theorem} \label{thm:generalization_bound_binary}
    Assume that for all $\theta \in \Theta$, $(A_1)$ and $(A_2)$ are verified. Assume, in addition, that there exists a constant $B>0$ such that for any $ \theta \in \Theta$, $\|\xi_{\alpha_\theta}\|_{\mathscr{H}} \leq B$. Then with probability at least $1-\delta$,
    \begin{equation}\label{eq:gen_bound_rnn_binary}
        \prob\big(\mathbf{y} \neq g_{\widehat{\theta}_n}(\mathbf{x}) | \mathscr{D}_n \big) \leq \widehat{\mathscr{R}}_n(\widehat{\theta}_n) + \frac{c_2}{T} + \frac{8 B K_{\ell} }{(1-L)\sqrt{n}}  + \frac{2BK_{\ell}}{1-L} \sqrt{\frac{\log(\nicefrac{1}{\delta})}{2n}},
    \end{equation}
    where $c_2 =  K_{\ell} \sup_\theta \Big(\|\psi \|_{\textnormal{op}}  K_{f_\theta} e^{K_{f_\theta}} \big( L +\|f_\theta\|_\infty e^{K_{f_\theta}} \big) \Big)$.
    
\end{theorem}
Close to our result are the bounds obtained by \citet{zhang2018stabilizing}, \citet{tu2019understanding}, and \citet{chen2020generalization}. 
The main difference is that the term in $\nicefrac{1}{T}$ does not usually appear, since it comes from the Euler discretization error, whereas the speed in $\nicefrac{1}{\sqrt{n}}$ is the same. For instance, \citet{chen2020generalization} show that, under some assumptions, the excess risk is of order $\sqrt{de + e^2} T^\alpha K_\ell n^{-1/2}$. We refer to Section \ref{sec:discussion} for further discussion on the dependency of the different bounds to the parameter $T$. The take-home message is that the detour by continuous-time neural ODE provides a theoretical framework adapted to RNN, at the modest price of an additional $\mathscr{O}(\nicefrac{1}{T})$ term. Moreover, we note that the bound \eqref{eq:gen_bound_rnn_binary} is `simple' and holds under mild conditions for a large class of RNN. More precisely, for any recurrent network of the form \eqref{eq:residual-rnn}, provided $(A_1)$ and $(A_2)$ are satisfied, then \eqref{eq:gen_bound_rnn_binary} is valid with constants $c_2$ and $B$ depending on the architecture. Such constants are given below in the example of a feedforward RNN. We stress that Theorem \ref{thm:generalization_bound_binary} can be extended without significant effort to the multi-class classification task, with an appropriate choice of loss function.

\begin{example}
  Take a feedforward RNN with logistic activation, and $\Theta = \{(W, b, \psi) \, | \, \|W\|_F \leq K_W < (1-L)/32d, \|b\| \leq K_b, \|\psi\|_{\textnormal{op}} \leq K_\psi \}$. Then, Proposition \ref{prop:bounding_dn_norm_for_rnns} states that $(A_2)$ is satisfied and, with Theorem \ref{thm:rnn_in_H_binary}, ensures that
\begin{align*}
     \underset{\theta \in \Theta}{\sup}\| \xi_{\alpha_\theta} \|_{\mathscr{H}} \leq  \frac{\sqrt{2} K_\psi(1-L)}{1-L-32dK_W} :=B, \quad K_{f_{\theta}} = \max(\|U\|_{\textnormal{op}}, \|V\|_{\textnormal{op}}), \quad \text{and} \quad \|f_\theta\|_\infty = 1.
\end{align*}
\end{example}

\paragraph{Sequence-to-sequence learning.} We conclude by showing how to extend both the RKHS embedding of Theorem \ref{thm:rnn_in_H_binary} and the generalization bound of Theorem \ref{thm:generalization_bound_binary} to the setting of sequence-to-sequence learning. In this case, the output of the network is a sequence
\begin{equation*}
    g_{\theta}(\mathbf{x}) = (z_1,\dots, z_T) \in (\R^p)^T.
\end{equation*}
An immediate extension of Theorem \ref{thm:rnn_in_H_binary} ensures that there exist $p$ elements $\alpha_{1,\theta}, \dots, \alpha_{p,\theta} \in \mathscr{T}$ such that, for any $1 \leq j \leq T$, 
\begin{equation}\label{eq:rnn_kernel_embedding_sequential}
    \big\|z_j - \big(\langle\alpha_{1,\theta},S_{[0,\nicefrac{j}{T}]}(\bar{X}) \rangle_{\mathscr{T}}, \dots, \langle\alpha_{p,\theta},S_{[0,\nicefrac{j}{T}]}(\bar{X}) \rangle_{\mathscr{T}}\big)^\top \big\| \leq \|\psi\|_{\textnormal{op}} \frac{c_1}{T}.
\end{equation}
The properties of the signature guarantee that $S_{[0,\nicefrac{j}{T}]}(X) = S(\Tilde{X}_{[j]})$ where $\Tilde{X}_{[j]}$ is the process equal to $\bar{X}$ on $[0, \nicefrac{j}{T}]$ and then constant on $[\nicefrac{j}{T}, 1]$---see Appendix \ref{apx:properties_signature}. With this trick, we have, for any $1 \leq \ell \leq p$, $ \langle\alpha_{\ell,\theta},S_{[0,\nicefrac{j}{T}]}(\bar{X}) \rangle_{\mathscr{T}} = \langle\alpha_{\ell,\theta},S(\Tilde{X}_{[j]}) \rangle_{\mathscr{T}}$, so that we are back in $\mathscr{H}$. Observe that the only difference with \eqref{eq:rkhs-embedding-binary-rnn} is that we consider vector-valued sequential outputs, which requires to introduce the process $\Tilde{X}_{[j]}$, but that the rationale is exactly the same. 

We let $\ell:(\R^p)^T \times (\R^p)^T \to \R^+$ be the $L_2$ distance, that is, for any $\mathbf{y} = (y_1, \dots, y_T)$, $\mathbf{y'} = (y_1', \dots, y_T')$, $\ell(\mathbf{y}, \mathbf{y'}) = \frac{1}{T} \sum_{j=1}^T \|y_j - y_j'\|^2$. It is assumed that $\mathbf{y}$ takes its values in a compact subset of $\R^q$, i.e., there exists $K_y > 0$ such that $\| y_j \| \leq K_y$.
\begin{theorem} \label{thm:generalization_bound_sequential}
    Assume that for all $\theta \in \Theta$, $(A_1)$ and $(A_2)$ are verified. Assume, in addition, that there exists a constant $B>0$ such that for any $1 \leq \ell \leq p$, $\theta \in \Theta$, $\|\xi_{\alpha_{\ell,\theta}} \|_{\mathscr{H}} \leq B$ . Then with probability at least $1-\delta$,
    \begin{equation} \label{eq:gen_bound_rnn_sequential}
         \mathscr{R}(\widehat{\theta}_n) \leq \widehat{\mathscr{R}}_n(\widehat{\theta}_n) + \frac{c_3}{T} + \frac{4 pc_4 B(1-L)^{-1}}{\sqrt{n}} +   \sqrt{\frac{2c_5 \log(\nicefrac{1}{\delta})}{n}},
    \end{equation}
    where  $c_3 = \underset{\theta}{\sup}\big(c_{1, \theta} +\|\psi\|_{\textnormal{op}}\|f_\theta\|_{\infty}\big) + 2\sqrt{p}B(1-L)^{-1} +2K_y$, $c_4=B(1-L)^{-1} + K_y$, and $c_5 =  4pB(1-L)^{-1}c_4 + K_y^2$.
\end{theorem}

\subsection{Regularization and stability} \label{subsec:regularization}

In addition to providing a sound theoretical framework, framing deep learning in an RKHS provides a natural norm, which can be used for regularization, as shown for example in the context of convolutional neural networks by \citet{bietti2019kernel}. This regularization ensures stability of predictions, which is crucial in particular in a small sample regime or in the presence of adversarial examples \citep{gao2018black,ko2019popqorn}. In our binary classification setting, for any inputs $\mathbf{x}, \mathbf{x}' \in (\R^d)^{T}$, by the Cauchy-Schwartz inequality, we have
\begin{align*}
    \|z_T - z_T'\| & \leq 2\|\psi\|_{\textnormal{op}}\| \frac{c_1}{T} + \|\xi_{\alpha_\theta}(\Bar{X}) - \xi_{\alpha_\theta}(\Bar{X}')\| \leq 2\|\psi\|_{\textnormal{op}}\| \frac{c_1}{T} +  \|\xi_{\alpha_{\theta}}\|_{\mathscr{H}} \|S(\bar{X}) - S(\bar{X}')\|_{\mathscr{T}}.
\end{align*}
If $\mathbf{x}$ and $\mathbf{x}'$ are close, so are their associated continuous processes $X$ and $X'$ (which can be approximated for example by taking a piecewise linear interpolation), and so are their signatures. The term $\|S(\bar{X}) - S(\bar{X}')\|_{\mathscr{T}}$ is therefore small \citep[][Proposition 7.66]{friz2010multidimensional}. Therefore, when $T$ is large, we see that the magnitude of $ \|\xi_{\alpha_{\theta}}\|_{\mathscr{H}}$ determines how close the predictions are. A natural training strategy to ensure stable predictions, for the types of networks covered in the present article, is then to penalize the problem by minimizing the loss $\widehat{\mathscr{R}}_n(\theta) + \lambda \|\xi_{\alpha_{\theta}}\|^2_{\mathscr{H}}$.
From a computational point of view, it is possible to compute the norm in $\mathscr{H}$, up to a truncation at $N$ of the Taylor expansion, which we know by Proposition \ref{prop:euler_convergence} to be reasonable. It remains that computing this norm is a non-trivial task, and implementing smart surrogates is an interesting problem for the future. Note however that computing the signature of the data is not necessary for this regularization strategy.


\begin{figure}[ht]
     \centering
     \begin{subfigure}[t]{0.45\textwidth}
         \centering
         \includegraphics[width=\textwidth]{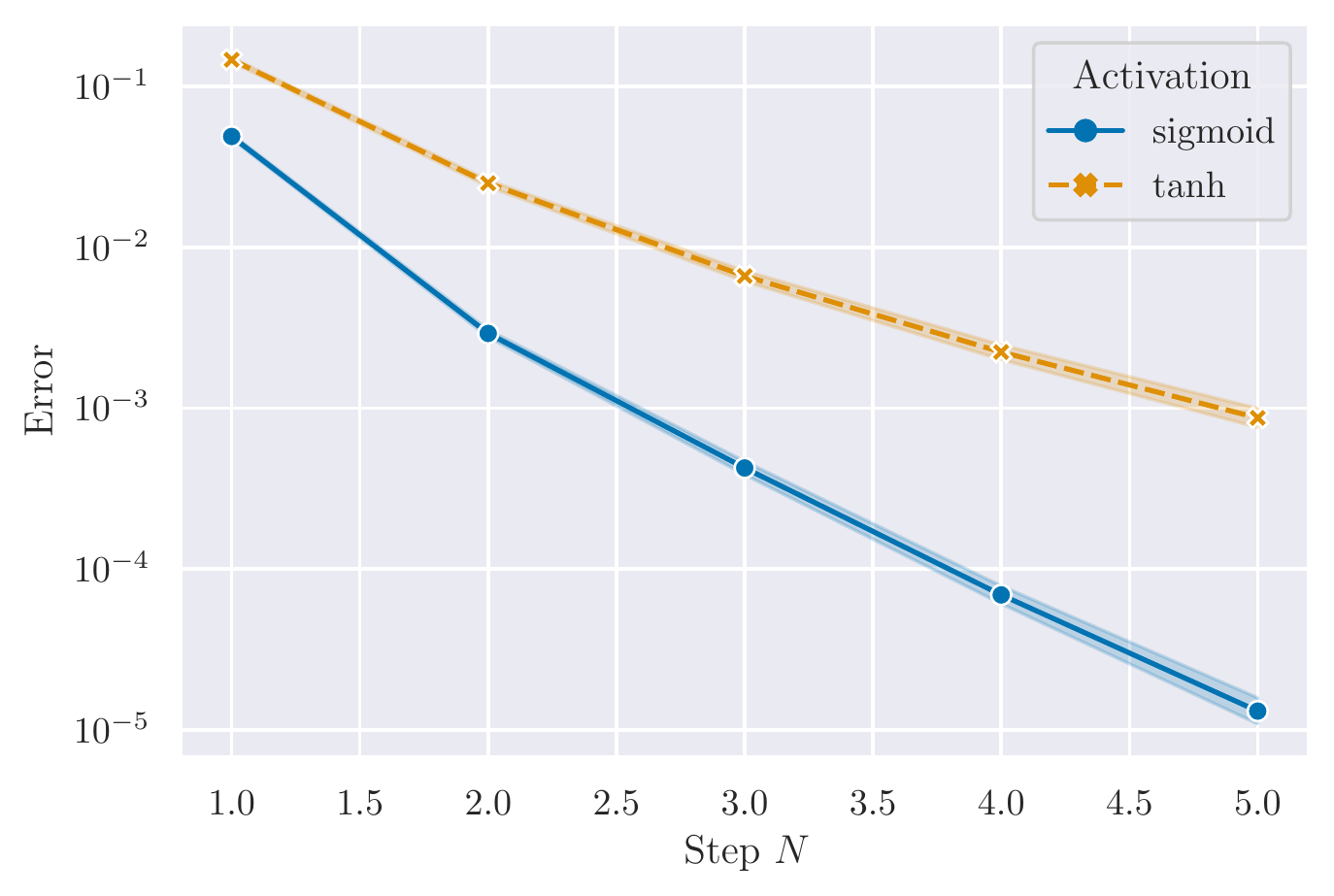}
         \caption{Error on a logarithmic scale as a function of $N$}
         \label{fig:euler_convergence_error}
     \end{subfigure}
     \hfill
     \begin{subfigure}[t]{0.45\textwidth}
         \centering
         \includegraphics[width=\textwidth]{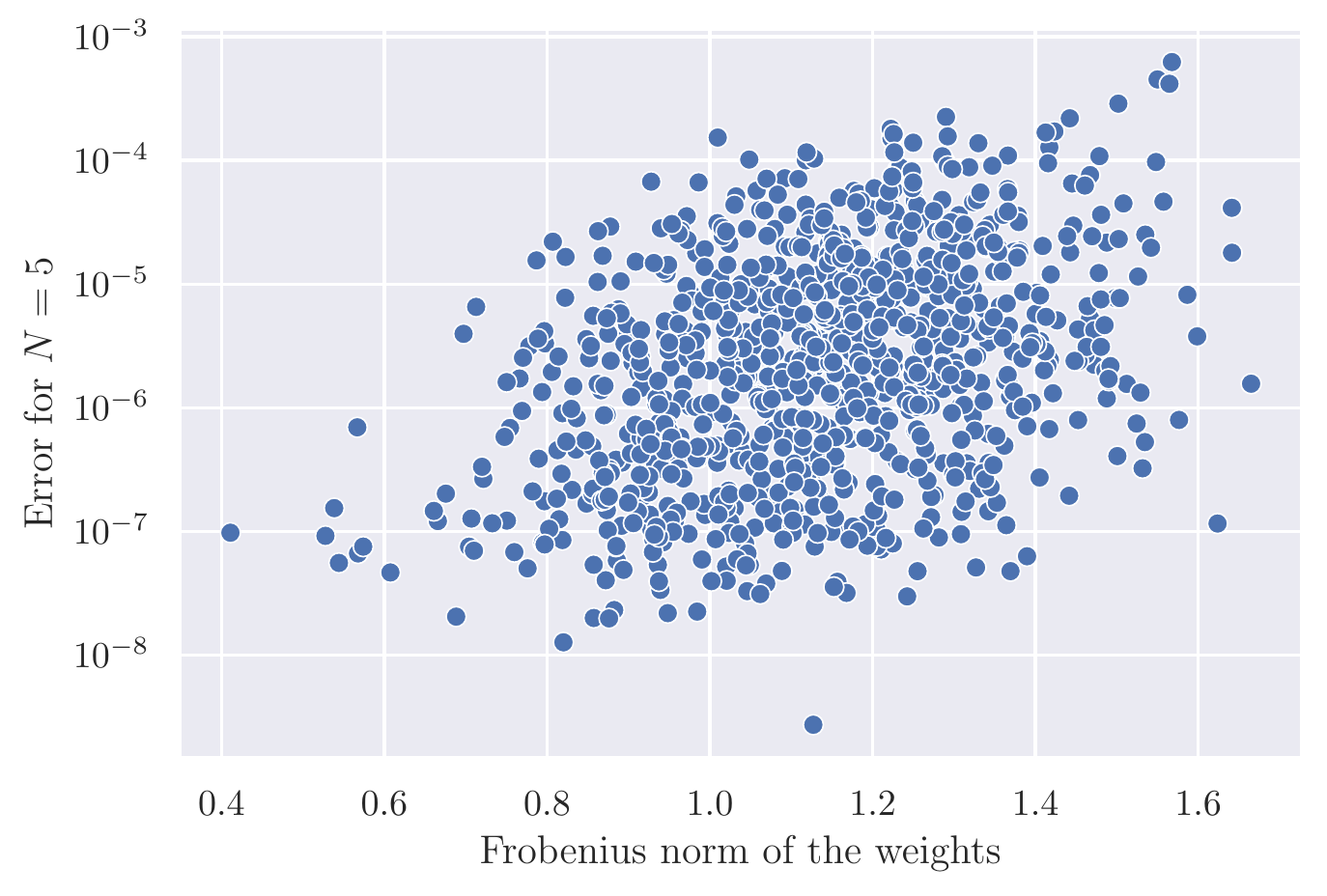}
         \caption{Error as a function of the norm of the weights}
         \label{fig:euler_convergence_scatter_plot}
     \end{subfigure}
    \caption{Approximation of the RNN ODE by the step-$N$ Taylor expansion}
    \label{fig:euler_convergence}
\end{figure}

\section{Numerical illustrations} \label{sec:experiments}

This section is here for illustration purposes. Our objective is not to achieve competitive performance, but rather to illustrate the theoretical results. We refer to Appendix \ref{apx:supp_experiments} for implementation details.

\paragraph{Convergence of the Taylor expansion towards the solution of the ODE.} We illustrate Proposition \ref{prop:euler_convergence} on a toy example. The process $X$ is a 2-dimensional spiral, and we take feedforward RNN with 2 hidden units. Repeating this procedure with $10^3$ uniform random weight initializations, we observe in Figure \ref{fig:euler_convergence_error} that the signature approximation converges exponentially fast in $N$. As seen in Figure \ref{fig:euler_convergence_scatter_plot}, the rate of convergence depends in particular on the norm of the weight matrices, as predicted by Proposition~\ref{prop:bounding_dn_norm_for_rnns}. However, condition~\eqref{eq:condition_activation_function} seems to be over-restrictive, since convergence happens even for weights with norm larger than the bound (we have $1/(8a^2d) \simeq 0.01$ here). 


\begin{figure}[ht]
    \centering
    \includegraphics[width=0.45\textwidth]{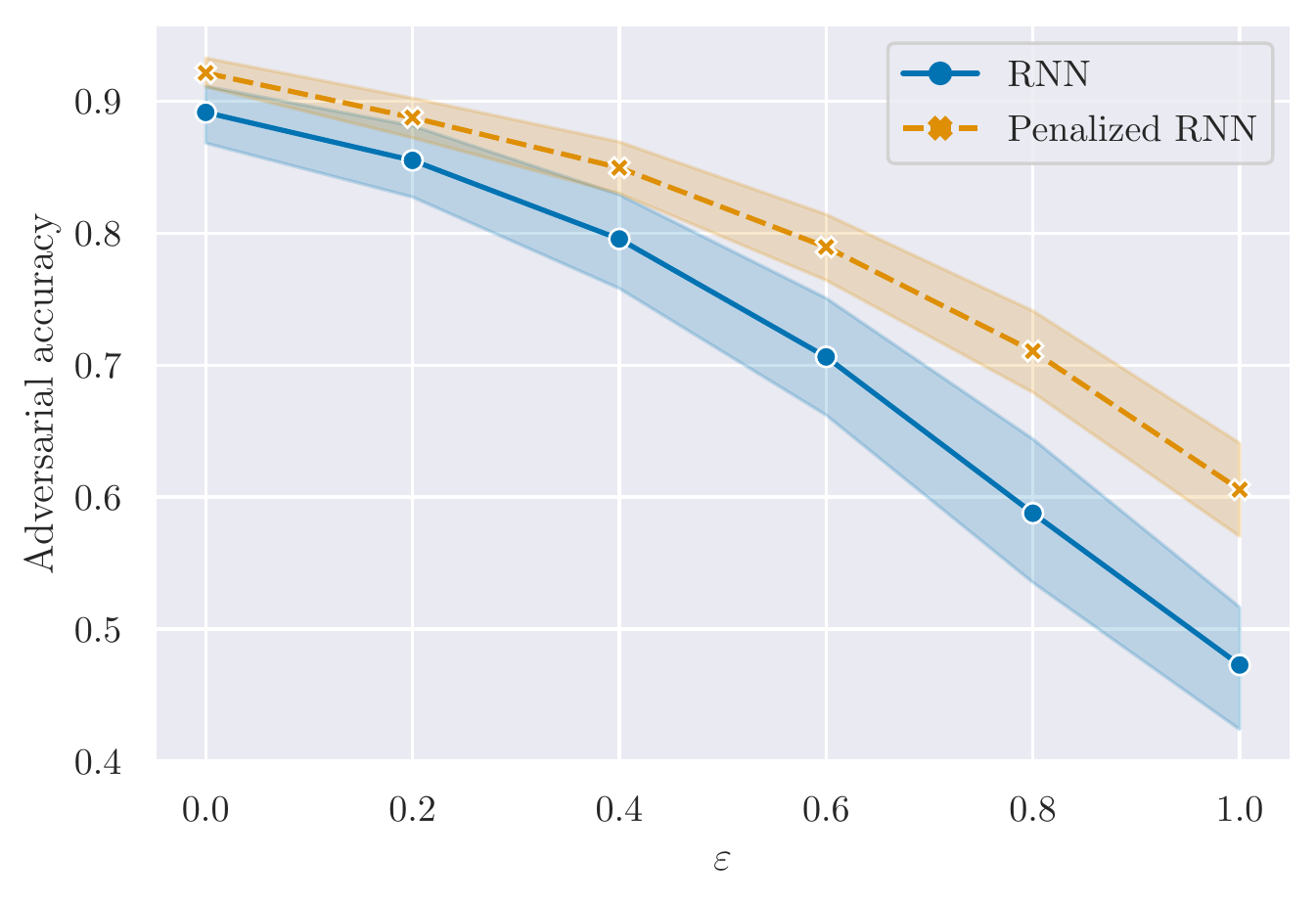}
    \caption{Adversarial accuracy as a function of the adversarial perturbation $\varepsilon$}
    \label{fig:adversarial-robustness}
\end{figure}

\paragraph{Adversarial robustness.} We illustrate the penalization proposed in Section  \ref{subsec:regularization} on a toy task that consists in classifying the rotation direction of  2-dimensional spirals. We take a feedforward RNN with 32 hidden units and hyperbolic tangent activation. It is trained on 50 examples, with and without penalization, for 200 epochs. Once trained, the RNN is tested on adversarial examples, generated with the projected gradient descent algorithm with Frobenius norm \citep{madry2018towards}, which modifies test examples to maximize the error while staying in a ball of radius $\varepsilon$. We observe in Figure \ref{fig:adversarial-robustness} that adding the penalization seems to make the network more stable.

\paragraph{Comparison of the trained networks.} The evolution of the Frobenius norm of the weights $\|W\|_F$ and the RKHS norm $\| \xi_{\alpha_\theta} \|_{\mathscr{H}}$ during training is shown in Figure \ref{fig:comparison_norms}. This points out that the penalization, which forces the RNN to keep a small norm in $\mathscr{H}$, leads indeed to learning different weights than the non-penalized RNN. The results also suggest that the Frobenius and RKHS norms are decoupled, since both networks have Frobenius norms of similar magnitude but very different RKHS norms. The figures show one random run, but we observe similar qualitative behavior on others.

\begin{figure}[ht]
     \begin{subfigure}[t]{0.45\textwidth}
             \centering
         \includegraphics[width=\textwidth]{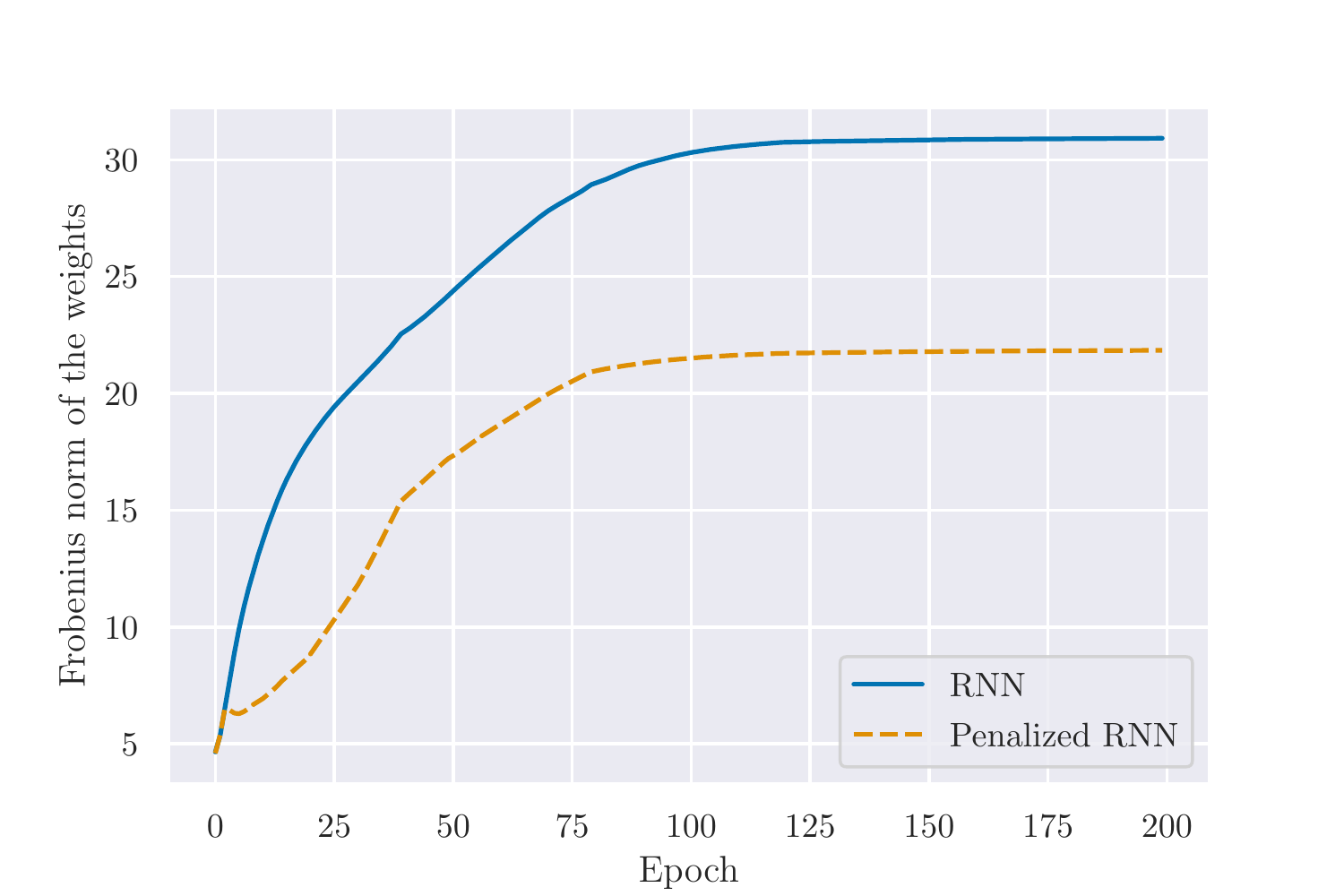}
     \end{subfigure}
     \hfill
     \begin{subfigure}[t]{0.45\textwidth}
         \centering
         \includegraphics[width=\textwidth]{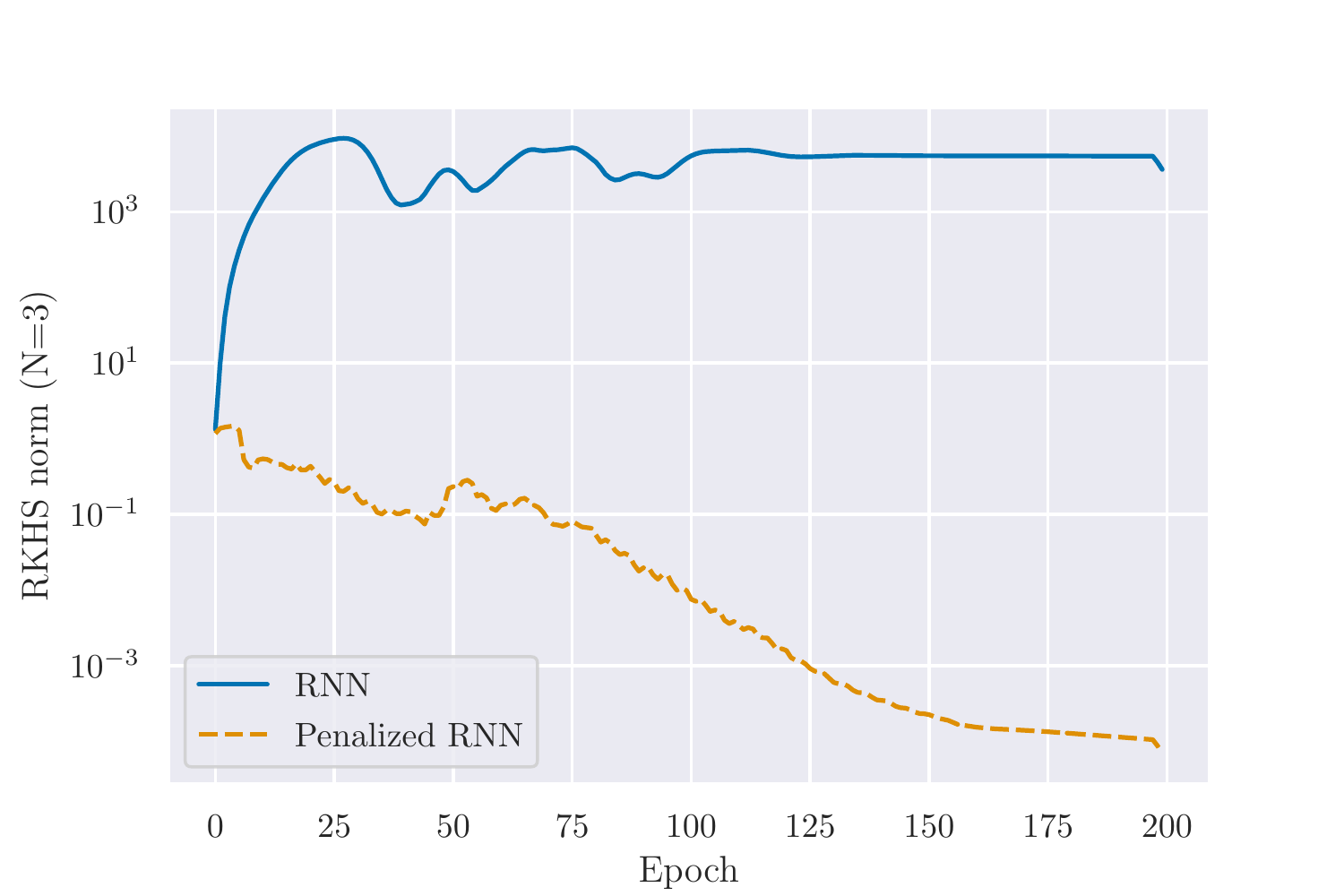}
     \end{subfigure}
    \caption{Evolution of the Frobenius norm of the weights and of the RKHS norm during training}
    \label{fig:comparison_norms}
\end{figure}

\section{Discussion and conclusion}
\label{sec:discussion}

\paragraph{Role of the discretization procedure.}
The starting point of the paper was motivated by the fact that the classical residual RNN formulation coincides with an Euler discretization of the ODE \eqref{eq:residual-rnn-ode}. This choice of discretization translates into a $\nicefrac{1}{T}$ term in Theorems \ref{thm:generalization_bound_binary} and \ref{thm:generalization_bound_sequential}. However, we could have considered higher-order discretization schemes, such as Runge-Kutta schemes, for which the discretization error decreases as $\nicefrac{1}{T^p}$. Such schemes correspond to alternative architectures, which were already proposed by \citet{wang98}, among others.
At the limit, we could also consider directly the continuous model \eqref{eq:residual-rnn-ode}, as proposed by \citet{chen2018neural}, in which case the discretization error term vanishes. Of course, such an option requires to be able to sample the continuous-time data at arbitrary times.

\paragraph{Long-term stability.}
RNN are known to be poor at learning long-term dependencies \citep{bengio1993problem,hochreiter1997long}. This is reflected in the literature by performance bounds increasing in $T$, which is not the case of our results \eqref{eq:gen_bound_rnn_binary} and \eqref{eq:gen_bound_rnn_sequential}, seemingly indicating that we fail to capture this phenomenon. This apparent paradox is related to our assumption that the total variation of $X$ is bounded. Indeed, if a time series is observed for a long time, then its total variation may become large. In this case, it is no longer valid to assume that $\|X\|_{\textnormal{TV}}$ is bounded by $L$. In other words, in our context, the parameter encapsulating the notion of ``long-term'' is not $T$ but the regularity of $X$ measured by its total variation. Note that the choice of defining $X$ on $[0,1]$ and not another interval $[0, U]$ is arbitrary and does not carry any meaning on the problem of learning long-term dependencies. A thorough analysis of these questions is an interesting research direction for future work. 

\paragraph{Radius of convergence.} The assumptions $\|X\|_{TV; [0,1]} \leq L < 1$ and $\|W\|_F \leq K_W < (1-L)/32d$ can be seen as radii of convergence of the Taylor expansion \eqref{eq:H_infinite_sum_signature}. They allow using the Taylor approximation---which is of a local nature---to prove a global result, the RKHS embedding.
In return, the condition on the Frobenius norm of the weights puts restrictions on the admissible parameters of the neural network. However, this bound can be improved, in particular by considering more exotic norms, which we did not explicit for clarity purposes.

\paragraph{Conclusion.} By bringing together the theory of neural ODE, the signature transform, and kernel methods, we have shown that a recurrent network can be framed in the continuous-time limit as a linear function in a well-chosen RKHS. In addition to giving theoretical insights on the function learned by the network and providing generalization guarantees, this framing suggests regularization strategies to obtain more robust RNN. We have only scratched the surface of the potentialities of leveraging this theory to practical applications, which is a subject of its own and will be tackled in future work.

\section*{Acknowledgements}

Authors thank T. Lévy for his inputs on the Picard-Lindelöf theorem and N. Doumèche for fruitful discussion. A. Fermanian has been supported by a grant from Région Île-de-France and P. Marion by a stipend from Corps des Mines.

\bibliographystyle{abbrvnat}
\bibliography{references}

\begin{thebibliography}{60}
\providecommand{\natexlab}[1]{#1}
\providecommand{\url}[1]{\texttt{#1}}
\expandafter\ifx\csname urlstyle\endcsname\relax
  \providecommand{\doi}[1]{doi: #1}\else
  \providecommand{\doi}{doi: \begingroup \urlstyle{rm}\Url}\fi

\bibitem[Akpinar et~al.(2019)Akpinar, Kratzwald, and
  Feuerriegel]{akpinar2019sample}
N.-J. Akpinar, B.~Kratzwald, and S.~Feuerriegel.
\newblock Sample complexity bounds for recurrent neural networks with
  application to combinatorial graph problems.
\newblock \emph{arXiv:1901.10289}, 2019.

\bibitem[Bartlett and Mendelson(2002)]{bartlett2002rademacher}
P.~L. Bartlett and S.~Mendelson.
\newblock Rademacher and {G}aussian complexities: {R}isk bounds and structural
  results.
\newblock \emph{Journal of Machine Learning Research}, 3:\penalty0 463--482,
  2002.

\bibitem[Belkin et~al.(2018)Belkin, Ma, and Mandal]{belkin2018understand}
M.~Belkin, S.~Ma, and S.~Mandal.
\newblock To understand deep learning we need to understand kernel learning.
\newblock In J.~Dy and A.~Krause, editors, \emph{Proceedings of the 35th
  International Conference on Machine Learning}, volume~80, pages 541--549.
  PMLR, 2018.

\bibitem[Bengio et~al.(1993)Bengio, Frasconi, and Simard]{bengio1993problem}
Y.~Bengio, P.~Frasconi, and P.~Simard.
\newblock The problem of learning long-term dependencies in recurrent networks.
\newblock In \emph{1993 IEEE International Conference on Neural Networks},
  pages 1183--1188, 1993.

\bibitem[Bietti and Mairal(2017)]{bietti2017invariance}
A.~Bietti and J.~Mairal.
\newblock Invariance and stability of deep convolutional representations.
\newblock In I.~Guyon, U.~V. Luxburg, S.~Bengio, H.~Wallach, R.~Fergus,
  S.~Vishwanathan, and R.~Garnett, editors, \emph{Advances in Neural
  Information Processing Systems}, volume~30, pages 6210--6220. Curran
  Associates, Inc., 2017.

\bibitem[Bietti and Mairal(2019)]{bietti2019group}
A.~Bietti and J.~Mairal.
\newblock Group invariance, stability to deformations, and complexity of deep
  convolutional representations.
\newblock \emph{Journal of Machine Learning Research}, 20:\penalty0 1--49,
  2019.

\bibitem[Bietti et~al.(2019)Bietti, Mialon, Chen, and Mairal]{bietti2019kernel}
A.~Bietti, G.~Mialon, D.~Chen, and J.~Mairal.
\newblock A kernel perspective for regularizing deep neural networks.
\newblock In K.~Chaudhuri and R.~Salakhutdinov, editors, \emph{Proceedings of
  the 36th International Conference on Machine Learning}, volume~97, pages
  664--674. PMLR, 2019.

\bibitem[Cass et~al.(2020)Cass, Lyons, Salvi, and Yang]{cass2020computing}
T.~Cass, T.~Lyons, C.~Salvi, and W.~Yang.
\newblock Computing the untruncated signature kernel as the solution of a
  {G}oursat problem.
\newblock \emph{arXiv:2006.14794}, 2020.

\bibitem[Chang et~al.(2019)Chang, Chen, Haber, and
  Chi]{chang2018antisymmetricrnn}
B.~Chang, M.~Chen, E.~Haber, and E.~H. Chi.
\newblock Antisymmetric{RNN}: {A} dynamical system view on recurrent neural
  networks.
\newblock In \emph{International Conference on Learning Representations}, 2019.

\bibitem[Chen(1958)]{chen1958integration}
K.-T. Chen.
\newblock Integration of paths--a faithful representation of paths by
  non-commutative formal power series.
\newblock \emph{Transactions of the American Mathematical Society},
  89:\penalty0 395--407, 1958.

\bibitem[Chen et~al.(2020)Chen, Li, and Zhao]{chen2020generalization}
M.~Chen, X.~Li, and T.~Zhao.
\newblock On generalization bounds of a family of recurrent neural networks.
\newblock In S.~Chiappa and R.~Calandra, editors, \emph{Proceedings of the
  Twenty Third International Conference on Artificial Intelligence and
  Statistics}, volume 108, pages 1233--1243, 2020.

\bibitem[Chen et~al.(2018)Chen, Rubanova, Bettencourt, and
  Duvenaud]{chen2018neural}
R.~T.~Q. Chen, Y.~Rubanova, J.~Bettencourt, and D.~K. Duvenaud.
\newblock Neural ordinary differential equations.
\newblock In S.~Bengio, H.~Wallach, H.~Larochelle, K.~Grauman, N.~Cesa-Bianchi,
  and R.~Garnett, editors, \emph{Advances in Neural Information Processing
  Systems}, volume~31, pages 6572--6583. Curran Associates, Inc., 2018.

\bibitem[Chevyrev and Kormilitzin(2016)]{primer2016}
I.~Chevyrev and A.~Kormilitzin.
\newblock A primer on the signature method in machine learning.
\newblock \emph{arXiv:1603.03788}, 2016.

\bibitem[Cho et~al.(2014)Cho, van Merri{\"e}nboer, Gulcehre, Bahdanau,
  Bougares, Schwenk, and Bengio]{Cho2014LearningPR}
K.~Cho, B.~van Merri{\"e}nboer, C.~Gulcehre, D.~Bahdanau, F.~Bougares,
  H.~Schwenk, and Y.~Bengio.
\newblock Learning phrase representations using {RNN} encoder-decoder for
  statistical machine translation.
\newblock In \emph{Proceedings of the 2014 Conference on Empirical Methods in
  Natural Language Processing}, pages 1724--1734. Association for Computational
  Linguistics, 2014.

\bibitem[Cho and Saul(2009)]{cho_saul}
Y.~Cho and L.~Saul.
\newblock Kernel methods for deep learning.
\newblock In Y.~Bengio, D.~Schuurmans, J.~Lafferty, C.~Williams, and
  A.~Culotta, editors, \emph{Advances in Neural Information Processing
  Systems}, volume~22, pages 342--350. Curran Associates, Inc., 2009.

\bibitem[Collobert et~al.(2011)Collobert, Weston, Bottou, Karlen, Kavukcuoglu,
  and Kuksa]{collobert2011natural}
R.~Collobert, J.~Weston, L.~Bottou, M.~Karlen, K.~Kavukcuoglu, and P.~Kuksa.
\newblock Natural language processing (almost) from scratch.
\newblock \emph{Journal of Machine Learning Research}, 12:\penalty0 2493--2537,
  2011.

\bibitem[De~Brouwer et~al.(2019)De~Brouwer, Simm, Arany, and Moreau]{de2019gru}
E.~De~Brouwer, J.~Simm, A.~Arany, and Y.~Moreau.
\newblock {GRU}-{ODE}-{B}ayes: {C}ontinuous modeling of sporadically-observed
  time series.
\newblock In H.~Wallach, H.~Larochelle, A.~Beygelzimer, F.~d\textquotesingle
  Alch\'{e}-Buc, E.~Fox, and R.~Garnett, editors, \emph{Advances in Neural
  Information Processing Systems}, volume~32, pages 7379--7390. Curran
  Associates, Inc., 2019.

\bibitem[Erichson et~al.(2021)Erichson, Azencot, Queiruga, Hodgkinson, and
  Mahoney]{erichson2021lipschitz}
N.~B. Erichson, O.~Azencot, A.~Queiruga, L.~Hodgkinson, and M.~W. Mahoney.
\newblock Lipschitz recurrent neural networks.
\newblock In \emph{International Conference on Learning Representations}, 2021.

\bibitem[Fermanian(2021)]{fermanian2021embedding}
A.~Fermanian.
\newblock Embedding and learning with signatures.
\newblock \emph{Computational Statistics \& Data Analysis}, 157:\penalty0
  107148, 2021.

\bibitem[Friz and Victoir(2008)]{friz2008euler}
P.~Friz and N.~Victoir.
\newblock Euler estimates for rough differential equations.
\newblock \emph{Journal of Differential Equations}, 244:\penalty0 388--412,
  2008.

\bibitem[Friz and Victoir(2010)]{friz2010multidimensional}
P.~K. Friz and N.~B. Victoir.
\newblock \emph{Multidimensional Stochastic Processes as Rough Paths: {T}heory
  and Applications}, volume 120 of \emph{Cambridge Studies in Advanced
  Mathematics}.
\newblock Cambridge University Press, Cambridge, 2010.

\bibitem[Gao et~al.(2018)Gao, Lanchantin, Soffa, and Qi]{gao2018black}
J.~Gao, J.~Lanchantin, M.~L. Soffa, and Y.~Qi.
\newblock Black-box generation of adversarial text sequences to evade deep
  learning classifiers.
\newblock In \emph{2018 IEEE Security and Privacy Workshops}, pages 50--56,
  2018.

\bibitem[Graves et~al.(2013)Graves, Mohamed, and Hinton]{graves2013speech}
A.~Graves, A.-r. Mohamed, and G.~Hinton.
\newblock Speech recognition with deep recurrent neural networks.
\newblock In \emph{2013 IEEE International Conference on Acoustics, Speech and
  Signal Processing}, pages 6645--6649, 2013.

\bibitem[Herrera et~al.(2020)Herrera, Krach, and
  Teichmann]{herrera2020theoretical}
C.~Herrera, F.~Krach, and J.~Teichmann.
\newblock Theoretical guarantees for learning conditional expectation using
  controlled {ODE}-{RNN}.
\newblock \emph{arXiv:2006.04727}, 2020.

\bibitem[Hinton et~al.(2012)Hinton, Deng, Yu, Dahl, Mohamed, Jaitly, Senior,
  Vanhoucke, Nguyen, Sainath, et~al.]{hinton2012deep}
G.~Hinton, L.~Deng, D.~Yu, G.~E. Dahl, A.-r. Mohamed, N.~Jaitly, A.~Senior,
  V.~Vanhoucke, P.~Nguyen, T.~N. Sainath, et~al.
\newblock Deep neural networks for acoustic modeling in speech recognition: The
  shared views of four research groups.
\newblock \emph{IEEE Signal Processing Magazine}, 29:\penalty0 82--97, 2012.

\bibitem[Hochreiter and Schmidhuber(1997)]{hochreiter1997long}
S.~Hochreiter and J.~Schmidhuber.
\newblock Long short-term memory.
\newblock \emph{Neural Computation}, 9:\penalty0 1735--1780, 1997.

\bibitem[Jacot et~al.(2018)Jacot, Gabriel, and Hongler]{neural_tangent_kernel}
A.~Jacot, F.~Gabriel, and C.~Hongler.
\newblock Neural tangent kernel: Convergence and generalization in neural
  networks.
\newblock In S.~Bengio, H.~Wallach, H.~Larochelle, K.~Grauman, N.~Cesa-Bianchi,
  and R.~Garnett, editors, \emph{Advances in Neural Information Processing
  Systems}, volume~31, pages 8580--8589. Curran Associates, Inc., 2018.

\bibitem[Kelly et~al.(2020)Kelly, Bettencourt, Johnson, and
  Duvenaud]{kelly2020higher}
J.~Kelly, J.~Bettencourt, M.~J. Johnson, and D.~K. Duvenaud.
\newblock Learning differential equations that are easy to solve.
\newblock In H.~Larochelle, M.~Ranzato, R.~Hadsell, M.~F. Balcan, and H.~Lin,
  editors, \emph{Advances in Neural Information Processing Systems}, volume~33,
  pages 4370--4380. Curran Associates, Inc., 2020.

\bibitem[Kidger and Lyons(2021)]{signatory}
P.~Kidger and T.~Lyons.
\newblock {S}ignatory: {D}ifferentiable computations of the signature and
  logsignature transforms, on both {CPU} and {GPU}.
\newblock In \emph{International Conference on Learning Representations}, 2021.

\bibitem[Kidger et~al.(2019)Kidger, Bonnier, Perez~Arribas, Salvi, and
  Lyons]{kidger2019deep}
P.~Kidger, P.~Bonnier, I.~Perez~Arribas, C.~Salvi, and T.~Lyons.
\newblock {Deep signature transforms}.
\newblock In H.~Wallach, H.~Larochelle, A.~Beygelzimer, F.~d\textquotesingle
  Alch\'{e}-Buc, E.~Fox, and R.~Garnett, editors, \emph{Advances in Neural
  Information Processing Systems}, volume~32, pages 3099--3109. Curran
  Associates, Inc., 2019.

\bibitem[Kidger et~al.(2020)Kidger, Morrill, Foster, and
  Lyons]{kidger2020neural}
P.~Kidger, J.~Morrill, J.~Foster, and T.~Lyons.
\newblock Neural controlled differential equations for irregular time series.
\newblock In H.~Larochelle, M.~Ranzato, R.~Hadsell, M.~F. Balcan, and H.~Lin,
  editors, \emph{Advances in Neural Information Processing Systems}, volume~33,
  pages 6696--6707. Curran Associates, Inc., 2020.

\bibitem[Kingma and Ba(2015)]{kingmaAdamMethodStochastic2017}
D.~P. Kingma and J.~Ba.
\newblock Adam: {{A Method}} for {{Stochastic Optimization}}.
\newblock In \emph{Proceedings of the 3rd International Conference on Learning
  Representations (ICLR)}, 2015.

\bibitem[Kir{\'{a}}ly and Oberhauser(2019)]{Kiraly2016}
F.~J. Kir{\'{a}}ly and H.~Oberhauser.
\newblock {Kernels for sequentially ordered data}.
\newblock \emph{Journal of Machine Learning Research}, 20:\penalty0 1--45,
  2019.

\bibitem[{K}laus {G}reff et~al.(2017){K}laus {G}reff, {A}aron {K}lein, {M}artin
  {C}hovanec, {F}rank {H}utter, and {J}\"urgen {S}chmidhuber]{sacred}
{K}laus {G}reff, {A}aron {K}lein, {M}artin {C}hovanec, {F}rank {H}utter, and
  {J}\"urgen {S}chmidhuber.
\newblock {T}he {S}acred {I}nfrastructure for {C}omputational {R}esearch.
\newblock In {K}aty {H}uff, {D}avid {L}ippa, {D}illon {N}iederhut, and
  M.~{P}acer, editors, \emph{{P}roceedings of the 16th {P}ython in {S}cience
  {C}onference}, pages 49 -- 56, 2017.

\bibitem[Ko et~al.(2019)Ko, Lyu, Weng, Daniel, Wong, and Lin]{ko2019popqorn}
C.-Y. Ko, Z.~Lyu, L.~Weng, L.~Daniel, N.~Wong, and D.~Lin.
\newblock {POPQORN}: {Q}uantifying robustness of recurrent neural networks.
\newblock In K.~Chaudhuri and R.~Salakhutdinov, editors, \emph{Proceedings of
  the 36th International Conference on Machine Learning}, volume~97, pages
  3468--3477. PMLR, 2019.

\bibitem[Levin et~al.(2013)Levin, Lyons, and Ni]{levin2013learning}
D.~Levin, T.~Lyons, and H.~Ni.
\newblock Learning from the past, predicting the statistics for the future,
  learning an evolving system.
\newblock \emph{arXiv:1309.0260}, 2013.

\bibitem[Liao et~al.(2019)Liao, Lyons, Yang, and Ni]{liao2019learning}
S.~Liao, T.~Lyons, W.~Yang, and H.~Ni.
\newblock {Learning stochastic differential equations using {RNN} with log
  signature features}.
\newblock \emph{arXiv:1908.08286}, 2019.

\bibitem[Lim(2021)]{lim2020understanding}
S.~H. Lim.
\newblock Understanding recurrent neural networks using nonequilibrium response
  theory.
\newblock \emph{Journal of Machine Learning Research}, 22:\penalty0 1--48,
  2021.

\bibitem[Lyons(2014)]{lyons2014rough}
T.~Lyons.
\newblock Rough paths, signatures and the modelling of functions on streams.
\newblock \emph{arXiv:1405.4537}, 2014.

\bibitem[Lyons et~al.(2007)Lyons, Caruana, and L{\'e}vy]{lyons2007differential}
T.~J. Lyons, M.~J. Caruana, and T.~L{\'e}vy.
\newblock \emph{Differential Equations Driven by Rough Paths}, volume 1908 of
  \emph{Lecture Notes in Mathematics}.
\newblock Springer, Berlin, 2007.

\bibitem[Madry et~al.(2018)Madry, Makelov, Schmidt, Tsipras, and
  Vladu]{madry2018towards}
A.~Madry, A.~Makelov, L.~Schmidt, D.~Tsipras, and A.~Vladu.
\newblock Towards deep learning models resistant to adversarial attacks.
\newblock In \emph{International Conference on Learning Representations}, 2018.

\bibitem[Mikolov et~al.(2010)Mikolov, Karafi{\'a}t, Burget, {\v{C}}ernock{\`y},
  and Khudanpur]{mikolov2010recurrent}
T.~Mikolov, M.~Karafi{\'a}t, L.~Burget, J.~{\v{C}}ernock{\`y}, and
  S.~Khudanpur.
\newblock Recurrent neural network based language model.
\newblock In \emph{Proceedings of the 11th Annual Conference of the
  International Speech Communication Association}, volume~2, pages 1045--1048,
  2010.

\bibitem[Minai and Williams(1993)]{minai1993derivatives}
A.~A. Minai and R.~D. Williams.
\newblock On the derivatives of the sigmoid.
\newblock \emph{Neural Networks}, 6:\penalty0 845--853, 1993.

\bibitem[Morrill et~al.(2020{\natexlab{a}})Morrill, Salvi, Kidger, Foster, and
  Lyons]{morrill2020neural}
J.~Morrill, C.~Salvi, P.~Kidger, J.~Foster, and T.~Lyons.
\newblock Neural rough differential equations for long time series.
\newblock \emph{arXiv:2009.08295}, 2020{\natexlab{a}}.

\bibitem[Morrill et~al.(2020{\natexlab{b}})Morrill, Kormilitzin,
  Nevado-Holgado, Swaminathan, Howison, and Lyons]{morrill2020utilization}
J.~H. Morrill, A.~Kormilitzin, A.~J. Nevado-Holgado, S.~Swaminathan, S.~D.
  Howison, and T.~J. Lyons.
\newblock Utilization of the signature method to identify the early onset of
  sepsis from multivariate physiological time series in critical care
  monitoring.
\newblock \emph{Critical Care Medicine}, 48:\penalty0 e976--e981,
  2020{\natexlab{b}}.

\bibitem[Paszke et~al.(2019)Paszke, Gross, Massa, Lerer, Bradbury, Chanan,
  Killeen, Lin, Gimelshein, Antiga, Desmaison, Kopf, Yang, DeVito, Raison,
  Tejani, Chilamkurthy, Steiner, Fang, Bai, and Chintala]{paszke2019pytorch}
A.~Paszke, S.~Gross, F.~Massa, A.~Lerer, J.~Bradbury, G.~Chanan, T.~Killeen,
  Z.~Lin, N.~Gimelshein, L.~Antiga, A.~Desmaison, A.~Kopf, E.~Yang, Z.~DeVito,
  M.~Raison, A.~Tejani, S.~Chilamkurthy, B.~Steiner, L.~Fang, J.~Bai, and
  S.~Chintala.
\newblock {PyTorch: An Imperative Style, High-Performance Deep Learning
  Library}.
\newblock In H.~Wallach, H.~Larochelle, A.~Beygelzimer, F.~d\textquotesingle
  Alch\'{e}-Buc, E.~Fox, and R.~Garnett, editors, \emph{Advances in Neural
  Information Processing Systems}, volume~32, pages 8024--8035. Curran
  Associates, Inc., 2019.

\bibitem[Perez~Arribas(2018)]{perez2018derivatives}
I.~Perez~Arribas.
\newblock Derivatives pricing using signature payoffs.
\newblock \emph{arXiv:1809.09466}, 2018.

\bibitem[Reizenstein and Graham(2020)]{reizenstein2018iisignature}
J.~F. Reizenstein and B.~Graham.
\newblock Algorithm 1004: {T}he iisignature library: {E}fficient calculation of
  iterated-integral signatures and log signatures.
\newblock \emph{ACM Transactions on Mathematical Software}, 46:\penalty0
  article 8, 2020.

\bibitem[Riordan(1958)]{riordan2014introduction}
J.~Riordan.
\newblock \emph{An Introduction to Combinatorial Analysis}.
\newblock John Wiley \& Sons, New York, 1958.

\bibitem[Rubanova et~al.(2019)Rubanova, Chen, and Duvenaud]{rubanova2019latent}
Y.~Rubanova, R.~T.~Q. Chen, and D.~K. Duvenaud.
\newblock Latent ordinary differential equations for irregularly-sampled time
  series.
\newblock In H.~Wallach, H.~Larochelle, A.~Beygelzimer, F.~d\textquotesingle
  Alch\'{e}-Buc, E.~Fox, and R.~Garnett, editors, \emph{Advances in Neural
  Information Processing Systems}, volume~32, pages 5320--5330. Curran
  Associates, Inc., 2019.

\bibitem[Sch{\"o}lkopf and Smola(2002)]{scholkopf2002learning}
B.~Sch{\"o}lkopf and A.~J. Smola.
\newblock \emph{Learning with kernels: support vector machines, regularization,
  optimization, and beyond}.
\newblock MIT press, Cambridge, Massachusetts, 2002.

\bibitem[Toth and Oberhauser(2020)]{toth2020bayesian}
C.~Toth and H.~Oberhauser.
\newblock Bayesian learning from sequential data using {G}aussian processes
  with signature covariances.
\newblock In H.~{Daum{\'e} III} and A.~Singh, editors, \emph{Proceedings of the
  37th International Conference on Machine Learning}, volume 119, pages
  9548--9560, 2020.

\bibitem[Tu et~al.(2019)Tu, He, and Tao]{tu2019understanding}
Z.~Tu, F.~He, and D.~Tao.
\newblock Understanding generalization in recurrent neural networks.
\newblock In \emph{International Conference on Learning Representations}, 2019.

\bibitem[Virtanen et~al.(2020)Virtanen, Gommers, Oliphant, Haberland, Reddy,
  Cournapeau, Burovski, Peterson, Weckesser, Bright, {van der Walt}, Brett,
  Wilson, Millman, Mayorov, Nelson, Jones, Kern, Larson, Carey, Polat, Feng,
  Moore, {VanderPlas}, Laxalde, Perktold, Cimrman, Henriksen, Quintero, Harris,
  Archibald, Ribeiro, Pedregosa, {van Mulbregt}, and {SciPy 1.0
  Contributors}]{scipy}
P.~Virtanen, R.~Gommers, T.~E. Oliphant, M.~Haberland, T.~Reddy, D.~Cournapeau,
  E.~Burovski, P.~Peterson, W.~Weckesser, J.~Bright, S.~J. {van der Walt},
  M.~Brett, J.~Wilson, K.~J. Millman, N.~Mayorov, A.~R.~J. Nelson, E.~Jones,
  R.~Kern, E.~Larson, C.~J. Carey, {\.I}.~Polat, Y.~Feng, E.~W. Moore,
  J.~{VanderPlas}, D.~Laxalde, J.~Perktold, R.~Cimrman, I.~Henriksen, E.~A.
  Quintero, C.~R. Harris, A.~M. Archibald, A.~H. Ribeiro, F.~Pedregosa, P.~{van
  Mulbregt}, and {SciPy 1.0 Contributors}.
\newblock {{SciPy} 1.0: Fundamental Algorithms for Scientific Computing in
  Python}.
\newblock \emph{Nature Methods}, 17:\penalty0 261--272, 2020.

\bibitem[Wang et~al.(2019)Wang, Liakata, Ni, Lyons, Nevado-Holgado, and
  Saunders]{wang2019path}
B.~Wang, M.~Liakata, H.~Ni, T.~Lyons, A.~J. Nevado-Holgado, and K.~Saunders.
\newblock A path signature approach for speech emotion recognition.
\newblock In \emph{Proceedings of Interspeech 2019}, pages 1661--1665, 2019.

\bibitem[Wang and Lin(1998)]{wang98}
Y.-J. Wang and C.-T. Lin.
\newblock Runge-{K}utta neural network for identification of dynamical systems
  in high accuracy.
\newblock \emph{IEEE Transactions on Neural Networks}, 9:\penalty0 294--307,
  1998.

\bibitem[Yang et~al.(2016)Yang, Jin, and Liu]{yang2016deepwriterid}
W.~Yang, L.~Jin, and M.~Liu.
\newblock {{D}eep{W}riter{ID}: {A}n end-to-end online text-independent writer
  identification system}.
\newblock \emph{IEEE Intelligent Systems}, 31:\penalty0 45--53, 2016.

\bibitem[Yang et~al.(2017)Yang, Lyons, Ni, Schmid, and Jin]{yang2017leveraging}
W.~Yang, T.~Lyons, H.~Ni, C.~Schmid, and L.~Jin.
\newblock {Developing the path signature methodology and its application to
  landmark-based human action recognition}.
\newblock \emph{arXiv:1707.03993}, 2017.

\bibitem[Yue et~al.(2018)Yue, Fu, and Liang]{yue2018residual}
B.~Yue, J.~Fu, and J.~Liang.
\newblock Residual recurrent neural networks for learning sequential
  representations.
\newblock \emph{Information}, 9:\penalty0 56, 2018.

\bibitem[Zhang et~al.(2018)Zhang, Lei, and Dhillon]{zhang2018stabilizing}
J.~Zhang, Q.~Lei, and I.~Dhillon.
\newblock Stabilizing gradients for deep neural networks via efficient {SVD}
  parameterization.
\newblock In J.~Dy and A.~Krause, editors, \emph{Proceedings of the 35th
  International Conference on Machine Learning}, volume~80, pages 5806--5814.
  PMLR, 2018.

\end{thebibliography}

\newpage

\appendix

    \begin{center}
        \Large{\textbf{Framing RNN as a kernel method: A neural ODE approach \\ Supplementary material}}
    \end{center}
    
    \bigskip

\section{Mathematical details} \label{apx:supp_math_details}

\subsection{Writing the GRU and LSTM in the neural ODE framework} \label{apx:extension_gru_lstm}

\paragraph{GRU.}

Recall that the equations of a GRU take the following form: for any $1 \leq j \leq T$,
\begin{align*}
    r_{j+1} &= \sigma(W_{r}x_{j+1} + b_{r} + U_r h_{j}) \\
    z_{j+1} &= \sigma(W_z x_{j+1} + b_z + U_z h_{j}) \\
    n_{j+1} &= \textnormal{tanh}\big(W_n x_{j+1} + b_n + r_{j+1} \ast (U_n h_{j} + c_n) \big) \\
    h_{j+1} &= (1-z_{j+1}) \ast h_{j}  + z_{j+1} \ast n_{j+1} ,
\end{align*}
where $\sigma$ is the logistic activation, $\textnormal{tanh}$ the hyperbolic tangent, $\ast$ the Hadamard product, $r_j$ the reset gate vector, $z_j$ the update gate vector, $W_{r}$, $U_r$, $W_z$, $U_z$, $W_n$, $U_n$ weight matrices, and $b_{r}$, $b_z$, $b_n$, $c_n$ biases. Since $r_{j+1}$, $z_{j+1}$, and $n_{j+1}$ depend only on $x_{j+1}$ and $h_{j}$, it is clear that these equations can be rewritten in the form
\begin{equation*}
   h_{j+1} = h_{j} + f(h_{j}, x_{j+1}).
\end{equation*}
We then obtain equation \eqref{eq:residual-rnn} by normalizing $f$ by $\nicefrac{1}{T}$.

\paragraph{LSTM.} The LSTM networks are defined, for any $1 \leq j \leq T$, by
\begin{align*}
    i_{j+1} &= \sigma(W_i x_{j+1} + b_i + U_i h_{j}) \\
    f_{j+1} &= \sigma(W_{f}x_{j+1} + b_{f} + U_f h_{j}) \\
    g_{j+1} &= \textnormal{tanh}(W_g x_{j+1} + b_g +  U_g h_{j})\\
    o_{j+1} &= \sigma(W_o x_{j+1} + b_o +  U_o h_{j})\\
    c_{j+1} &= f_{j+1} \ast c_{j} + i_{j+1} \ast g_{j+1} \\
    h_{j+1} &= o_{j+1} \ast \textnormal{tanh}(c_{j+1}),
\end{align*}
where $\sigma$ is the logistic activation, $\textnormal{tanh}$ the hyperbolic tangent, $\ast$ the Hadamard product, $i_j$ the input gate, $f_j$ the forget gate, $g_j$ the cell gate, $o_j$ the output gate, $c_j$ the cell state, $W_{i}$, $U_i$, $W_{f}$, $U_f$, $W_{g}$, $U_g$ $W_o$, $U_o$ weight matrices, and $b_{i}$, $b_f$, $b_g$, $b_o$ biases. Since $i_{j+1}$, $f_{j+1}$, $g_{j+1}$, $o_{j+1}$ depend only on $x_{j+1}$ and $h_{j}$, these equations can be rewritten in the form
\begin{align*}
    h_{j+1} &= f_1(h_{j}, x_{j+1}, c_{j+1}) \\
    c_{j+1} &= f_2(h_{j}, x_{j+1}, c_{j}).
\end{align*}
Let $\Tilde{h}_j = (h_j^\top, c_j^\top)^\top$ be the hidden state defined by stacking the hidden and cell state. Then, clearly, $\tilde{h}$ follows an equation of the form
\begin{equation*}
    \Tilde{h}_{j+1} = f(\Tilde{h}_{j}, x_{j+1}).
\end{equation*}
We obtain \eqref{eq:residual-rnn} by subtracting $\tilde{h}_{j}$ and normalizing by $\nicefrac{1}{T}$.

\subsection{Picard-Lindelöf theorem}
\label{apx:picard-lindelof}

Consider a CDE of the form \eqref{eq:cde-vector-field}. We recall the Picard-Lindelöf theorem as given by \citet[][Theorem 1.3]{lyons2007differential}, and provide a proof for the sake of completeness.
\begin{theorem}[Picard-Lindelöf theorem]
\label{thm:picard-lindelof}
Assume that $X \in BV^{c}([0,1], \R^d)$ and that $\mathbf{F}$ is Lipschitz-continuous with constant $K_{\mathbf{F}}$. Then, for any $H_0 \in \R^e$, the differential equation \eqref{eq:cde-vector-field} admits a unique solution $H:[0,1] \to \R^e$.
\end{theorem}
\begin{proof}
    Let $ \mathscr{C}([s,t]), \R^{e})$ be the set of continuous functions from $[s,t]$ to $\R^{e}$. For any $[s,t] \subset [0,1]$, $\zeta \in \R^{e}$, let  $\Psi$ be the function
    \begin{align*}
        \Psi: \mathscr{C}([s,t]), \R^{e}) &\to \mathscr{C}([s,t], \R^{e}) \\
        Y &\mapsto \big(v \mapsto \zeta + \int_{s}^v \mathbf{F}(Y_u) dX_u \big).
    \end{align*}
For any $Y, Y' \in \mathscr{C}([s,t]), \R^{e})$, $v \in [s,t]$,
\begin{align*}
    \| \Psi(Y)_v - \Psi(Y')_v\| &\leq \int_{s}^v \big\| \big(\mathbf{F}(Y_u) - \mathbf{F}(Y'_u) \big)dX_u \big\| \\
    &\leq \int_{s}^v \| \mathbf{F}(Y_u) - \mathbf{F}(Y'_u)\|_{\textnormal{op}}\|dX_u\|\\
    &\leq \int_{s}^v K_{\mathbf{F}} \|Y_u - Y'_u\| \|dX_u\| \\
    & \leq K_{\mathbf{F}} \| Y-Y'\|_\infty \int_{s}^v \|dX_u\| \\
    & \leq K_{\mathbf{F}} \| Y-Y'\|_\infty \|X\|_{TV;[s,t]}.
\end{align*}
This shows that the function $\Psi$ is Lipschitz-continuous on $\mathscr{C}([s,t]), \R^{e})$ endowed with the supremum norm, with Lipschitz constant $K_{\mathbf{F}}\|X\|_{TV;[s,t]}$. Clearly, the function $t \mapsto \|X\|_{TV;[0,t]}$ is non-decreasing and uniformly continuous on the compact interval $[0,1]$. Therefore, for any $\varepsilon > 0$, there exists $\delta > 0$ such that
\begin{equation*}
    |t-s|< \delta \Rightarrow \big| \|X\|_{TV;[0,t]} - \|X\|_{TV;[0,s]} \big| < \varepsilon.
\end{equation*}
Take $\varepsilon = \nicefrac{1}{K_{\mathbf{F}}}$. Then on any interval $[s,t]$ of length smaller than $\delta$, one has $\|X\|_{TV;[s,t]} =\|X\|_{TV;[0,t]} - \|X\|_{TV;[0,s]} < \nicefrac{1}{K_{\mathbf{F}}}$, so that the function $\Psi$ is a contraction. By the Banach fixed-point theorem, for any initial value $\zeta$, $\Psi$ has a unique fixed point. Hence, there exists a solution to \eqref{eq:cde-vector-field} on any interval of length $\delta$ with any initial condition. To obtain a solution on $[0,1]$ it is sufficient to concatenate these solutions.
\end{proof}

A corollary of this theorem is a Picard-Lindelöf theorem for initial value problems of the form
\begin{equation} \label{eq:ivp}
    dH_t = f(H_t, X_t)dt, \quad H_0 = \zeta,
\end{equation}
where $f: \R^e \times \R^d \to \R^{e}$, $\zeta \in \R^{e}$.
\begin{corollary}\label{cor:picard-ode}
Assume that $f$ is Lipschitz continuous in its first variable. Then, for any $\zeta \in \R^e$, the initial value problem \eqref{eq:ivp} admits a unique solution.
\end{corollary}
\begin{proof}
  Let $f_X:(h, t) \mapsto f(h,X_t)$. Then the solution of \eqref{eq:ivp} is solution of the differential equation
  \begin{equation*}
      dH_t =f_X(H_t,t)dt.
  \end{equation*}
 Let $d=1$, $\bar{e} = e+1$, and $\mathbf{F}$ be the vector field defined by 
  \begin{align*}
      \mathbf{F} : h \mapsto \begin{pmatrix}f_X(h^{1:e}, h^{e+1}) \\ 1 \end{pmatrix},
  \end{align*}
  where $h^{1:e}$ denotes the projection of $h$ on its first $e$ coordinates. Then, since $f_X$ is Lipschitz, so is the vector field $\mathbf{F}$. Theorem \ref{thm:picard-lindelof} therefore applies to the differential equation
  \[dH_t = \mathbf{F}(H_t)dt, \quad H_0 =(\zeta^\top, 0)^\top .\]
  Projecting this differential equation on the last coordinate gives $dH^{e+1}_t = dt$, that is, $H^{e+1}_t = t$. Projecting on the first $e$ coordinates exactly provides equation \eqref{eq:ivp}, which therefore has a unique solution, equal to $H^{1:e}$.
\end{proof}

\subsection{Operator norm} \label{apx:operator-norm}

\begin{definition}
Let $(E, \| \cdot \|_E)$ and $(F, \| \cdot \|_F)$ be two normed vector spaces and let $f \in \mathscr{L}(E, F)$, where $ \mathscr{L}(E, F)$ is the space of linear functions from $E$ to $F$. The operator norm of $f$ is defined by
\begin{equation*} \label{eq:def_operator_norm}
    \|f\|_{\textnormal{op}} = \underset{u \in E, \|u\|_E = 1}{\sup}\|f(u)\|_F.
\end{equation*}
Equipped with this norm, $\mathscr{L}(E,F)$ is a normed vector space.  
\end{definition}
This definition is valid when $f$ is represented by a matrix.

\subsection{Tensor Hilbert space} \label{apx:tensor_hilbert_space}

Let us first briefly recall some elements on tensor spaces. If $e_1, \dots, e_d$ is the canonical basis of $\R^d$, then $(e_{i_1} \otimes \dots \otimes e_{i_k})_{1 \leq i_1, \dots, i_k \leq d}$ is a basis of $(\R^d)^{\otimes k}$. Any element $a \in (\R^d)^{\otimes k} $ can therefore be written as
\[a = \sum_{1 \leq i_1, \dots, i_k \leq d}a^{(i_1, \dots, i_k)} e_{i_1} \otimes \dots \otimes e_{i_k}, \]
where $a^{(i_1, \dots, i_k)} \in \R$. The tensor space $(\R^d)^{\otimes k}$ is a Hilbert space of dimension $d^k$, with scalar product
\begin{equation*}
    \langle a, b \rangle_{(\R^d)^{\otimes k}} = \sum_{1 \leq i_1, \dots, i_k \leq d} a^{(i_1, \dots, i_k)} b^{(i_1, \dots, i_k)}
\end{equation*}
and associated norm $\| \cdot \|_{(\R^d)^{\otimes k}}$. 

We now consider the space $\mathscr{T}$ defined by \eqref{eq:def_T}. The sum, multiplication by a scalar, and scalar product on $\mathscr{T}$ are defined as follows: for any $a =(a_0, \dots, a_k, \dots) \in \mathscr{T}, \, b = (b_0, \dots, b_k, \dots) \in \mathscr{T}$, $\lambda \in \R$, 
\begin{align*}
    a + \lambda b = (a_0 + \lambda b_0, \dots, a_k + \lambda b_k, \dots )\quad \text{ and } \quad \langle a, b \rangle_{\mathscr{T}} = \sum_{k=0}^{\infty} \langle a_k, b_k \rangle_{(\R^d)^{\otimes k}},
\end{align*}
with the convention $(\R^d)^{\otimes 0} = \R$.

\begin{proposition}
    $(\mathscr{T}, +, \cdot, \langle \cdot, \cdot \rangle_{\mathscr{T}})$ is a Hilbert space.
\end{proposition}
\begin{proof}
    By the Cauchy-Schwartz inequality, $\langle \cdot, \cdot \rangle_{\mathscr{T}}$ is well-defined: for any $a, b \in \mathscr{T}$,
    \begin{align*}
        | \langle a, b \rangle_{\mathscr{T}} | \leq  \sum_{k=0}^{\infty} |\langle a_k, b_k \rangle_{(\R^d)^{\otimes k}}| &\leq \sum_{k=0}^{\infty} \| a_k \|_{(\R^d)^{\otimes k}} \|b_k\|_{(\R^d)^{\otimes k}} \\
        & \leq \Big( \sum_{k=0}^{\infty} \| a_k \|^2_{(\R^d)^{\otimes k}}\Big)^{1/2}\Big( \sum_{k=0}^{\infty} \| b_k \|^2_{(\R^d)^{\otimes k}}\Big)^{1/2} < \infty.
    \end{align*}
     Moreover, $\mathscr{T}$ is a vector space: for any $a, b \in \mathscr{T}$, $\lambda \in \R$, since
    \begin{equation*}
        a + \lambda b = (a_0 + \lambda b_0, \dots, a_k + \lambda b_k, \dots ),
    \end{equation*}
    and
    \begin{align*}
        \sum_{k=0}^{\infty} \| a_k + \lambda b_k \|^2_{(\R^d)^{\otimes k}} & = \sum_{k=0}^{\infty}\| a_k \|^2_{(\R^d)^{\otimes k}} + \lambda^2 \sum_{k=0}^{\infty} \| b_k \|^2_{(\R^d)^{\otimes k}} \\
        & \quad + 2\lambda \sum_{k=0}^{\infty} \langle a_k, b_k \rangle_{(\R^d)^{\otimes k}} \\
        & \leq  \sum_{k=0}^{\infty} \| a_k \|^2_{(\R^d)^{\otimes k}} + \lambda^2 \sum_{k=0}^{\infty} \| b_k \|^2_{(\R^d)^{\otimes k}} + 2 \lambda \langle a,b\rangle_{\mathscr{T}} < \infty,
    \end{align*}
    we see that $a + \lambda b \in \mathscr{T}$. The operation $\langle \cdot, \cdot \rangle_{\mathscr{T}} $ is also bilinear, symmetric, and positive definite:
   \begin{equation*}
      \langle a, a \rangle_{\mathscr{T}} = 0 \Leftrightarrow \sum_{k=0}^{\infty} \|a_k \|^2_{(\R^d)^{\otimes k}} = 0 \Leftrightarrow \forall k \in \N, \|a_k \|^2_{(\R^d)^{\otimes k}} = 0 \Leftrightarrow \forall k \in \N, a_k = 0 \Leftrightarrow a=0.
   \end{equation*}
   Therefore $\langle \cdot, \cdot \rangle_{\mathscr{T}} $ is an inner product on $\mathscr{T}$. Finally, let $(a^{(n)})_{n \in \N}$ be a Cauchy sequence in $\mathscr{T}$. Then, for any $n, m \geq 0$,
   \begin{align*}
       \| a^{(n)} - a^{(m)}\|_{\mathscr{T}}^2 = \sum_{k=0}^{\infty} \| a^{(n)}_k - a^{(m)}_k\|^2_{(\R^d)^{\otimes k}},
   \end{align*}
  so for any $k \in \N$, the sequence $(a^{(n)}_k)_{n\in \N}$ is Cauchy in $(\R^d)^{\otimes k}$. Since $(\R^d)^{\otimes k}$ is a Hilbert space,  $(a^{(n)}_k)_{n\in \N}$ converges to a limit $a^{(\infty)}_k \in (\R^d)^{\otimes k}$. Let $a^{(\infty)} = (a_{0}^{(\infty)}, \dots, a_k^{(\infty)}, \dots)$. To finish the proof, we need to show that $a^{(\infty)} \in \mathscr{T}$ and that  $a^{(n)}$ converges to $a^{(\infty)}$ in $\mathscr{T}$. First, note that there exists a constant $B > 0$ such that for any $n \in \N$, \[\|a^{(n)}\|_{\mathscr{T}} \leq B.\]
  To see this, observe that for $\varepsilon >0$, there exists $N \in \N$ such that for any $n \geq N$, $ \| a^{(n)} - a^{(N)}\|_{\mathscr{T}} < \varepsilon$, and so $\| a^{(n)}\|_{\mathscr{T}} \leq \varepsilon + \|a^{(N)}\|_{\mathscr{T}}$. Take $B= \max(\|a^{(1)}\|_{\mathscr{T}}, \dots, \|a^{(N)}\|_{\mathscr{T}}, \varepsilon + \|a^{(N)}\|_{\mathscr{T}})$. Then, for any $K \in \N$,
  \begin{align*}
      \sum_{k=0}^K \| a^{(n)}_k\|^2_{(\R^d)^{\otimes k}} \leq  \|a^{(n)}\|_{\mathscr{T}} \leq B.
  \end{align*}
  Letting $K \to \infty$, we obtain that $\| a^{(\infty)}\|_{\mathscr{T}} \leq B$, and therefore $a^{(\infty)} \in \mathscr{T}$. Finally, let $\varepsilon > 0$ and let $N \in \N$ be such that for any $n,m \geq N$, $\| a^{(n)} - a^{(m)}\|_{\mathscr{T}} < \varepsilon$. Clearly, for any $K \in \N$,
  \[  \sum_{k=0}^{K} \| a^{(n)}_k - a^{(m)}_k\|^2_{(\R^d)^{\otimes k}} < \varepsilon^2. \] 
  Letting $m \to \infty$ leads to
    \begin{align*}
      \sum_{k=1}^K  \| a^{(n)}_k - a^{(\infty)}_k\|^2_{(\R^d)^{\otimes k}} < \varepsilon^2,
  \end{align*}
  and letting $ K\to \infty$ gives 
     \begin{align*}
     \|a^{(n)} - a^{(\infty)}\|_{\mathscr{T}} < \varepsilon,
  \end{align*}
  which completes the proof.
  \end{proof}

\subsection{Bounding the derivatives of the logistic and hyperbolic tangent activations} \label{apx:bounding-derivatives-activations}

\begin{lemma}\label{lemma:bound_logistic}
    Let $\sigma$ be the logistic function defined, for any $x \in \R$, by $\sigma(x) = \nicefrac{1}{(1 + e^{-x})}$. Then, for any $n \geq 0$,
    \begin{equation*}
        \|\sigma^{(n)}\|_\infty \leq 2^{n-1} n!\,.
    \end{equation*}
\end{lemma}
\begin{proof}
    For any $x \in \R$, one has \citep[][Theorem 2]{minai1993derivatives}
    \begin{align*}
        \sigma^{(n)}(x) = \sum_{k=1}^{n+1}(-1)^{k-1}(k-1)!\stirling{n+1}{k}\sigma(x)^k ,
    \end{align*}
    where $\stirling{n}{k}$ stands for the Stirling number of the second kind \citep[see, e.g.,][]{riordan2014introduction}. Let 
    \[ u_n = \sum_{k=1}^{n+1} (k-1)! \stirling{n+1}{k}\]
    for $n \geq 1$ and $u_0 = 1$. Since $0 \leq \sigma(x) \leq 1$, it is clear that $|\sigma^{(n)}(x)| \leq u_n$. Using the fact that the Stirling numbers satisfy the recurrence relation
    \begin{equation*}
        \stirling{n+1}{k} = k\stirling{n}{k} + \stirling{n}{k-1},
    \end{equation*}
    valid for all $0 \leq k \leq n$, we have
    \begin{align*}
        u_n
          & =\sum_{k=1}^{n} (k-1)!\Big(k\stirling{n}{k} + \stirling{n}{k-1} \Big) + n! = \sum_{k=1}^{n} k! \stirling{n}{k} +\sum_{k=0}^{n-1}k! \stirling{n}{k} + n! = 2 \sum_{k=1}^{n} k! \stirling{n}{k} \\
        & \qquad \mbox{(since $\stirling{n}{0} =0$)} \\
        & \leq 2n\sum_{k=1}^{n} (k-1)! \stirling{n}{k} = 2nu_{n-1}.
    \end{align*}
    Thus, by induction, $u_n \leq 2^{n-1} n!$, from which the claim follows.
\end{proof}

\begin{lemma}
     Let $\textnormal{tanh}$ be the hyperbolic tangent function. Then, for any $n \geq 0$,
    \begin{equation*}
        \|\textnormal{tanh}^{(n)}\|_\infty \leq 4^n n!\,.
    \end{equation*}
\end{lemma}
\begin{proof}
    Let $\sigma$ be the logistic function. Straightforward calculations yield the equality, valid for any $x \in \R$,
    \begin{equation*}
        \textnormal{tanh}(x) = 2 \sigma(2x) -1.
    \end{equation*}
    But, for any $n \geq 1$,
    \begin{equation*}
        \textnormal{tanh}^{(n)}(x) = 2^{n+1} \sigma^{(n)}(2x),
    \end{equation*}
    and thus, by Lemma \ref{lemma:bound_logistic},
    \begin{equation*}
         \| \textnormal{tanh}^{(n)}\|_\infty \leq  2^{n+1} \|\sigma^{(n)}\|_\infty \leq 4^n n!\,.
    \end{equation*}
     The inequality is also true for $n=0$ since $\|\textnormal{tanh}\|_\infty \leq 1$. 
\end{proof}

\subsection{Chen's formula} \label{apx:properties_signature}
First, note that it is straightforward to extend the definition of the signature to any interval $[s,t] \subset [0,1]$. The next proposition, known as Chen's formula \citep[][Theorem 2.9]{lyons2007differential}, tells us that the signature can be computed iteratively as tensor products of signatures on subintervals.
\begin{proposition} \label{prop:chen_formula}
 Let $X \in BV^{c}([s,t], \R^d)$ and $u \in (s,t)$. Then
\begin{equation*}
    S_{[s,t]}(X) = S_{[s,u]}(X) \otimes S_{[u,t]}(X).
\end{equation*}
\end{proposition}
Next, it is clear that the signature of a constant path is equal to $\mathbf{1} = (1, 0, \dots, 0, \dots)$ which is the null element in $\mathscr{T}$. Indeed, let $Y \in BV^{c}([s,t],\R^d)$ be a constant path. Then, for any $k \geq 1$,
\begin{equation*}
    \mathbb{Y}^k_{[s,t]} = k! \idotsint\limits_{ s\leq u_1 <  \cdots <u_k \leq t } dY_{u_1}\otimes \dots \otimes dY_{u_k} = k! \idotsint\limits_{ s\leq u_1 <  \cdots <u_k \leq t } 0 \otimes \dots \otimes 0 = 0.
\end{equation*}

Now let $X \in BV^{c}([0,1], \R^d)$ and consider the path $\Tilde{X}_{[j]}$ equal to the time-augmented path $\bar{X}$ on $[0, \nicefrac{j}{T}]$ and then constant on $[\nicefrac{j}{T}, 1]$---see Figure \ref{fig:examples_tilde_X}. We have by Proposition \ref{prop:chen_formula}
\begin{align*}
    S_{[0,1]}(\Tilde{X}_{[j]}) &= S_{[0,\nicefrac{j}{T}]}(\Tilde{X}_{[j]}) \otimes S_{[\nicefrac{j}{T}, 1]}(\Tilde{X}_{[j]}) = S_{[0,\nicefrac{j}{T}]}(\Bar{X}) \otimes \mathbf{1} =  S_{[0,\nicefrac{j}{T}]}(\Bar{X}).
\end{align*}

\begin{figure}[ht]
     \centering
     \begin{subfigure}[t]{0.45\textwidth}
         \centering
         \includegraphics[width=\textwidth]{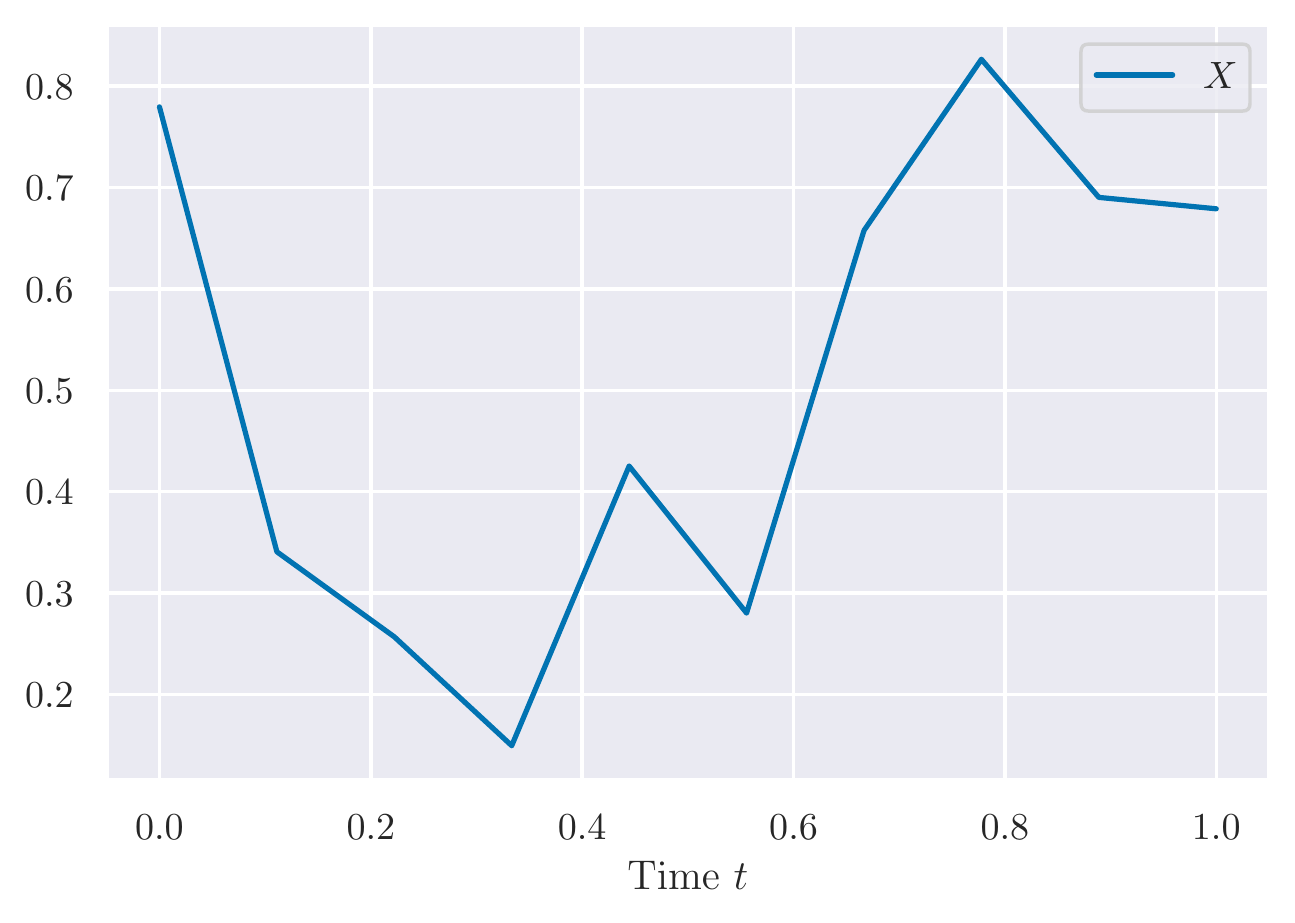}
     \end{subfigure}
     \hfill
     \begin{subfigure}[t]{0.45\textwidth}
         \centering
         \includegraphics[width=\textwidth]{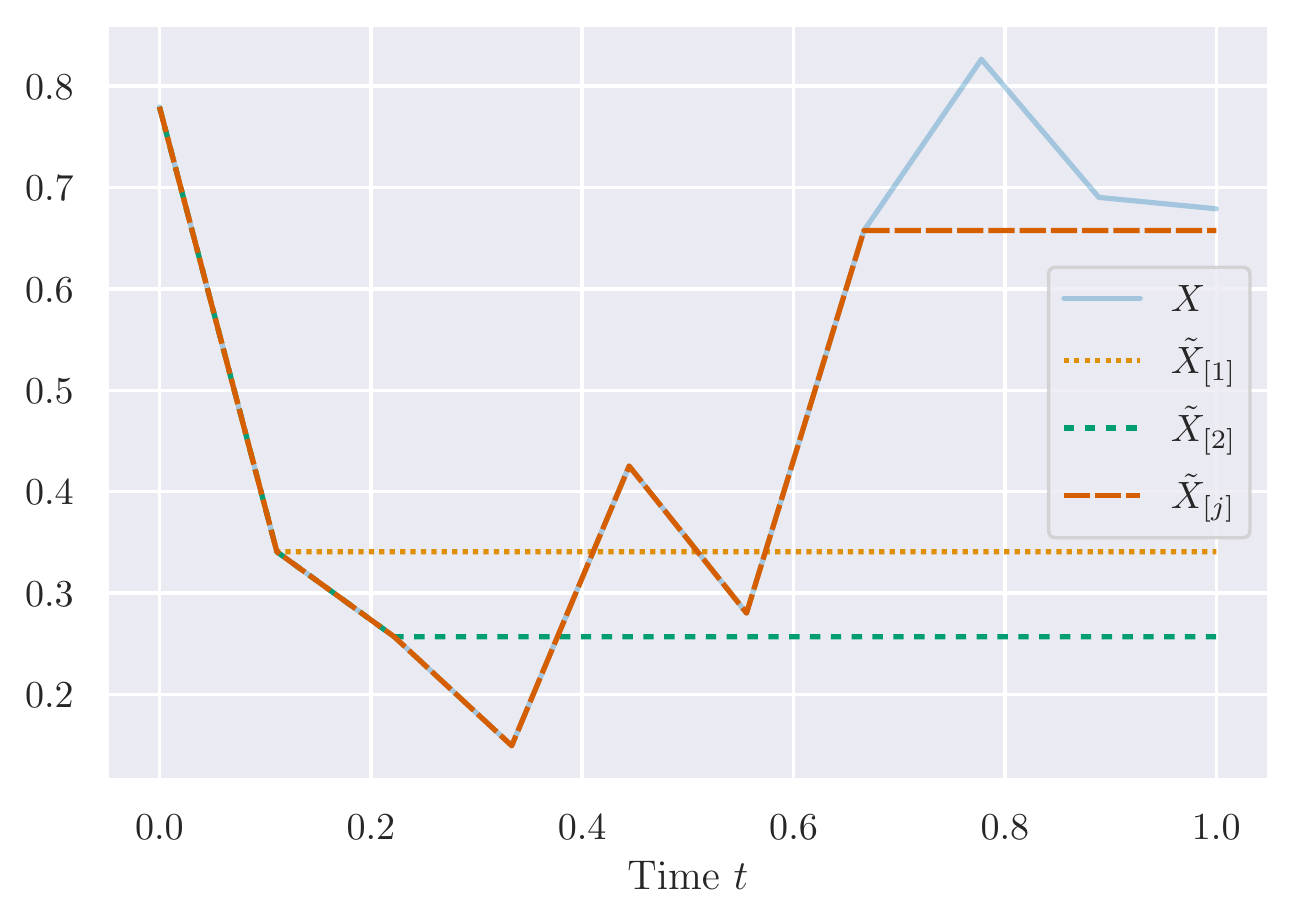}
     \end{subfigure}
    \caption{Example of a path $X \in BV^{c}([0,1], \R)$ (left) and its corresponding paths $\Tilde{X}_{[j]}$, plotted against time, for different values of $j \in \{1, \dots, T\}$ (right)}
    \label{fig:examples_tilde_X}
\end{figure}

\section{Proofs} \label{apx:proofs}

\subsection{Proof of Proposition \ref{prop:forward_euler}}
\label{apx:proof_forward_euler}

According to Assumption $(A_1)$, for any $h_1, h_2 \in \R^e, x_1, x_2 \in \R^d$, one has
    \[\| f(h_1,x_1) - f(h_2,x_1) \| \leq K_f \|h_1 - h_2\|\quad  \text{and} \quad  \| f(h_1,x_1) - f(h_1,x_2) \| \leq K_f \|x_1 - x_2\|.\]
Under assumption $(A_1)$, by Corollary \ref{cor:picard-ode}, the initial value problem \eqref{eq:residual-rnn-ode} admits a unique solution $H$. Let us first show that for any $t \in [0,1]$, $H_t$ is bounded independently of $X$. For any $t\in [0,1]$,
\begin{align*}
    \|H_t-H_0\| =  \Big\| \int_0^t f(H_u, X_u)du \Big\| &\leq  \int_0^t\|  f(H_u, X_u)\|du \\
    & =  \int_0^t\|  f(H_u, X_u) - f(H_0, X_u) + f(H_0, X_u)\|du \\
    & \leq \int_0^t\|  f(H_u, X_u) - f(H_0, X_u)\|  +\int_0^t \| f(H_0, X_u)\|du \\
    & \leq K_f \int_0^t \|H_u - H_0\|du + t \underset{\|x\| \leq L}{\sup} \|f(H_0,x)\|.
\end{align*}
Applying Grönwall's inequality to the function $t \mapsto \|H_t - H_0\|$ yields
\begin{equation*}
    \|H_t - H_0\| \leq t \underset{\|x\| \leq L}{\sup} \|f(H_0,x)\| \exp\Big(\int_0^t K_f du \Big) \leq \underset{\|x\| \leq L}{\sup} \|f(H_0,x)\| e^{K_f}:=M.
\end{equation*}
Given that $H_0 = h_0 = 0$, we conclude that $\|H_t\| \leq M$.

Next, let
\begin{equation*}
    \|f\|_{\infty} = \underset{\|x\| \leq L, \|h\| \leq M}{\sup}f(h,x).
\end{equation*}
By similar arguments, for any $[s,t] \subset [0,1]$, Grönwall's inequality applied to the function $t \mapsto \|H_t - H_s\|$ yields
\begin{equation*}
    \|H_t-H_s\| \leq (t-s) \|f\|_{\infty} e^{K_f}.
\end{equation*}
Therefore, for any partition $(t_0,\dots, t_k)$ of $[s,t]$,
\begin{equation*}
    \sum_{i=1}^k \|H_{t_i} - H_{t_{i-1}}\| \leq \|f\|_{\infty} e^{K_f} \sum_{i=1}^k (t_i-t_{i-1}) \leq \|f\|_{\infty} e^{K_f}(t-s),
\end{equation*}
and, taking the supremum over all partitions of $[s,t]$, $\|H\|_{TV;[s,t]}\leq \|f\|_{\infty} e^{K_f} (t-s)$. In other words, $H$ is of bounded variation on any interval $[s,t] \subset [0,1]$.
Let $(t_0,\dots, t_T)$ denote the regular partition of $[0,1]$ with $t_j = \nicefrac{j}{T}$. For any $1 \leq j \leq T$, we have
\begin{align*}
    \|H_{t_j} - h_j\| &= \big\| H_{t_{j-1}} +  \int_{t_{j-1}}^{t_j} f(H_u, X_u)du - h_{j-1} - \frac{1}{T} f(h_{j-1}, x_j) \big\| \\ 
    & \leq \|  H_{t_{j-1}} - h_{j-1}\| + \int_{t_{j-1}}^{t_j} \big\|f(H_u, X_u) - f(h_{j-1}, x_j) \big\| du.
\end{align*}
Writing
\begin{align*}
    \big\|f(H_u, X_u) - f(h_{j-1}, x_j) \big\| &= \big\|f(H_u, X_u) - f(H_u, x_j) + f(H_u, x_j) - f(h_{j-1}, x_j) \big\| \\
    &\leq \big\|f(H_u, X_u) - f(H_u, x_j) \big\| + \big\|f(H_u, x_j) - f(h_{j-1}, x_j) \big\| \\
    &\leq K_f \big\|X_u - x_j \big\| + K_f \big\|H_u -h_{j-1}\big\|,
\end{align*}
we obtain
\begin{align*}
  \|H_{t_j} - h_j\|  & \leq \|  H_{t_{j-1}} - h_{j-1}\| + K_f \int_{t_{j-1}}^{t_j} \|H_u - h_{j-1} \| du + K_f \int_{t_{j-1}}^{t_j} \|X_u - x_{j} \| du \\
   & \leq \|  H_{t_{j-1}} - h_{j-1}\| + K_f \int_{t_{j-1}}^{t_j} \big(\|H_u - H_{t_{j-1}}\| + \|H_{t_{j-1}} - h_{j-1} \| \big)du \\
   & \quad + \frac{K_f}{T} \|X\|_{TV;[t_{j-1}, t_j]} \\
    &\leq \big(1 + \frac{K_f}{T} \big)\|  H_{t_{j-1}} - h_{j-1}\| + \frac{K_f}{T} \big( \|H\|_{TV;[t_{j-1}, t_j]} + \|X\|_{TV;[t_{j-1}, t_j]} \big).
\end{align*}
By induction, we are led to
\begin{align*}
     \|H_{t_j} - h_j\| &\leq \frac{K_f}{T}  \sum_{k=0}^{j-1} \Big(1 + \frac{K_f}{T}\Big)^{k}\big(\|H\|_{TV;[t_{k}, t_{k+1}]} +\|X\|_{TV;[t_{k}, t_{k+1}]} \big)  \\
     & \leq\frac{K_f}{T} \Big(1 + \frac{K_f}{T}\Big)^{T} \big( \|X\|_{TV;[0, 1]} + \|H\|_{TV;[0, 1]} \big) \\
    & \leq \frac{K_f e^{K_f}}{T}\big( L + \|f\|_{\infty} e^{K_f} \big),
\end{align*}
which concludes the proof.

\subsection{Proof of Proposition \ref{prop:ode_to_cde}}

Let $\bar{h} \in \R^{\bar{e}}$ and let $\bar{h}^{i:j}=(\bar{h}^i, \dots, \bar{h}^{j})$ be its projection on a subset of coordinates. It is sufficient to take $\mathbf{F}$ defined by
\begin{equation*} 
    \mathbf{F}(\bar{h}) =  \begin{pmatrix} 0_{e \times d} & \frac{2}{1-L}f(\bar{h}^{1:e}, \bar{h}^{e+1:e+d}) \\ I_{d\times d} & 0_{d \times 1} \end{pmatrix},
\end{equation*}
where $I_{d \times d}$ denotes the identity matrix and $0_{\cdot \times \cdot}$ the matrix full of zeros. The function $\bar{H}$ is then solution of 
\begin{equation*}
    d\bar{H}_t = \begin{pmatrix} 0_{e \times d} & \frac{2}{1-L} f(\bar{H}_t^{1:e}, \bar{H}_t^{e+1:e+d}) \\  I_{d\times d} & 0_{d \times 1} \end{pmatrix} \begin{pmatrix} dX_t \\ \frac{1-L}{2}dt \end{pmatrix}.
\end{equation*}
Note that under assumption $(A_1)$, the tensor field $\mathbf{F}$ satisfies the assumptions of the Picard-Lindelöf theorem (Theorem \ref{thm:picard-lindelof}) so that $\Bar{H}$ is well-defined. The projection of this equation on the last $d$ coordinates gives
\begin{equation*}
    d\bar{H}_t^{e+1:e+d} = dX_t, \quad \bar{H}_0^{e+1:e+d} = X_0,
\end{equation*}
and therefore $\bar{H}_t^{e+1:e+d}=X_t$. The projection on the first $e$ coordinates gives
\begin{equation*}
   d\bar{H}_t^{1:e} = \frac{2}{1-L}f(\bar{H}_t^{1:e}, X_t) \frac{1-L}{2}dt = f(\bar{H}_t^{1:e}, X_t)dt, \quad \bar{H}_0^{1:e} = h_0,
\end{equation*}
which is exactly \eqref{eq:residual-rnn-ode}.

\subsection{Proof of Proposition \ref{prop:up_bound_norm_sig}}

According to \citet[][Lemma 5.1]{lyons2014rough}, one has
\begin{equation*}
    \|\Bar{\mathbb{X}}^k_{[0,t]}\|_{(\R^d)^{\otimes k}} \leq \|\bar{X}\|_{TV; [0,t]}^k.
\end{equation*}
Let $(t_0, \dots, t_k)$ be a  partition of $[0,t]$. Then
\begin{align*}
    \sum_{j=1}^k \|\bar{X}_{t_j} - \bar{X}_{t_{j-1}}\| &=  \sum_{j=1}^k \sqrt{\|X_{t_j} - X_{t_{j-1}}\|^2 + \Big(\frac{1-L}{2} \Big)^2 (t_{j} -t_{j-1})^2} \\
    & \leq  \sum_{j=1}^k \|X_{t_j} - X_{t_{j-1}}\| + \frac{1-L}{2} \sum_{j=1}^k (t_{j} -t_{j-1}) \\
    & =  \sum_{j=1}^k \|X_{t_j} - X_{t_{j-1}}\| + \frac{1-L}{2}t.
\end{align*}
Taking the supremum over any partition of $[0,t]$ we obtain
\begin{align*}
    \|\Bar{X}\|_{TV;[0,t]} \leq \|X\|_{TV;[0,t]} + \frac{1-L}{2}t \leq L + \frac{1-L}{2} = \frac{1+L}{2} < 1,
\end{align*}
and thus $ \|\mathbb{\Bar{X}}^k_{[0,t]}\|_{(\R^d)^{\otimes k}}  \leq \Big(\frac{1+L}{2} \Big)^k$. It is then clear that
\begin{align*}
    \|S_{[0,t]}(\bar{X}) \|_{\mathscr{T}} =\Big( \sum_{k=0}^{\infty} \|\bar{\mathbb{X}}^k_{[0,t]}\|_{(\R^d)^{\otimes k}}^2 \Big)^{\nicefrac{1}{2}} \leq \sum_{k=0}^{\infty} \|\bar{\mathbb{X}}^k_{[0,t]}\|_{(\R^d)^{\otimes k}} \leq \sum_{k=0}^{\infty} \Big(\frac{1+L}{2} \Big)^k = 2(1-L)^{-1}.
\end{align*}

\subsection{Proof of Proposition \ref{prop:euler_convergence}} \label{apx:proof_euler_convergence}

We first recall the fundamental theorem of calculus for line integrals (also known as gradient theorem). 
\begin{theorem}\label{thm:fund_thm_calculus}
Let $g: \R^e \to \R$ be a continuously differentiable function, and let $\gamma: [a,b] \to \R^e$ be a smooth curve in $\R^e$. Then
\[ \int_a^b \nabla g(\gamma_t) d\gamma_t = g(\gamma_b) - g(\gamma_a),  \]
where $\nabla g$ denotes the gradient of $g$.
\end{theorem}
The identity above immediately generalizes to a function $g: \R^e \to \R^e$:
\[ \int_a^b J(g)(\gamma_t) d\gamma_t = g(\gamma_b) - g(\gamma_a),  \]
where $J(g) \in \R^{e \times e}$ is the Jacobian matrix of $g$. Let us apply Theorem \ref{thm:fund_thm_calculus} to the vector field $F^i$ between $0$ and $t$, with $\gamma = H$. We have
\begin{align*}
    F^i(H_t) - F^i(H_0) &= \int_0^t J(F^i)(H_u)dH_u =\int_0^t J(F^i)(H_u) \sum_{j=1}^d F^j(H_u)dX_u \\
    &= \sum_{j=1}^d \int_0^t J(F^i)(H_u) F^j(H_u)dX_u  = \sum_{j=1}^d \int_0^t F^j \star F^i(H_u)dX_u.
\end{align*}
Iterating this procedure $(N-1)$ times for the vector fields $F^1, \dots, F^d$ yields
\begin{align*}
    H_t &=H_0 +\sum_{i=1}^d \int_0^t F^i(H_u) dX^i_u \\
    &=H_0 + \sum_{i=1}^d \int_0^t F^i(H_0) dX^i_u  + \sum_{i=1}^d \int_0^t \sum_{j=1}^d \int_0^u F^j \star F^i(H_v)dX^j_v dX^i_u\\
    & =H_0 + \sum_{i=1}^d F^i(H_0) S_{[0,t]}^{(i)}(X)  + \sum_{1 \leq i,j \leq d} \int_{0 \leq v \leq u \leq t} F^j \star F^i(H_v)dX^j_v dX^i_u\\
    & = \cdots \\
    & = H_0 + \sum_{k=1}^{N} \sum_{1 \leq i_1, \dots, i_k \leq d} F^{i_1} \star \dots \star F^{i_k}(H_0) \frac{1}{k!} S_{[0,t]}^{(i_1, \dots, i_k)}(X) \\
    & \qquad + \sum_{1 \leq i_1, \dots, i_{N+1} \leq d} \int_{\Delta_{N+1;[0,t]}} F^{i_1} \star \dots \star F^{i_{N+1}}(H_{u_1}) dX^{i_1}_{u_1} \cdots dX^{i_{N+1}}_{u_{N+1}},
\end{align*}
where $\Delta_{N;[0,t]} := \{(u_1,\cdots,u_N) \in [0,t]^N \, | \, 0\leq u_1<\cdots<u_N\leq t \}$ is the simplex in $[0,t]^N$. The first $(N+1)$ terms equal $H^{N}_t$. Hence,
\begin{align*}
    &\| H_t - H^N_t \| \\
    &\quad = \Big\| \sum_{1 \leq i_1, \dots, i_{N+1} \leq d} \int_{\Delta_{N+1;[0,t]}} F^{i_1} \star \dots \star F^{i_{N+1}}(H_{u_1}) dX^{i_1}_{u_1} \cdots dX^{i_{N+1}}_{u_{N+1}} \Big\| \\
     &\quad \leq \sum_{1 \leq i_1, \dots, i_{N+1} \leq d} \int_{\Delta_{N+1;[0,t]}} \|F^{i_1} \star \dots \star F^{i_{N+1}}(H_{u_1})\| |dX^{i_1}_{u_1}| \cdots |dX^{i_{N+1}}_{u_{N+1}}| \\
    &\quad \leq \sum_{1 \leq i_1, \dots, i_{N+1} \leq d} \int_{\Delta_{N+1;[0,t]}} \sup_{1 \leq i_1, \dots, i_{N+1} \leq d, \|h\| \leq M} \|F^{i_1} \star \dots \star F^{i_{N+1}}(h)\| |dX^{i_1}_{u_1}| \cdots |dX^{i_{N+1}}_{u_{N+1}}| \\
    &\quad \leq  \Lambda_{N+1}(\mathbf{F}) \sum_{1 \leq i_1, \dots, i_{N+1} \leq d} \int_{\Delta_{N+1;[0,t]}} |dX^{i_1}_{u_1}| \cdots |dX^{i_{N+1}}_{u_{N+1}}|.
\end{align*}

Thus,
\begin{align*}
    \| H_t - H^N_t \| & \leq  \Lambda_{N+1}(\mathbf{F}) \sum_{1 \leq i_1, \dots, i_{N+1} \leq d} \int_{\Delta_{N+1;[0,t]}} |dX^{i_1}_{u_1}| \cdots |dX^{i_{N+1}}_{u_{N+1}}| \\
    & \leq  \Lambda_{N+1}(\mathbf{F}) \sum_{1 \leq i_1, \dots, i_{N+1} \leq d} \int_{\Delta_{N+1;[0,t]}} \|dX_{u_1}\| \cdots \|dX_{u_{N+1}}\| \\
    & = \Lambda_{N+1}(\mathbf{F}) \frac{d^{N+1}}{(N+1)!} \int_{[0,t]^{N+1}}\|dX_{u_1}\| \cdots \|dX_{u_{N+1}}\| \\
    & = \Lambda_{N+1}(\mathbf{F}) \frac{d^{N+1}}{(N+1)!} \Big(\int_0^t  \|dX_{u}\| \Big)^{N+1} \\
    & =  \Lambda_{N+1}(\mathbf{F}) \frac{d^{N+1}}{(N+1)!} \|X\|_{TV; [0,t]}^{N+1} \leq \Lambda_{N+1}(\mathbf{F}) \frac{d^{N+1}}{(N+1)!}.
\end{align*}

\subsection{Proof of Proposition \ref{prop:bounding_dn_norm_for_rnns}}\label{apx:proof_bounding_dn_norm_for_rnns}
For simplicity of notation, since the context is clear, we now use the notation $\| \cdot \|$ instead of $\| \cdot \|_{(\R^e)^{\otimes k}}$. According to Proposition \ref{prop:forward_euler}, the solution $\bar{H}$ of \eqref{eq:residual-rnn-cde} verifies $\|\Bar{H}_t\| \leq M + L:=\Bar{M}$. We therefore place ourselves in the ball $\mathscr{B}_{\Bar{M}}$. Recall that for any $1 \leq i_1, \dots, i_N \leq d$, $\bar{h} \in \mathscr{B}_{\Bar{M}}$,
\begin{equation}\label{apx:eq:iterated_stars}
F^{i_1} \star \dots \star F^{i_N} (\bar{h}) = J(F^{i_2} \star \dots \star F^{i_N})(\bar{h}) F^{i_1} (\bar{h}).
\end{equation}

\paragraph{Linear case.}
We start with the proof of the linear case before moving on to the general case. When $\sigma$ is chosen to be the identity function, each $F_{\textnormal{RNN}}^i$ is an affine vector field, in the sense that $F_{\textnormal{RNN}}^i(\Bar{h}) =  W_i \Bar{h} + b_i$, where 
$W_i = 0_{\Bar{e} \times \Bar{e}}$, $b_i$ is the $i+d$th vector of the canonical basis of $\R^{e+d}$, and
\[    
    W_{d+1} = \begin{pmatrix} \frac{2}{1-L} W \\ 0_{d\times \Bar{e}}\end{pmatrix} \quad \text{and} \quad b_{d+1} =  \begin{pmatrix} \frac{2}{1-L} b \\ 0_{d}\end{pmatrix}.
\]
Since $J(F_{\textnormal{RNN}}^i) = W_i$, we have, for any $\Bar{h} \in \R^{e+d}$ and any $1 \leq i_1, \dots, i_k \leq d$, 
\begin{equation*}
F_{\textnormal{RNN}}^{i_1} \star \dots \star F_{\textnormal{RNN}}^{i_k}(\Bar{h}) = W_{i_k} \cdots W_{i_2} (W_{i_1}\Bar{h} + b_{i_1}).
\end{equation*}
Thus, for any $\bar{h} \in \mathscr{B}_{\Bar{M}}$,
\begin{equation*}
    \| F_{\textnormal{RNN}}^{i_1} \star \dots \star F_{\textnormal{RNN}}^{i_k}(\bar{h})\|  \leq \|W_{i_k}\|_{\textnormal{op}} \cdots  \|W_{i_2}\|_{\textnormal{op}} ( \|W_{i_1}\|_{\textnormal{op}} \bar{M} + \|b_{i_1}\|).
\end{equation*}
For $i \neq d+1$, $\|W_{i_1}\|_{\textnormal{op}} = 0$, and so 
\begin{align*}
     \Lambda_{k}(\mathbf{F}_{\textnormal{RNN}}) \leq C \|W_{d+1}\|_{\textnormal{op}}^{k-1}, 
\end{align*}
with $C= \|W_{d+1}\|_{\textnormal{op}}\bar{M} + \max(1, 2(1-L)^{-1}\|b\|).$
Therefore, 
\begin{align*}
    \sum_{k=1}^\infty  \frac{d^k}{k!} \Lambda_{k}(\mathbf{F}_{\textnormal{RNN}}) \leq  C d \sum_{k=0}^\infty  \frac{1}{k!} \big( 2d(1-L)^{-1}\|W\|_{\textnormal{op}} \big)^{k-1}< \infty.
\end{align*}

\paragraph{General case.}
In the general case, the proof is two-fold. First, we upper bound \eqref{apx:eq:iterated_stars} by a function of the norms of higher-order Jacobians of $F^{i_1}, \dots, F^{i_N}$. We then apply this bound to the specific case $\mathbf{F} = \mathbf{F_{\textnormal{RNN}}}$. We refer to Appendix \ref{apx:tensor_spaces} for details on higher-order derivatives in tensor spaces. Let $F: \R^e \to \R^e$ be a smooth vector field. If $F(h) = (F_1(h), \dots, F_e(h))^\top$, each of its coordinates $F_i$ is a function from $\R^e$ to $\R$, $\mathscr{C}^\infty$ with respect to all its input variables. We define the derivative of order $k$ of $F$ as the tensor field
 \begin{align*}
     J^k(F): \R^e &\to (\R^e)^{\otimes k +1} \\
            h & \mapsto J^k(F)(h),
 \end{align*}
 where
 \begin{equation*}
     J^k(F)(h) = \sum_{1 \leq j, i_1, \dots, i_k \leq e} \frac{\partial^k F_j (h)}{\partial h_{i_1} \dots \partial h_{i_k}} e_j \otimes e_{i_1} \otimes \dots \otimes e_{i_k}.
 \end{equation*}
 We take the convention $J^{0}(F) = F$, and note that $J(F) = J^1(F)$ is the Jacobian matrix, and that $J^{k}(J^{k'}(F)) = J^{k+k'}(F)$.




\begin{lemma}\label{ref:lemma_bound_star_product}
    Let $A^1, \dots, A^k: \R^e \to \R^e$ be smooth vector fields. Then, for any $h \in \R^e$
    \begin{equation*}
        \big\|   A^{k} \star \dots  \star A^{1}(h) \big\| \leq  \sum_{n_1 + \dots + n_{k} = k-1} C(k; n_1, \dots, n_{k}) \| J^{n_1}(A^{1})(h)\| \cdots \| J^{n_k}(A^{k})(h) \|,
    \end{equation*}
where $C(k; n_1, \dots, n_k)$ is defined by the following recurrence on $k$: $C(1;0) = 1$ and for any $n_1, \dots, n_{k+1} \geq 0$,
\begin{align} \label{apx:eq:coeff-rec}
    C(k+1; n_1, \dots, n_{k+1}) &= \sum_{\ell=1}^{k} C(k;n_1, \dots, n_\ell-1, \dots, n_{k}) \quad  &&\textnormal{if} \quad n_{k+1} = 0,  \\
    C(k+1; n_1, \dots, n_{k+1}) &= 0 \quad &&\textnormal{otherwise.} \nonumber
\end{align}
\end{lemma}

\begin{proof}
    We refer to Appendix \ref{apx:tensor_spaces} for the definitions of the tensor dot product $\odot$ and tensor permutations, as well as for computation rules involving these operations. We show in fact by induction a stronger result, namely that there exist tensor permutations $\pi_p$ such that
    \begin{equation} \label{proof:eq:jacob_of_iterated_stars}
        A^{k} \star \dots \star A^{1}(h) = \sum_{n_1 + \dots + n_{k} = k-1\vphantom{C(n_1)}} \, \sum_{1 \leq  p \leq C(k; n_1, \dots, n_k)} \pi_p \left[ J^{n_{1}}(A^{1})(h) \odot \cdots \odot J^{n_k}(A^{k})(h) \right].
    \end{equation}
Note that we do not make explicit the permutations nor the axes of the tensor dot operations since we are only interested in bounding the norm of the iterated star products. Also, for simplicity, we denote all permutations by $\pi$, even though they may change from line to line.

We proceed by induction on $k$. For $k=1$, the formula is clear. Assume that the formula is true at order $k$. Then
    \begin{alignat*}{2}
       J(&A^{k} \star \dots \star A^{1}) && \\
       &= \sum_{\substack{n_1 + \dots + n_{k} = k-1\vphantom{C(n_1)}}} \,\, \sum_{1 \leq p \leq C(k; n_1, \dots, n_k)} \,\, J \Big[ \,\, \pi_p [ \, J^{n_{1}}(A^{1}) &&\odot \cdots \odot J^{n_k}(A^{k}) \, ] \,\, \Big] \\
       &= \sum_{\substack{n_1 + \dots + n_{k} = k-1\vphantom{C(n_1)}}} \,\, \sum_{1 \leq p \leq C(k; n_1, \dots, n_k)} \,\, \pi_p \Big[ \,\, J  [ \, J^{n_{1}}(A^{1}) &&\odot \cdots \odot J^{n_k}(A^{k}) \, ] \,\, \Big] \\
        & = \sum_{n_1 + \dots + n_{k} = k-1\vphantom{C(n_1)}} \,\, \sum_{1 \leq p \leq C(k; n_1, \dots, n_k)} \,\, \sum_{\ell=1}^k \pi_p \circ \pi_\ell \Big[ J^{n_1}&& (A^{1}) \; \odot \\
        & &&\cdots \odot J^{n_\ell+1}(A^{\ell}) \odot \cdots \odot J^{n_{k}}(A^{k}) \Big].
    \end{alignat*}
In the inner sum, we introduce the change of variable $p_i = n_i$ for $i \neq \ell$ and $p_\ell = n_\ell + 1$. This yields
\begin{alignat*}{2}
       J(&A^{k} \star \dots \star A^{1}) && \\
       & = \sum_{p_1 + \dots + p_{k} = k\vphantom{C(n_1)}} \; \sum_{\ell=1\vphantom{C(n_1)}}^k \; \sum_{1 \leq p \leq C(k;p_1, \dots, p_\ell - 1, \dots, p_k)} \pi_p \circ \pi_\ell \Big[ \; J^{n_1}&&(A^{1}) \; \odot \\
       & &&\cdots \odot J^{n_\ell+1}(A^{\ell}) \odot \cdots  \odot J^{n_{k}}(A^{k}) \; \Big] \\
        & = \sum_{\substack{p_1 + \dots + p_{k+1} = k\vphantom{C(n_1)}}} \; \sum_{1 \leq q \leq C(k+1;p_1, \dots, p_{k+1})} \pi_q \Big[ J^{n_{1}}(A^{1}) \odot \cdots  &&\odot J^{p_k}(A^{k}) \Big],
\end{alignat*}
where in the last sum the only non-zero term is for $p_{k+1} = 0$. To conclude the induction, it remains to note that
\begin{equation*}
A^{k+1} \star \dots \star A^{1} =  J(A^{k} \star \dots \star A^{1}) \odot A^{k+1} = J(A^{k} \star \dots \star A^{1}) \odot J^0(A^{k+1}).
\end{equation*}
Hence,
\begin{align*}
&A^{k+1} \star \dots \star A^{1} \\
&\;\;= \sum_{\substack{p_1 + \dots + p_{k+1} = k\vphantom{C(n_1)}}} \; \sum_{1 \leq q \leq C(k+1;p_1, \dots, p_{k+1})} \pi_q \left[ J^{n_{1}}(A^{1}) \odot \cdots  \odot J^{p_k}(A^{k}) \right]  \odot J^{p_{k+1}}(A^{k+1}) \\
&\;\;= \sum_{\substack{p_1 + \dots + p_{k+1} = k\vphantom{C(n_1)}}} \; \sum_{1 \leq q \leq C(k+1;p_1, \dots, p_{k+1})} \pi_q \left[ J^{n_{1}}(A^{1}) \odot \cdots  \odot J^{p_k}(A^{k}) \odot J^{p_{k+1}}(A^{k+1}) \right].
\end{align*}

The result is then a consequence of \eqref{proof:eq:jacob_of_iterated_stars} and of Lemma \ref{lemma:norm_cdot_product}.
\end{proof}

We now restrict ourselves to the case $\mathbf{F} = \mathbf{F_{\textnormal{RNN}}}$ as defined by \eqref{eq:rnn-vector_field} and give an upper bound on the higher-order derivatives of the tensor fields $F^{i_1}, \dots, F^{i_N}$. 
\begin{lemma}\label{lemma:bound_J_k_RNN}
For any $i \in \{1, \dots, d+1\}$, $\bar{h} \in \mathscr{B}_{\Bar{M}}$, for any $k \geq 0$,
\begin{equation*}
    \| J^{k}(F^i_{\textnormal{RNN}}) (\bar{h})\| \leq  \Big( \frac{2}{1-L}\|W\|_F \Big)^{k} \|\sigma^{(k)}\|_\infty.
\end{equation*}
\end{lemma}
\begin{proof}
For any $1 \leq i \leq d$, $F^i_{\textnormal{RNN}}(\bar{h})$ is constant, so $J^k(F^1_{\textnormal{RNN}}) = \dots = J^k(F^d_{\textnormal{RNN}}) = 0$. For $i=d+1$, we have, for any $1 \leq j \leq e$,
\begin{equation*}
\frac{\partial^k F^{d+1}_{\textnormal{RNN},j} (\bar{h})}{\partial \bar{h}_{i_1} \dots \partial \bar{h}_{i_k}} = \Big( \frac{2}{1-L} \Big)^k W_{j i_1} \cdots W_{j i_k} \sigma^{(k)}(W_{j\cdot }\bar{h} + b),
\end{equation*}
where $W_{j\cdot }$ denotes the $j$th row of $W$ and for $e+1 \leq j \leq \bar{e}$, $F^{d+1}_j = 0$. Therefore,
\begin{align*}
    \| J^{k}(F^{d+1}_{\textnormal{RNN}}) (\bar{h})\|^2 & \leq \Big( \frac{2}{1-L} \Big)^{2k}  \sum_{1 \leq j, i_1, \dots, i_k \leq e} |W_{j i_1} \cdots W_{j i_k} \sigma^{(k)}(W_{j\cdot }\bar{h} + b)|^2 \\
  & =  \Big( \frac{2}{1-L} \Big)^{2k}\|\sigma^{(k)}\|_\infty^2 \sum_j \big(\sum_i |W_{ji}|^2 \big)^k \\
    & \leq \Big( \frac{2}{1-L} \Big)^{2k} \|\sigma^{(k)}\|_\infty^2 \|W\|_F^{2k}.
\end{align*}
\end{proof}

We are now in a position to conclude the proof using condition \eqref{eq:condition_activation_function}. By Lemma \ref{ref:lemma_bound_star_product} and \ref{lemma:bound_J_k_RNN}, for any $1 \leq i_1, \dots, i_N \leq d+1$,
\begin{align*}
      &\big\| F^{i_1}_{\textnormal{RNN}} \star \dots  \star F^{i_N}_{\textnormal{RNN}}(\bar{h}) \big\| \\
      &\quad \leq  \sum_{n_1 + \dots + n_{N} = N-1} C(N;n_N, \dots, n_1) \| J^{n_{N}}(F^{i_{N}}_{\textnormal{RNN}})(\bar{h})\| \cdots \| J^{n_1}(F^{i_1}_{\textnormal{RNN}})(\bar{h}) \| \\
     &\quad  \leq \Big(\frac{2}{1-L}\|W\|_{F} \Big)^{N-1}  \sum_{n_1 + \dots + n_{N} = N-1} C(N;n_N, \dots, n_1)  a^{n_1+1} n_1! \cdots a^{n_N +1}n_N! \\
    &\quad  \leq a\Big(\frac{2}{1-L}a^2\|W\|_{F} \Big)^{N-1} \sum_{n_1 + \dots + n_{N} = N-1} C(N;n_N, \dots, n_1)  n_1! \cdots n_N! \,.
\end{align*}
Assume for the moment that $C(N;n_N, \dots, n_1)$ is smaller than the multinomial coefficient $\binom{N}{n_N, \dots, n_1}$. Then, using the fact that there are $\binom{n+k-1}{k-1}$ weak compositions of $n$ in $k$ parts and Stirling's approximation, we have
\begin{align*}
      \Lambda_{N}(\mathbf{F})
      &\leq a \Big(\frac{2}{1-L}a^2\|W\|_{F} \Big)^{N-1}  N! \times \text{ Card}\big(\{n_1 + \cdots +n_N = N-1  \}\big) \\
      &\leq a \Big(\frac{2}{1-L}a^2\|W\|_{F} \Big)^{N-1} N! \binom{2N-2}{N-1} \\
      &\leq \frac{a}{2} \Big(\frac{2}{1-L}a^2\|W\|_{F} \Big)^{N-1} N! \binom{2N}{N} \\
      &\leq a \frac{\sqrt{2}e}{\pi} \Big(\frac{8}{1-L}a^2\|W\|_{F} \Big)^{N-1} \frac{N!}{\sqrt{N}}.
\end{align*}
Hence, provided $\|W\|_F < \nicefrac{(1-L)}{8a^2d}$,
\begin{align*}
    \sum_{k=1}^\infty\frac{d^k}{k!} \Lambda_{k}(\mathbf{F}) \leq ad \frac{\sqrt{2}e}{\pi} \sum_{k=1}^\infty \Big(\frac{8da^2\|W\|_{F}}{1-L} \Big)^{k-1} \frac{1}{\sqrt{k}}  < \infty,
\end{align*}
and $(A_2)$ is verified.

To conclude the proof, it remains to prove the following lemma.

\begin{lemma}
    For any $k \geq 1$ and $n_1, \dots, n_k \geq 0$,  $ C(k; n_1, \dots, n_k) \leq \binom{k-1}{n_1, \dots, n_k}$.
\end{lemma}
\begin{proof}
    The proof is done by induction, by comparing the recurrence formula \eqref{apx:eq:coeff-rec} with the following recurrence formula for multinomial coefficients:
    \begin{equation*}
        \binom{k}{n_1, \dots, n_{k+1}} = \sum_{\ell=1}^{k+1} \binom{k-1}{n_1, \dots, n_\ell -1, \dots ,n_{k+1}}.
    \end{equation*}
    More precisely, for $k=1$, $C(1;0) = 1 \leq \binom{0}{0} = 1$ and $C(1;1) = 0 \leq \binom{0}{1} = 0$. Assume that the formula is true at order $k$. Then, at order $k+1$, there are two cases. If $n_{k+1} \neq 0$, $C(k+1; n_1, \dots, n_{k+1}) = 0$, and the result is clear. On the other hand, if $n_{k+1} = 0$,
    \begin{align*}
        C(k+1;n_1, \dots, n_k, 0)
        &= \sum_{\ell=1}^{k} C(k;n_1, \dots, n_\ell-1, \dots, n_{k}) \\
        &\leq \sum_{\ell=1}^{k} \binom{k-1}{n_1, \dots, n_\ell -1, \dots ,n_k} \\
        &\leq \sum_{\ell=1}^{k+1} \binom{k-1}{n_1, \dots, n_\ell -1, \dots ,n_{k+1}} \\
        &\leq \binom{k}{n_1, \dots, n_{k+1}}.
    \end{align*}
\end{proof}

\subsection{Proof of Theorem \ref{thm:rnn_in_H_binary}}

First, Propositions \ref{prop:forward_euler} and \ref{prop:ode_to_cde} state that if $\bar{H}$ is the solution of \eqref{eq:residual-rnn-cde} and $\textnormal{Proj}$ denotes the projection on the first $e$ coordinates, then
\begin{equation*}
    \big|z_T - \psi\big(\textnormal{Proj}(\bar{H}_1) \big) \big| = \big|\psi(h_T)- \psi \big(\textnormal{Proj}(\bar{H}_1]) \big) \big| \leq \|\psi\|_{\textnormal{op}} \big\|h_T - \textnormal{Proj}(\Bar{H}_1) \big\| \leq \|\psi\|_{\textnormal{op}}  \frac{c_1}{T}.
\end{equation*}
For any $1 \leq k \leq N $, we let $\mathscr{D}^k(\bar{H}_0): (\R^d)^{\otimes k} \rightarrow \R^e$ be the linear function defined by
\begin{equation} \label{eq:definition_D_k}
    \mathscr{D}^k(\bar{H}_0)(e_{i_1} \otimes \dots \otimes e_{i_k}) = F^{i_1}  \star  \cdots \star F^{i_k} (\bar{H}_0),
\end{equation}
where $e_1, \dots, e_d$ denotes the canonical basis of $\R^{\bar{d}}$. 
Then, under assumptions $(A_1)$ and $(A_2)$, if $\bar{\mathbb{X}}^k$ denotes the signature of order $k$ of the path $\bar{X}_t = (X_t^\top, \frac{1-L}{2}t)^\top$, according to Propositions \ref{prop:euler_convergence} and \ref{prop:bounding_dn_norm_for_rnns},
\begin{equation*}
    \bar{H}_1 = \bar{H}_0 + \sum_{k=1}^\infty \frac{1}{k!} \sum_{1 \leq i_1, \dots, i_k \leq d} S^{(i_1, \dots, i_k)}_{[0,t]}(X) F^{i_1}  \star \cdots \star F^{i_k} (\bar{H}_0) =
    \sum_{k=1}^\infty \frac{1}{k!} \mathscr{D}^k(\bar{H}_0)( \mathbb{X}^k_{[0,t]}),
\end{equation*}
and
\begin{equation*}
    \psi \circ \textnormal{Proj} (\bar{H}_1) = \psi \circ \textnormal{Proj} \Big(\sum_{k=0}^{\infty} \frac{1}{k!} \mathscr{D}^k(\bar{H}_0)(\Bar{\mathbb{X}}^k) \Big) = \sum_{k=0}^{\infty} \frac{1}{k!} \psi \circ \textnormal{Proj}\big( \mathscr{D}^k(\bar{H}_0)(\Bar{\mathbb{X}}^k) \big),
\end{equation*}
by linearity of $\psi$ and $\textnormal{Proj}$. Since the maps $ \mathscr{D}^k(\bar{H}_0): (\R^d)^{\otimes k} \to \R^e$ are linear, the above equality takes the form
\begin{equation}\label{eq:def_alpha_RNN}
    \psi \circ \textnormal{Proj} (\bar{H}_1) = \sum_{k=0}^{\infty} \langle \alpha^k, \Bar{\mathbb{X}}^k \rangle_{(\R^d)^{\otimes k}},
\end{equation}
where $\alpha^k \in (\R^d)^{\otimes k}$ is the coefficient of the linear map $\frac{1}{k!} \psi \circ  \textnormal{Proj} \circ \mathscr{D}^k(\bar{H}_0)$ in the canonical basis. Let $\alpha = (\alpha^0, \dots, \alpha^k, \dots)$. Under assumption $(A_2)$, 
\begin{align*}
    \sum_{k=0}^\infty \|\alpha^k\|^2_{(\R^d)^{\otimes k}} & \leq \sum_{k=0}^\infty \, \sum_{1 \leq i_1, \dots, i_k \leq d} \Big(\frac{1}{k!} \Big)^2 \|\psi\|_{\textnormal{op}}^2 \|F^{i_1} \star \dots \star F^{i_k}(\Bar{H}_0)\|^2 \\
    & \leq \|\psi\|_{\textnormal{op}}^2 \sum_{k=0}^\infty \sum_{1 \leq i_1, \dots, i_k \leq d}  \Big(\frac{1}{k!} \Big)^2 \Lambda_k(\mathbf{F})^2 \\
    & \leq  \|\psi\|_{\textnormal{op}}^2 \sum_{k=0}^\infty  \Big(\frac{d^k}{k!}\Lambda_k(\mathbf{F}) \Big)^2 < \infty.
\end{align*}
This shows that $\alpha\in \mathscr{T}$, and therefore, using \eqref{eq:def_alpha_RNN}, we conclude
\begin{equation*} \label{eq:rnn-approx-scalar-output}
   \|z_T -  \langle \alpha, S(\bar{X}) \rangle_{\mathscr{T}}\| \leq \|\psi\|_{\textnormal{op}}  \frac{c_1}{T}.
\end{equation*}

\subsection{Proof of Theorem \ref{thm:generalization_bound_binary}}\label{apx:proof_generalization_bound_binary}
Let 
\begin{equation*}
    \mathscr{G} = \Big\{ g_\theta: (\R^d)^{T} \to \R \, | \, g_\theta(\mathbf{x}) =z_T, \theta \in \Theta \Big\}
\end{equation*}
be the function class of (discrete) RNN and 
\begin{equation*}
    \mathscr{S}= \Big\{\xi_{\alpha_\theta}: \mathscr{X} \to \R \, | \, \xi_{\alpha_\theta} (X)= \langle \alpha_\theta, S(\Bar{X}) \rangle_{\mathscr{T}}, \theta \in \Theta \Big\},
\end{equation*}
be the class of their RKHS embeddings, where $\alpha_\theta$ is defined by \eqref{eq:def_alpha_RNN}. For any $\theta \in \Theta$, we let
\begin{align*}
    \mathscr{R}_{ \mathscr{G}}(\theta) &= \esp[\ell(\mathbf{y}, g_\theta(\mathbf{x}))], \quad \text{ and } \quad \mathscr{R}_{\mathscr{S}}(\theta) = \esp[\ell(\mathbf{y}, \xi_{\alpha_\theta}(\bar{X}))],
\end{align*}
and denote by $ \widehat{\mathscr{R}}_{n, \mathscr{G}}$ and $ \widehat{\mathscr{R}}_{n,\mathscr{S}}$ the corresponding empirical risks. We also let $ \theta^\ast_{ \mathscr{G}}$, $\theta^\ast_{\mathscr{S}}$, $ \widehat{\theta}_{n, \mathscr{G}}$, and $\widehat{\theta}_{n,\mathscr{S}}$ be the corresponding minimizers. 
We have
\begin{align*}
     \prob\big(\mathbf{y} \neq g_{\widehat{\theta}_{n, \mathscr{G}}}(\mathbf{x}) \big) -  \widehat{\mathscr{R}}_{n, \mathscr{G}}(\widehat{\theta}_{n, \mathscr{G}}) & \leq \esp \big[ \ell(\mathbf{y}, g_{\widehat{\theta}_{n, \mathscr{G}}} (\mathbf{x}))\big] -  \widehat{\mathscr{R}}_{n, \mathscr{G}}(\widehat{\theta}_{n, \mathscr{G}})\\
     & = \mathscr{R}_{ \mathscr{G}}(\widehat{\theta}_{n, \mathscr{G}}) -\widehat{\mathscr{R}}_{n, \mathscr{G}}(\widehat{\theta}_{n, \mathscr{G}}) \\
     & =\mathscr{R}_{\mathscr{G}}(\widehat{\theta}_{n, \mathscr{G}}) -  \mathscr{R}_{\mathscr{S}}(\widehat{\theta}_{n, \mathscr{G}}) + \mathscr{R}_{\mathscr{S}}(\widehat{\theta}_{n, \mathscr{G}}) - \widehat{\mathscr{R}}_{n, \mathscr{S}}(\widehat{\theta}_{n, \mathscr{G}}) \\
     & \quad +\widehat{\mathscr{R}}_{n, \mathscr{S}}(\widehat{\theta}_{n, \mathscr{G}}) - \widehat{\mathscr{R}}_{n,\mathscr{G}}(\widehat{\theta}_{n, \mathscr{G}}) \\
     & \leq \underset{\theta}{\sup} | \mathscr{R}_{\mathscr{G}}(\theta) -  \mathscr{R}_{\mathscr{S}}(\theta)| + \underset{\theta}{\sup} |\mathscr{R}_{\mathscr{S}}(\theta) - \widehat{\mathscr{R}}_{n, \mathscr{S}}(\theta) | \\
     &\quad + \underset{\theta}{\sup} | \widehat{\mathscr{R}}_{n, \mathscr{G}}(\theta) - \widehat{\mathscr{R}}_{n, \mathscr{S}}(\theta) |.
\end{align*}
Using Theorem \ref{thm:rnn_in_H_binary}, we have
\begin{align*}
     \underset{\theta}{\sup} | \mathscr{R}_{\mathscr{G}}(\theta) -  \mathscr{R}_{\mathscr{S}}(\theta)| &=  \underset{\theta}{\sup}\big|\esp \big[\ell(\mathbf{y}, g_\theta(\mathbf{x})) - \ell(\mathbf{y}, \xi_{\alpha_\theta}(\Bar{X})) \big] \big| \\
    &\leq \underset{\theta}{\sup}\esp\big[|\phi(\mathbf{y}g_\theta(\mathbf{x})) - \phi(\mathbf{y}\xi_{\alpha_\theta}(\Bar{X})) | \big] \\
     & \leq \underset{\theta}{\sup}\esp\big[K_{\ell} |\mathbf{y}| \times |g_\theta(\mathbf{x}) - \xi_{\alpha_\theta}(\Bar{X})| \big] \\
     & \leq K_{\ell} \sup_\theta (\|\psi \|_{\textnormal{op}} c_{1,\theta}) \frac{1}{T} := \frac{c_2}{2T},
\end{align*}
where $c_{1, \theta} = K_{f_\theta} e^{K_{f_\theta}} \big( L +\|f_\theta\|_\infty e^{K_{f_\theta}} \big)$ (the infinity norm $\|f_\theta\|_\infty$ is taken on the balls $\mathscr{B}_L$ and  $\mathscr{B}_M$). One proves with similar arguments that 
   \[
     \underset{\theta}{\sup} | \widehat{\mathscr{R}}_{n, \mathscr{G}}(\theta) - \widehat{\mathscr{R}}_{n, \mathscr{S}}(\theta) |  \leq  \frac{c_2}{2T}.
\]
Under the assumption of the theorem, there exists a ball $\mathscr{B} \subset \mathscr{H}$ of radius $B$ such that $\mathscr{S} \subset \mathscr{B}$. This yields
\begin{align*}
    \underset{\theta}{\sup} |\mathscr{R}_{\mathscr{S}}(\theta) - \widehat{\mathscr{R}}_{n, \mathscr{S}}(\theta) |\leq  \underset{\alpha \in \mathscr{T}, \|\alpha\|_{\mathscr{T}} \leq B}{\sup} |\mathscr{R}_{\mathscr{B}}(\alpha) - \widehat{\mathscr{R}}_{n, \mathscr{B}}(\alpha) |,
\end{align*}
where
\begin{equation*}
    \mathscr{R}_{\mathscr{B}}(\alpha) = \esp[\ell(Y,\xi_\alpha(\bar{X}))] \quad \text{and} \quad \widehat{\mathscr{R}}_{n, \mathscr{B}}(\alpha) = \frac{1}{n} \sum_{i=1}^n \ell(Y^{(i)}, \xi_{\alpha}(\bar{X}^{(i)})).
\end{equation*}
We now have reached a familiar situation where the supremum is over a ball in an RKHS. 
A slight extension of \citet[][Theorem 8]{bartlett2002rademacher} yields that with probability at least $1-\delta$,
\begin{equation*}
   \underset{\alpha \in \mathscr{T}, \|\alpha\|_{\mathscr{T}} \leq B}{\sup} |\mathscr{R}_{\mathscr{B}}(\alpha) - \widehat{\mathscr{R}}_{n, \mathscr{B}}(\alpha) | \leq 4 K_{\ell} \esp  \textnormal{Rad}_n(\mathscr{B}) + 2BK_{\ell} (1-L)^{-1} \sqrt{\frac{\log(\nicefrac{1}{\delta})}{2n}},
\end{equation*}
where $\textnormal{Rad}_n(\mathscr{B})$ denotes the Rademacher complexity of $\mathscr{B}$. Observe that we have used the fact that the loss is bounded by $2BK_{\ell}(1-L)^{-1}$ since, for any $\xi_\alpha \in \mathscr{B}$, by the Cauchy-Schwartz inequality,
\begin{align*}
    \ell(\mathbf{y},\xi_\alpha(\bar{X}))  = \phi(\mathbf{y}\langle \alpha, S(\bar{X}) \rangle_{\mathscr{T}}) \leq K_{\ell} |\mathbf{y}\langle \alpha, S(\bar{X}) \rangle_{\mathscr{T}}| &\leq K_{\ell} \|\alpha\|_{\mathscr{T}} \|S(\bar{X}) \|_{\mathscr{T}} \\
    &\leq 2K_{\ell} B (1-L)^{-1}.
\end{align*}
Finally, the proof follows by noting that Rademacher complexity of $\mathscr{B}$ is bounded by
\begin{equation*}
    \textnormal{Rad}_n(\mathscr{B}) \leq \frac{B}{n} \sqrt{\sum_{i=1}^n K(X^{(i)}, X^{(i)})} = \frac{B}{n} \sqrt{\sum_{i=1}^n \|S(\bar{X}^{(i)})\|^2_{\mathscr{T}}} \leq \frac{2B(1-L)^{-1}}{\sqrt{n}}.
\end{equation*}

\subsection{Proof of Theorem \ref{thm:generalization_bound_sequential}}

Let 
\begin{equation*}
    \mathscr{G} = \Big\{ g_\theta: (\R^d)^{T} \to (\R^p)^T \, | \, g_\theta(\mathbf{x}) = \big(z_1, \dots, z_T \big), \theta \in \Theta \Big\}
\end{equation*}
be the function class of discrete RNN in a sequential setting. Let
\begin{equation*}
    \mathscr{S} = \Big\{\Gamma_\theta: \mathscr{X} \to (\R^p)^T \, | \, \Gamma_\theta(X) = \big(\Xi_{\theta}(\Tilde{X}_{[1]}), \dots, \Xi_{\theta}(\Tilde{X}_{[T]}) \big) \Big\},
\end{equation*}
be the class of their RKHS embeddings, where $\Tilde{X}_{[j]}$ is the path equal to $X$ on $[0, \nicefrac{j}{T}]$ and then constant on $[\nicefrac{j}{T}, 1]$ (see Figure \ref{fig:examples_tilde_X}). For any $X \in \mathscr{X}$,
\begin{equation*}
    \Xi_{\theta}(a) = \begin{pmatrix} \langle \alpha_{1,\theta}, S(\bar{X}) \rangle_{\mathscr{T}} \\ \vdots \\ \langle \alpha_{p,\theta}, S(\bar{X}) \rangle_{\mathscr{T}}  \end{pmatrix} = \begin{pmatrix} \xi_{\alpha_{1, \theta}}(X)\\ \vdots \\ \xi_{\alpha_{p, \theta}}(X)  \end{pmatrix} \in \R^p,
\end{equation*}
where $(\alpha_{1,\theta}, \dots, \alpha_{p,\theta})^\top \in (\mathscr{T})^p$ are the coefficients of the linear maps $\frac{1}{k!} \psi\circ \textnormal{Proj} \circ \mathscr{D}^k(\Bar{H}_0): (\R^d)^{\otimes k} \to \R^p$, $k\geq 0$, in the canonical basis, where $\mathscr{D}^k$ is defined by \eqref{eq:definition_D_k}. 

We start the proof as in Theorem \ref{thm:generalization_bound_binary}, until we obtain
\begin{align*}
    \mathscr{R}_{\mathscr{G}}(\widehat{\theta}_{n, \mathscr{G}}) -  \widehat{\mathscr{R}}_{n,\mathscr{G}}(\widehat{\theta}_{n, \mathscr{G}}) & \leq \underset{\theta}{\sup} | \mathscr{R}_{\mathscr{G}}(\theta) -  \mathscr{R}_{\mathscr{S}}(\theta)| + \underset{\theta}{\sup} |\mathscr{R}_{\mathscr{S}}(\theta) - \widehat{\mathscr{R}}_{n, \mathscr{S}}(\theta) | \\
     &\quad + \underset{\theta}{\sup} | \widehat{\mathscr{R}}_{n, \mathscr{G}}(\theta) - \widehat{\mathscr{R}}_{n, \mathscr{S}}(\theta) |.
\end{align*}    
By definition of the loss, for any $ \theta \in \Theta$,
\begin{align*}
     | \mathscr{R}_{\mathscr{G}}(\theta) -  \mathscr{R}_{\mathscr{S}}(\theta)| &=  \Big|\esp \big[\ell\big(\mathbf{y}, g_\theta(\mathbf{x})\big) - \ell\big(\mathbf{y}, \Gamma_{\theta}(X) \big) \big] \Big| \\
     & \leq \esp\Big[ \big|\frac{1}{T} \sum_{j=1}^T \big(\|y_j- z_j\|^2 - \|y_j - \Xi_{\theta}(\Tilde{X}_{[j]})\|^2 \big) \big| \Big] \\
    & \leq \esp\Big[\frac{1}{T} \sum_{j=1}^T \big| \big\langle z_j + \Xi_{\theta}(\Tilde{X}_{[j]}) -  2y_j, z_j - \Xi_{\theta}(\Tilde{X}_{[j]}) \big\rangle \big|\Big] \\
    & \leq \esp\Big[\frac{1}{T} \sum_{j=1}^T \| z_j + \Xi_{\theta}(\Tilde{X}_{[j]}) -  2y_j\| \times \| z_j - \Xi_{\theta}(\Tilde{X}_{[j]}) \| \Big]\\
    & \qquad \mbox{(by the Cauchy-Schwartz inequality).}\\
\end{align*}    
According to inequality \eqref{eq:rnn_kernel_embedding_sequential}, one has
\[ \| z_j - \Xi_{\theta}(\Tilde{X}_{[j]}) \| \leq \|\psi\|_{\textnormal{op}} \frac{c_{1, \theta}}{T},\]
where $c_{1, \theta} = K_{f_\theta} e^{K_{f_\theta}} \big( L +\|f_\theta\|_\infty e^{K_{f_\theta}} \big)$. Moreover, 
\begin{align*}
    \big\| \Xi_{\theta}(\Tilde{X}_{[j]}) \big\|^2 = \sum_{\ell=1}^p \big| \langle \alpha_{\ell, \theta}, S(\Tilde{X}_{[j]}) \rangle_{\mathscr{T}} \big|^2 \leq \sum_{\ell=1}^p \| \alpha_{\ell, \theta}\|^2_{\mathscr{T}} \| S(\Tilde{X}_{[j]})\|_{\mathscr{T}}^2 \leq p B^2 \big(2(1-L)^{-1}\big)^2,
\end{align*}
since $\| S(\Tilde{X}_{[j]})\|_{\mathscr{T}} = \| S_{[0, \nicefrac{j}{T}]}(\bar{X})\|_{\mathscr{T}} \leq \| S(\bar{X})\|_{\mathscr{T}}$. This yields
\begin{align*}
    \| z_j + \Xi_{\theta}(\Tilde{X}_{[j]}) -  2y_j\| &\leq  \|z_j\| + \| \Xi_{\theta}(\Tilde{X}_{[j]})\| +2 \|y_j\| \\
    & \leq \|\psi\|_{\textnormal{op}}\|f_\theta\|_{\infty} + 2\sqrt{p}B(1-L)^{-1} +2K_y.
\end{align*}
Finally,
\begin{align*}
     \underset{\theta}{\sup} | \mathscr{R}_{\mathscr{G}}(\theta) -  \mathscr{R}_{\mathscr{S}}(\theta)| \leq \frac{c_3}{2T},
\end{align*}
where $c_3 = \underset{\theta}{\sup}\big(c_{1, \theta} +\|\psi\|_{\textnormal{op}}\|f_\theta\|_{\infty}\big) + 2\sqrt{p}B(1-L)^{-1} +2K_y$.
One proves with similar arguments that
   \[
     \underset{\theta}{\sup} | \widehat{\mathscr{R}}_{n, \mathscr{G}}(\theta) - \widehat{\mathscr{R}}_{n, \mathscr{S}}(\theta) |  \leq  \frac{c_3}{2T}.
\]     
We now turn to the term $    \underset{\theta}{\sup} |\mathscr{R}_{\mathscr{S}}(\theta) - \widehat{\mathscr{R}}_{n, \mathscr{S}}(\theta) |$. We have
\begin{align*}
    &\mathscr{R}_{\mathscr{S}}(\theta)- \widehat{\mathscr{R}}_{n, \mathscr{S}}(\theta)\\
    &\quad = \esp[ \ell(\mathbf{y}, \Gamma_\theta(X))] - \frac{1}{n} \sum_{i=1}^n \ell(\mathbf{y}^{(i)}, \Gamma_\theta(X^{(i)})) \\
    &\quad  = \frac{1}{T} \sum_{j=1}^{T} \Big( \esp[ \|y_j - \Xi_\theta(\Tilde{X}_{[j]})]\|^2 - \frac{1}{n} \sum_{i=1}^n \big\| y^{(i)}_j - \Xi_\theta(\Tilde{X}_{[j]}^{(i)}) \big\|^2 \Big).
\end{align*}
Therefore,
\begin{align*}
    \underset{\theta}{\sup}| \mathscr{R}_{\mathscr{S}}(\theta)- \widehat{\mathscr{R}}_{n, \mathscr{S}}(\theta)| &\leq  \frac{1}{T} \sum_{j=1}^{T}   \underset{\theta}{\sup} \Big| \esp[ \|y_j - \Xi_\theta(\Tilde{X}_{[j]})]\|^2 - \frac{1}{n} \sum_{i=1}^n \big\| y^{(i)}_j - \Xi_\theta(\Tilde{X}_{[j]}^{(i)}) \big\|^2 \Big|.
\end{align*}
Note that for a fixed $j$, the pairs $( \Tilde{X}_{[j]}^{(i)}, y^{(i)}_j)$ are i.i.d. Under the assumptions of the theorem, there exists a ball $\mathscr{B} \subset \mathscr{H}$ such that for any $1 \leq \ell \leq p$, $\theta \in \Theta$, $\xi_{\alpha_{\ell, \theta}} \in \mathscr{B}$ . We denote by $\mathscr{B}_p$ the sum of $p$ such spaces, that is,
\begin{equation*}
   \mathscr{B}_p = \big\{f_{\alpha} : \mathscr{X} \to \R^p \, | \, f_\alpha(X) = (f_{\alpha_1}(X),  \dots, f_{\alpha_p}(X))^\top, f_{\alpha_\ell} \in \mathscr{B}\big\}.
\end{equation*}
Clearly, $\Xi_\theta \in \mathscr{B}_p$, and it follows that 
\begin{align*}
     &\underset{\theta}{\sup} \Big| \esp[ \|y_j - \Xi_\theta(\Tilde{X}_{[j]})]\|^2 - \frac{1}{n} \sum_{i=1}^n \big\| y^{(i)}_j - \Xi_\theta(\Tilde{X}_{[j]}^{(i)}) \big\|^2 \Big| \\
     & \qquad \leq  \underset{f_{\alpha} \in \mathscr{B}_p}{\sup} \Big| \esp \big[ \| y_j - f_\alpha(\Tilde{X}_{[j]}) \|^2 \big] - \frac{1}{n} \sum_{i=1}^n \| y^{(i)}_j - f_\alpha(\tilde{X}^{(i)}_{[j]}) \|^2 \Big|.
\end{align*}
We have once again reached a familiar situation, which can be dealt with by an easy extension of \citet[][Theorem 12]{bartlett2002rademacher}. For any $f_{\alpha} \in \mathscr{B}_p$, let $\Tilde{\phi} \circ f_{\alpha}:\mathscr{X} \times \R^p : (X,y) \mapsto \|y - f_{\alpha}(X)\|^2 - \|y\|^2$. Then, $\Tilde{\phi} \circ f_{\alpha}$ is upper bounded by 
\begin{align*}
   | \Tilde{\phi} \circ f_\alpha(X,y) | = \big| \|y-f_{\alpha}(X)\|^2 - \|y\|^2 \big| &\leq \|f_{\alpha}(X)\| \big(\|f_{\alpha}(X)\| + 2 \|y\| \big)\\
   &\leq 2\sqrt{p}B(1-L)^{-1} ( 2\sqrt{p}B(1-L)^{-1} + 2K_y) \\
    &\leq 4pB(1-L)^{-1} (B(1-L)^{-1} + K_y).
\end{align*}
Let $c_4=B(1-L)^{-1} + K_y$ and $c_5 =  4pB(1-L)^{-1}c_4 + K_y^2$. Then with probability at least $1-\delta$,
\begin{align*}
     \underset{f_{\alpha} \in \mathscr{B}_p}{\sup} \Big| \esp \big[ \| y_j - f_\alpha(\Tilde{X}_{[j]}) \| \big] - \frac{1}{n} \sum_{i=1}^n \| y^{(i)}_j - f_\alpha(\tilde{X}^{(i)}_{[j]}) \| \Big| \leq  \textnormal{Rad}_n(\tilde{\phi} \circ \mathscr{B}_p) + \sqrt{\frac{2c_5\log(\nicefrac{1}{\delta})}{n}},
\end{align*}
where $\Tilde{\phi} \circ \mathscr{B}_p = \big\{(X,y) \mapsto  \Tilde{\phi} \circ f_{\alpha}(X,y) | f_{\alpha} \in \mathscr{B}_p \big\}$. Elementary computations on Rademacher complexities yield
\begin{align*}
    \textnormal{Rad}_n(\tilde{\phi} \circ \mathscr{B}_p) \leq 2 p c_4  \textnormal{Rad}_n(\mathscr{B}) \leq \frac{4 pc_4B(1-L)^{-1}}{\sqrt{n}},
\end{align*}
which concludes the proof.

\section{Differentiation with higher-order tensors} \label{apx:tensor_spaces}
 
 \subsection{Definition}
 
We define the generalization of matrix product between square tensors of order $k$ and $\ell$. 
 
 \begin{definition}
Let $a \in (\R^e)^{\otimes k}$, $b \in (\R^e)^{\otimes \ell}$, $p \in \{1, \dots, k\}$, $q \in \{1, \dots, \ell\}$. Then the tensor dot product along $(p, q)$, denoted by  $a \odot_{p,q} b \in (\R^e)^{ \otimes (k + \ell -2)}$, is defined by
\begin{align*}
    (a \odot_{p,q} b)_{(i_1, \dots, i_{k-1}, j_1, \dots, j_{\ell - 1})} = \sum_{j=1}^{e} a_{(i_1, \dots, i_{p-1}, j, i_{p}, \dots, i_{k-1})} b_{(j_1, \dots, j_{q-1}, j, j_{q}, \dots, 
    j_{\ell - 1})}.
\end{align*}
\end{definition}
This operation just consists in computing $a \otimes b$, and then summing the $p$th coordinate of $a$ with the $q$th coordinate of $b$. The $\odot$~operator is not associative. To simplify notation, we take the convention that it is evaluated from left to right, that is, we write $a \odot b \odot c$ for $(a \odot b) \odot c$.

\begin{definition}
Let $a \in (\R^e)^{\otimes k}$. For a given permutation $\pi$ of $\{1, \dots, k \}$, we denote by $\pi(a)$ the permuted tensor in $(\R^e)^{\otimes k}$ such that 
\begin{equation*}
    \pi(a)_{(i_1, \dots, i_k)} = a_{(i_{\Pi(1)}, \dots, i_{\Pi(k)})}.
\end{equation*} 
\end{definition}
 
 \begin{example}
   If $A$ is a matrix, then $A^T = \pi(A)$, with $\pi$ defined by $\pi(1) = 2, \pi(2) = 1$.
 \end{example}

\subsection{Computation rules}
 
We need to obtain two computation rules for the tensor dot product: bounding the norm (Lemma \ref{lemma:norm_cdot_product}) and differentiating (Lemma \ref{lemma:product_rule}).
 
\begin{lemma} \label{lemma:norm_cdot_product}
    Let $a \in (\R^e)^{\otimes k}$, $b \in (\R^e)^{\otimes \ell}$. Then, for all $p, q$,
    \begin{equation*}
        \| a \odot_{p,q} b \|_{(\R^e)^{\otimes k + \ell -2d}} \leq \| a \|_{(\R^e)^{\otimes k}} \| b \|_{(\R^e)^{\otimes \ell}}.
    \end{equation*}
\end{lemma}
\begin{proof}
  By the Cauchy-Schwartz inequality,
    \begin{align*}
         &\| a \odot_{p,q} b \|^2_{(\R^e)^{\otimes k + \ell -2}} \\
         & \quad= \sum_{1 \leq i_1, \dots, i_{k-1}, j_1, \dots, j_{\ell - 1} \leq e} ( a\odot_{p,q}b)_{(i_1, \dots, i_{k-1}, j_1, \dots, j_{\ell - 1})}^2 \\
         & \quad= \sum_{1 \leq i_1, \dots, i_{k-1}, j_1, \dots, j_{\ell - 1} \leq e} \Big(\sum_{1 \leq j \leq e} a_{(i_1, \dots, i_{p-1}, j, i_{p}, \dots, i_{k-1})} b_{(j_1, \dots, j_{q-1}, j, j_{q}, \dots, j_{\ell - 1})} \Big)^2 \\
         & \quad \leq \sum_{i_1, \dots, i_{k-1}, j_1, \dots, j_{\ell - 1}} \Big(\sum_{j} a_{(i_1, \dots, i_{p-1}, j, i_{p}, \dots, i_{k-1})}^2 \Big) \Big(\sum_{j} b_{(j_1, \dots, j_{q-1}, j, j_{q}, \dots, j_{\ell - 1})}^2 \Big) \\
         & \quad \leq \sum_{i_1, \dots, i_{k-1}, j} a_{(i_1, \dots, i_{p-1}, j, i_{p}, \dots, i_{k-1})}^2 \sum_{j_1, \dots, j_{\ell - 1}, j} b_{(j_1, \dots, j_{q-1}, j, j_{q}, \dots, j_{\ell - 1})}^2 \\
         & \quad \leq  \| a \|_{(\R^e)^{\otimes k}}^2 \| b \|_{(\R^e)^{\otimes \ell}}^2.
    \end{align*}
\end{proof}



\begin{lemma} \label{lemma:product_rule}
    Let $A: \R^e \to (\R^e)^{\otimes k}$, $B: \R^e \to (\R^e)^{\otimes \ell}$ be smooth vector fields, $p \in \{1, \dots, k\}$, $q \in \{1, \dots, \ell\}$. Let $A \odot_{p,q} B: \R^e \to (\R^e)^{\otimes k + \ell - 2}$ be defined by $A \odot_{p,q} B(h) = A(h) \odot_{p,q} B(h)$. Then there exists a permutation $\pi$ such that
    \begin{equation*}
        J(A \odot_{p,q} B) = \pi(J(A) \odot_{p,q} B) + A \odot_{p,q} J(B).
    \end{equation*}
\end{lemma}

\begin{proof}
The left-hand side takes the form
\begin{align*}
    (J(A \odot_{p,q} B))_{i_1, \dots, i_{k-1}, j_1, \dots, j_{\ell - 1}, m} 
    = \mathlarger{\sum}_{j} 
        \Big[ 
            & \frac{\partial A}{\partial h_m}_{(i_1, \dots, i_{p-1}, j, i_p, \dots, i_{k-1})} B_{(j_1, \dots, j_{q-1}, j, j_q, \dots, j_{\ell - 1})} \nonumber \\
            & + A_{(i_1, \dots, i_{p-1}, j, i_p, \dots, i_{k-1})} \frac{\partial B}{\partial h_m}_{(j_1, \dots, j_{q-1}, j, j_q, \dots, j_{\ell - 1})}
        \Big].
\end{align*}
The first term of the right-hand side writes
\begin{equation*}
    (J(A) \odot_{p,q} B)_{i_1, \dots, i_{k-1}, m, j_1, \dots, j_{\ell - 1}} 
    = \mathlarger{\sum}_{j} 
        \Big[ 
            \frac{\partial A}{\partial h_m}_{(i_1, \dots, i_{p-1}, j, i_p, \dots, i_{k-1})} B_{(j_1, \dots, j_{q-1}, j, j_q, \dots, j_{\ell - 1})}
        \Big],
\end{equation*}
and the second one
\begin{equation*}
    (A \odot_{p,q} J(B))_{i_1, \dots, i_{k-1}, j_1, \dots, j_{\ell - 1}, m} 
    = \mathlarger{\sum}_{j} 
        \Big[ 
             A_{(i_1, \dots, i_{p-1}, j, i_p, \dots, i_{k-1})} \frac{\partial B}{\partial h_m}_{(j_1, \dots, j_{q-1}, j, j_q, \dots, j_{\ell - 1})}
        \Big].
\end{equation*}
Let us introduce the permutation $\pi$ which keeps the first $(k-1)$ axes unmoved, and rotates the remaining $\ell$ ones such that the last axis ends up in $k$th position. Then
\begin{equation*}
    \pi(J(A) \odot_{p,q} B)_{i_1, \dots, i_{k-1}, j_1, \dots, j_{\ell - 1}, m}
    = \mathlarger{\sum}_{j} 
        \Big[ 
            \frac{\partial A}{\partial h_m}_{(i_1, \dots, i_{p-1}, j, i_p, \dots, i_{k-1})} B_{(j_1, \dots, j_{q-1}, j, j_q, \dots, j_{\ell - 1})}
        \Big].
\end{equation*}
Hence $J(A \odot_{p,q} B) = \pi(J(A) \odot_{p,q} B) + A \odot_{p,q} J(B)$, which concludes the proof.
\end{proof}
The following two lemmas show how to compose the Jacobian and the tensor dot operations with permutations. Their proofs follow elementary operations and are therefore omitted.
\begin{lemma}
    Let $A: \R^e \to (\R^e)^{\otimes k}$ and $\pi$ a permutation of $\{1, \dots, k\}$. Then there exists a permutation $\tilde{\pi}$ of $\{1, \dots, k+1\}$ such that
    \begin{equation*}
        J(\pi(A)) = \tilde{\pi}(J(A)).
    \end{equation*}
\end{lemma}

\begin{lemma} \label{lemma:commute-pi-odot}
    Let $a \in (\R^e)^{\otimes k}$, $b \in (\R^e)^{\otimes \ell}$, $p \in \{1, \dots, k\}$, $q \in \{1, \dots, \ell\}$, $\pi$ a permutation of $\{1, \dots, k\}$. Then there exists $\tilde{p} \in \{1, \dots, k\}$, $\tilde{q} \in \{1, \dots, \ell\}$, and a permutation $\tilde{\pi}$ of $\{1, \dots, k + \ell - 2\}$  such that
    \begin{equation*}
        \pi(a) \odot_{p,q} b = \tilde{\pi}(a \odot_{\tilde{p},\tilde{q}} b).
    \end{equation*}
\end{lemma}

The following result is a generalization of Lemma \ref{lemma:product_rule} to the case of a dot product of several tensors. 

\begin{lemma} \label{lemma:generalized_product_rule}
   For $\ell \in \{1, \dots, k\}$, $n_\ell \in \N$, let $A_\ell: \R^e \to (\R^e)^{\otimes n_\ell}$ be smooth tensor fields. For any $(p_\ell)_{1 \leq \ell \leq k-1}$ and $(q_\ell)_{1 \leq \ell \leq k-1}$ such that $p_\ell \in \{1, \dots, n_\ell\}$, $q_\ell \in \{1, \dots, n_{\ell+1}\}$, there exist $k$ permutations $(\pi_\ell)_{1 \leq \ell \leq k}$ such that
    \begin{equation*}
        J(A_1 \odot_{p_1,q_1} A_2 \odot_{p_2,q_2} \dots \odot_{p_{k-1}, q_{k-1}} A_k) = \sum_{\ell=1}^k  \pi_\ell \left[ A_1 \odot A_2 \odot \dots \odot J(A_\ell) \odot \dots \odot A_k \right],
    \end{equation*}
    where the dot products of the right-hand side are along some axes that are not specify for simplicity.
\end{lemma}
\begin{proof}
    The proof is done by induction on $k$. The formula for $k=1$ is straightforward. Assume that the formula is true at order $k$. As before, we do not specify indexes for tensor dot products as we are only interested in their existence. By Lemma \ref{lemma:commute-pi-odot}, we have
    \begin{align*}
        &J(A_1 \odot \dots \odot A_{k+1}) \\
            & \quad = J((A_1 \odot \dots \odot A_k) \odot A_{k+1}) \\
            & \quad= \pi( J(A_1 \odot \dots \odot A_k) \odot A_{k+1}) + A_1 \odot \dots \odot A_k \odot J(A_{k+1}) \\
            & \quad= \pi \left[ \sum_{\ell=1}^k  \pi_\ell \left[ A_1 \odot A_2 \odot \dots \odot J(A_\ell) \odot \dots \odot A_k \right] \odot A_{k+1} \right] + A_1 \odot \dots \odot A_k \odot J(A_{k+1}) \\
            & \quad = \pi \left[ \sum_{\ell=1}^k  \tilde{\pi}_\ell \left[ A_1 \odot A_2 \odot \dots \odot J(A_\ell) \odot \dots \odot A_k \odot A_{k+1} \right] \right] + A_1 \odot \dots \odot A_k \odot J(A_{k+1}) \\
            & \quad= \sum_{\ell=1}^k  \hat{\pi}_\ell \left[ A_1 \odot A_2 \odot \dots \odot J(A_\ell) \odot \dots \odot A_k \odot A_{k+1} \right]  + A_1 \odot \dots \odot A_k \odot J(A_{k+1}) \\
            & \qquad \mbox{(where $\hat{\pi} = \pi \circ \tilde{\pi}$)}\\
            & \quad= \sum_{\ell=1}^{k+1}  \hat{\pi}_\ell \left[ A_1 \odot A_2 \odot \dots \odot J(A_\ell) \odot \dots \odot A_k \odot A_{k+1} \right].
    \end{align*}
\end{proof}

\section{Experimental details} \label{apx:supp_experiments}

All the code to reproduce the experiments is available on GitHub at \url{https://github.com/afermanian/rnn-kernel}.
Our experiments are based on the PyTorch \citep{paszke2019pytorch} framework. When not specified, the default parameters of PyTorch are used.

\paragraph{Convergence of the Taylor expansion.} For Figure \ref{fig:euler_convergence}, $10^3$ random RNN with 2 hidden units are generated, with the default weight initialization. The activation is either the logistic or the hyperbolic tangent. In Figure \ref{fig:euler_convergence_scatter_plot}, only the results with the logistic activation are plotted. The process $X$ is taken as a 2-dimensional spiral. The reference solution to the ODE \eqref{eq:residual-rnn-ode} is computed with a numerical integration method from SciPy \citep[][\texttt{scipy.integrate.solve\_ivp} with the `LSODA' method]{scipy}. The signature in the step-$N$ Taylor expansion is computed with the package Signatory \citep{signatory}.

The step-$N$ Taylor expansion requires computing higher-order derivatives of tensor fields (up to order $N$). This is a highly non-trivial task since standard deep learning frameworks are optimized for first-order differentiation only. We refer to, for example, \citet{kelly2020higher}, for a discussion on higher-order differentiation in the context of a deep learning framework. To compute it efficiently, we manually implement forward-mode higher-order automatic differentiation for the operations needed in our context (described in Appendix \ref{apx:tensor_spaces}). A more efficient and general approach is left for future work. Our code is optimized for GPU.

\paragraph{Penalization on a toy example.} For Figure \ref{fig:adversarial-robustness}, the RNN is taken with 32 hidden units and hyperbolic tangent activation. The data are 50 examples of spirals, sampled at 100 points and labeled $\pm1$ according to their rotation direction. We do not use batching and the loss is taken as the cross entropy. It is trained for 200 epochs with Adam \citep{kingmaAdamMethodStochastic2017} with an initial learning rate of 0.1. The learning rate is divided by 2 every 40 epochs. For the penalized RNN, the RKHS norm is truncated at $N=3$ and the regularization parameter is selected at $\lambda = 0.1$. Earlier experiments show that this order of magnitude is sensible. We do not perform hyperparameter optimization since our goal is not to achieve high performance. The initial hidden state $h_0$ is learned (for simplicity of presentation, our theoretical results were written with $h_0=0$ but they extend to this case). The accuracy is computed on a test set of size 1000. We generate adversarial examples using 50 steps of projected gradient descent \citep[following][]{bietti2019kernel}. The whole methodology (data generation + training) is repeated 20 times. The average training time on a Tesla V100 GPU for the RNN is 8.5 seconds and for the penalized RNN 12 seconds.

Figure \ref{fig:comparison_norms} is obtained by selecting randomly one run among the 20 of Figure \ref{fig:adversarial-robustness}.

\paragraph{Libraries.} We use PyTorch \citep{paszke2019pytorch} as our overall framework, Signatory \citep{signatory} to compute the signatures, and SciPy \citep{scipy} for ODE integration. We use Sacred \citep{sacred} for experiment management.
The links and licences for the assets are given in the following table:

\begin{table}[ht]
    \centering
    \begin{tabular}{ccc}
       \toprule
       Name & Homepage link & License \\
       \midrule
       PyTorch  & \href{https://github.com/pytorch/pytorch}{GitHub repository}  & \href{https://github.com/pytorch/pytorch/blob/master/LICENSE}{BSD-style License} \\
       Sacred  & \href{https://github.com/IDSIA/sacred}{GitHub repository}  & \href{https://github.com/IDSIA/sacred/blob/master/LICENSE.txt}{MIT License} \\
       SciPy & \href{https://github.com/scipy/scipy}{GitHub repository} & \href{https://github.com/scipy/scipy/blob/master/LICENSE.txt}{BSD 3-Clause "New" or "Revised" License} \\
       Signatory  & \href{https://github.com/patrick-kidger/signatory}{GitHub repository}  & \href{https://github.com/patrick-kidger/signatory/blob/master/LICENSE}{Apache License 2.0} \\
       \bottomrule
    \end{tabular}
\end{table}

\end{document}